\documentclass{article}

\usepackage{arxiv}
\usepackage{fullpage}
\usepackage{charter}

\usepackage[utf8]{inputenc} 
\usepackage[T1]{fontenc}    
\usepackage{hyperref}       
\usepackage{url}            
\usepackage{booktabs}       
\usepackage{amsfonts}       
\usepackage{nicefrac}       
\usepackage{microtype}      
\usepackage{xcolor}         
\usepackage{dsfont}         

\usepackage{graphicx}
\usepackage{xcolor}
\definecolor{Highlight}{HTML}{39b54a}  

\usepackage{bbding}
\usepackage{multirow}
\usepackage{algorithm}
\usepackage{algorithmic}
\usepackage{amsmath}
\usepackage{amsthm}
\usepackage{amssymb}
\usepackage{bm}
\usepackage{threeparttable}
\usepackage{xspace}
\usepackage{enumitem}
\usepackage{bbm}
\usepackage{cleveref}
\usepackage{dsfont}
\usepackage{subfigure}
\usepackage{caption}
\usepackage{graphicx}
\usepackage[textsize=scriptsize,backgroundcolor=white]{todonotes}
\usepackage{enumitem}
\usepackage{titletoc}
\Crefname{problem}{Problem}{Problems}

\usepackage{stfloats}

\usepackage{aisecure-math}
\usepackage{wrapfig}
\usepackage{footmisc} 
\Crefname{definition}{Def.}{Defs.}
\Crefname{theorem}{Thm.}{Thms.}

\newcommand{\squishlist}{
\begin{list}{{{\small{$\bullet$}}}}
{\setlength{\itemsep}{3pt}      \setlength{\parsep}{1pt}
\setlength{\topsep}{1pt}       \setlength{\partopsep}{0pt}
\setlength{\leftmargin}{1em} \setlength{\labelwidth}{1em}
\setlength{\labelsep}{0.5em} } }
\newcommand{\squishend}{  \end{list}  }

\definecolor{amber}{rgb}{1.0, 0.49, 0.0}

\newcommand{\shiftingone}{{sensitive shifting}\xspace}
\newcommand{\shiftingonecapital}{{Sensitive Shifting}\xspace}
\newcommand{\shiftingtwo}{{general shifting}\xspace}
\newcommand{\shiftingtwocapital}{{General Shifting}\xspace}

\newcommand{\cP}{\mathcal{P}}

\newcommand{\cX}{\mathcal{X}}
\newcommand{\cY}{\mathcal{Y}}

\newcommand{\dist}{\mathrm{dist}}

\newcommand\DoToC{%
  \startcontents
  \printcontents{}{1}{\textbf{Contents}\vskip3pt\hrule\vskip5pt}
  \vskip3pt\hrule\vskip5pt
}

\title{Certifying Some Distributional Fairness\\ with Subpopulation Decomposition}

\author{%
  Mintong Kang
  \thanks{The first two authors contributed equally.} \\
  UIUC\\
  \href{mailto:mintong2@illinois.edu}{\small\texttt{mintong2@illinois.edu}}\\
  \And
  Linyi Li~\footnotemark[1]  \\
  UIUC\\
  \href{mailto:linyi2@illinois.edu}{\small\texttt{linyi2@illinois.edu}}
  \And 
  Maurice Weber \\
  ETH Zurich \\
  \href{mailto:webermau@inf.ethz.ch}{\small\texttt{webermau@inf.ethz.ch}}
  \And
  Yang Liu \\
  UC Santa Cruz \\
  \href{mailto:yangliu@ucsc.edu}{\small\texttt{yangliu@ucsc.edu}} \\
  \And
  Ce Zhang \\
  ETH Zurich \\
  \href{mailto:ce.zhang@inf.ethz.ch}{\small\texttt{ce.zhang@inf.ethz.ch}} \\
  \And
  Bo Li \\
  UIUC\\
  \href{mailto:lbo@illinois.edu}{\small\texttt{lbo@illinois.edu}}
}

\date{}

\begin{document}

\maketitle

\part*{}
\vspace{-4em}

\begin{abstract}
    Extensive efforts have been made to understand and improve the fairness of machine learning models based on different fairness measurement metrics, especially in high-stakes domains such as medical insurance, education, and hiring decisions. However, there is a lack of \textit{certified fairness} on the end-to-end performance of an ML model. In this paper, we first formulate the certified fairness of an ML model trained on a given data distribution as an optimization problem based on the model performance loss bound on a fairness constrained distribution, which is within bounded distributional distance with the training distribution.
    We then propose a general fairness certification framework and instantiate it for both sensitive shifting and general shifting scenarios. In particular, we propose to solve the optimization problem by decomposing the original data distribution into analytical subpopulations and proving the convexity of the sub-problems to solve them.
     We evaluate our certified fairness on six real-world datasets and show that our certification is tight in the sensitive shifting scenario and provides non-trivial certification under general shifting. Our framework is flexible to integrate additional non-skewness constraints and
     we show that it provides even tighter certification under different real-world scenarios.
    We also compare our certified fairness bound with adapted existing distributional robustness bounds  on Gaussian data and demonstrate 
    that our method is 
    significantly tighter.
\end{abstract}

\section{Introduction}
\vspace{-2mm}


As machine learning (ML) has become ubiquitous~\cite{lakkaraju2017selective,hardt2016equality,barocas2016big,chouldechova2017fair,berk2021fairness,datta2015automated}, fairness of ML has attracted a lot of attention from different perspectives. For instance, some automated hiring systems are biased towards males due to gender imbalanced training data~\cite{asuncion2007uci}. Different approaches have been proposed to improve ML fairness, such as 
regularized training~\cite{edwards2015censoring,kehrenberg2020null,liao2019learning,madras2018learning}, disentanglement~\cite{creager2019flexibly,locatello2019fairness,sarhan2020fairness}, duality~\cite{song2019learning}, low-rank matrix factorization~\cite{oneto2020learning}, and distribution alignment~\cite{balunovic2021fair,louizos2015variational,zhao2019conditional}.

In addition to existing approaches
that \textit{evaluate} fairness, it is important and challenging to provide \textit{certification} for ML fairness.  
Recent studies have explored the certified fair \textit{representation} of ML~\cite{ruoss2020learning,balunovic2021fair,peychev2021latent}.
However, there lacks certified fairness on the \textit{predictions} of an end-to-end ML model trained on an arbitrary data {distribution}. 
 In addition, current fairness literature mainly focuses on training an ML model on a potentially (im)balanced distribution and evaluate its performance in a target domain measured by existing statistical fairness definitions~\cite{gliner1994reviewing,jin2021transferability}. Since in practice these selected target domains can encode certain forms of unfairness of their own (e.g., sampling bias), the evaluation would be more informative if we can evaluate and certify fairness of an ML model on an \textit{objective} distribution.
Taking these factors into account, in this work, we aim to provide the first definition of \textit{certified fairness} given an ML model and a training distribution by bounding its end-to-end performance on an objective, \textit{fairness constrained distribution}. In particular, we define \textit{certified fairness} as the worst-case upper bound of the ML prediction loss on a fairness constrained test distribution $\gQ$, which is within a bounded distance to the training distribution $\gP$.
For example, for an ML model of crime rate prediction, we can define the model performance as the expected loss within a specific age group.
Suppose the model is deployed in a fair environment that does not deviate too much from the training, our fairness certificate can guarantee that the loss of crime rate prediction for a particular age group is upper bounded, which is an indicator of model's fairness.

We mainly focus on the base rate condition as the fairness constraint for $\gQ$.
We prove that our certified fairness based on a base rate constrained distribution will imply other fairness metrics, such as demographic parity (DP) and equalized odds (EO). 
Moreover, our framework is flexible to integrate other fairness constraints into $\gQ$.
We consider two 
scenarios: (1) \textit{\shiftingone} where
only the joint distribution of sensitive attribute and label
can be changed when optimizing $\gQ$;
and (2) \textit{\shiftingtwo}
where everything including the conditioned distribution of non-sensitive attributes can be changed.
We then propose an effective \textit{fairness certification framework} to  compute the certificate.

In our fairness certification framework, we first formulate the problem as constrained optimization, where the fairness constrained distribution is encoded by base rate constraints. 
Our key technique is to decompose both training and the fairness constrained test distributions to several subpopulations based on sensitive attributes and target labels, which can be used to encode the base rate constraints.
With such a decomposition, in \shiftingone, we can decompose the distance constraint to subpopulation ratio constraints and prove the transformed low-dimensional optimization problem is convex and thus efficiently solvable.
In \shiftingtwo case, we propose to solve it based on divide and conquer: we first partition the feasible space into different subpopulations,  then optimize the density (ratio) of each subpopulation, apply relaxation on each subpopulation as a sub-problem, and finally prove the convexity of the sub-problems with respect to other low-dimensional variables.
Our framework is applicable for any black-box ML models and any distributional shifts bounded by the Hellinger distance, which is a type of $f$-divergence studied in the literature  \cite{weber2022certifying,duchi2019variance,ben2013robust,lam2016robust,duchi2021statistics}.

To demonstrate the effectiveness and tightness of our framework, we evaluate our fairness bounds on \textit{six} real-world fairness related datasets~\cite{asuncion2007uci,angwin2016machine,healthdataset,wightman1998lsac}. We show that our certificate is tight under different scenarios. In addition, we verify that our framework is flexible to integrate additional constraints on  $\gQ$ and evaluate the certified fairness with additional non-skewness constraints, with which our fairness certificate is tighter. Finally, as the first work on certifying  fairness of an end-to-end ML model, we adapt existing distributional robustness bound~\cite{sinha2017certifying} for comparison to provide more intuition. Note that directly integrating the fairness constraint to the existing distributional robustness bound is challenging, which is one of the main contributions for our framework. We show that with the fairness constraints and our effective solution, our bound is strictly tighter.

\underline{\textbf{Technical Contributions}}. In this work, we take the first attempt towards formulating and computing the \textit{certified fairness}  on an end-to-end ML model, which is trained on a given distribution. We make contributions on both theoretical and empirical fronts.
\vspace{-3mm}
\begin{enumerate}[leftmargin=*]
    \item We formulate the \textit{certified fairness} of an end-to-end ML model trained on a given distribution $\gP$ as the worst-case upper bound of its prediction loss on a fairness constrained distribution $\gQ$, which is within bounded distributional distance with $\gP$.
    \item We propose an effective fairness certification framework that simulates the problem as constrained optimization and solve it by decomposing the training and fairness constrained test distributions into subpopulations and proving the convexity of each sub-problem to solve it.
    \item We evaluate our certified fairness on six real-world datasets to show its tightness and scalability. We also show that with additional distribution constraints on $\gQ$, our certification would be tighter.
    \item We show that our bound is strictly tighter than adapted  distributional robustness bound on Gaussian dataset due to the added  fairness constraints and our effective optimization approach.
\end{enumerate}

\vspace{-4mm}
\paragraph{Related Work}
Fairness in ML can be generally categorized into individual fairness and group fairness.
Individual fairness guarantees that similar inputs should lead to similar outputs for a model and it is analyzed with optimization approaches~\cite{winter2019learning,McNamara2019costs} and different types of relaxations~\cite{john2020verifying}.
Group fairness indicates to measure the \textit{independence} between the sensitive features and model prediction, the \textit{separation} which means that the sensitive features are statistically independent of model prediction given  the target label, and the \textit{sufficiency} which means that the sensitive features are statistically independent of the target label given the model prediction~\cite{liu2019implicit}. Different approaches are proposed to analyze group fairness via static analysis~\cite{urban2020perfectly}, interactive computation~\cite{segal2021fairness}, and probabilistic approaches~\cite{albarghouthi2017fairsquare,choi2020group,bastani2019probabilistic}.
In addition, there is a line of work trying to certify the \textit{fair representation}~\cite{ruoss2020learning,balunovic2021fair,peychev2021latent}. In \cite{chen2022fairness}, the authors have provided bounds for how group fairness transfers subject to bounded distribution shift.  
Our certified fairness differs from existing work from three perspectives: 1) we provide fairness certification considering the end-to-end model performance instead of the representation level, 2) we define and certify fairness based on a fairness constrained distribution which implies other fairness notions, and 3) our certified fairness can be computed for \textit{any} black-box models trained on an arbitrary given data distribution.

\vspace{-0.5em}
\section{Certified Fairness Based on Fairness Constrained Distribution}
    \label{sec:background-and-formulatiuon}
    \vspace{-2mm}

    In this section, we first introduce  preliminaries, and then propose the definition of \textit{certified fairness} based on a bounded fairness constrained distribution, which to the best of our knowledge is the first formal fairness certification on end-to-end model prediction.
    We also show that our proposed certified fairness relates to established fairness definitions in the literature.
    
    \textbf{Notations.}
        \label{subsec:prelim}
        We consider the general classification setting: we denote by $\gX$ and $\gY = [C]$ the feature space and labels, $[C] := \{1,2,\cdots,C\}$.
        $h_\theta\colon\cX \to \Delta^{|\cY|}$ represents a mapping function parameterized with $\theta\in\Theta$, and $\ell\colon\Delta^{|\cY|}\times\cY\to\R_+$ is a non-negative loss function such as cross-entropy loss. 
        Within feature space $\gX$, we identify a \emph{sensitive} or \emph{protected attribute} $\gX_s$ that takes a finite number of values: $\gX_s := [S]$, i.e., for any $X\in \gX$, $X_s \in [S]$.
        
        \begin{definition}[Base Rate]
            \label{def:base-rate}
            Given a distribution $\gP$ supported over $\gX \times \gY$, the base rate for sensitive attribute value $s \in [S]$ with respect to label $y\in [C]$ is
                $b_{s,y}^{\gP} = \Pr_{(X,Y)\sim\gP} [Y=y \,|\, X_s=s]$.
        \end{definition}
        Given the definition of base rate, we define a \textit{fair base rate distribution}~(in short as \textit{fair distribution}).
        \begin{definition}[Fair Base Rate Distribution]
            A distribution $\gP$ supported over $\gX\times\gY$ is a fair base rate distribution if and only if for any label $y\in [C]$,  the base rate $b_{s,y}^{\gP}$ is equal across all $s\in [S]$, i.e., $\forall i\in [S], \forall j\in [S], b^{\gP}_{i,y} = b^{\gP}_{j,y}$.
            \label{def:fair-dist}
        \end{definition}
        
        \begin{remark}
            In the literature, the concepts of fairness are usually directly defined at the model prediction level, where the criterion is whether the model prediction is fair against individual attribute changes~\cite{ruoss2020learning,peychev2021latent,yeom2020individual} or fair at population level~\cite{zhao2019inherent}.
            In this work, to certify the fairness of model prediction, we define a fairness constrained distribution on which we will certify the model prediction (e.g., bound the prediction error), rather than relying on the empirical fairness evaluation.
            In particular, we first define the fairness constrained distribution through the lens of base rate parity, i.e., the probability of being any class should be independent of sensitive attribute values, and then define the certified fairness of a given model based on its performance on the fairness constrained distribution as we will show next.
            
            The choice of focusing on fair base rate may look restrictive but its definition aligns very well with the celebrated fairness definition Demographic Parity \cite{zemel2013learning}, which promotes that $\Pr[h_{\theta}(X)=1|X_s = i] = \Pr[h_{\theta}(X)=1|X_s = j].$ In this case, the prediction performance of $h_{\theta}$ on $\gQ$ with fair base rate will relate directly to $\Pr[h_{\theta}(X)=1|X_s = i]$.    
            Secondly, under certain popular data generation process, the base rate sufficiently encodes the differences in distributions and a fair base rate will imply a homogeneous (therefore equal or ``fair") distribution over $X,Y$: consider when $\Pr(X|Y=y,X_s=i)$ is the same across different group $X_s$. Then $\Pr(X,Y|X_s=i)$ is simply a linear combination of basis distributions $\Pr(X|Y=y,X_s=i)$, and the difference between different groups' joint distribution of $X,Y$ is fully characterized by the difference in base rate $\Pr(Y=y|X_s)$. This assumption will greatly enable trackable analysis and is not an uncommon modeling choice in the recent discussion of fairness when distribution shifts \cite{zhang2020fair,raab2021unintended}. 

        \end{remark}
    
        \vspace{-0.4em}
    \subsection{Certified Fairness}
        \vspace{-0.4em}
        \label{subsec:formulation}
        Now we are ready to define the fairness certification 
        based on the optimized fairness constrained distribution.
        We define the certification under two data generation scenarios: \textit{\shiftingtwo} and \textit{\shiftingone}. 
         In particular, consider the data generative model $\Pr(X_o, X_s, Y) = \Pr(Y)\Pr(X_s|Y) \Pr(X_o|Y, X_s)$, where $X_o$ and $X_s$ represent the non-sensitive and sensitive features, respectively. If all three random variables on the RHS are allowed to change, we call it \textit{\shiftingtwo}; if both $\Pr(Y)$ and $\Pr(X_s|Y)$ are allowed to change to ensure the fair base rate (Def.~\ref{def:fair-dist}) while $\Pr(X_o|Y, X_s)$ is the same across different groups, we call it \textit{\shiftingone}.
        In \Cref{sec:method} we will introduce our certification framework for both scenarios.

        \begin{problem}[Certified Fairness with \shiftingtwocapital]
            Given a training distribution $\gP$ supported on $\gX\times\gY$, a model $h_\theta(\cdot)$ trained on $\gP$, and distribution distance bound $\rho > 0$, 
            we call $\bar\ell \in \R$ a \emph{fairness certificate} with \shiftingtwo, if $\bar\ell$ upper bounds
            \begin{equation*}
                \small
                \max_{\gQ} \, \E_{(X,Y)\sim\gQ} [\ell(h_\theta(X), Y)]
                \quad \mathrm{s.t.} \quad \dist(\gP,\gQ) \le \rho, \quad \gQ \text{ is a fair distribution},
            \end{equation*}
            where $\dist(\cdot,\cdot)$ is a predetermined distribution distance metric.
            \label{prob:certified-fairness-shifting-two}
        \end{problem}
        In the above definition, we define the fairness certificate as the upper bound of the model's loss among all fair base rate distributions $\gQ$ within a bounded distance from  $\gP$.
        Besides the bounded distance constraint $\dist(\gP,\gQ) \le \rho$, there is no other constraint between $\gP$ and $\gQ$ so this satisfies ``\textit{\shiftingtwo}''.
        This bounded distance constraint, parameterized by a tunable parameter $\rho$, ensures that the test distribution should not be too far away from the training.
        In practice, the model  $h_\theta$ may represent a DNN whose complex analytical forms would pose challenges for solving \Cref{prob:certified-fairness-shifting-two}.
        As a result, as we will show in \Cref{eq:generic-prob} we can query some statistics of  $h_\theta$ trained on $\gP$ as constraints to characterize $h_\theta$, and thus compute the upper bound certificate.
        
        The feasible region of optimization problem \ref{prob:certified-fairness-shifting-two} might be empty if the distance bound $\rho$ is too small, and thus we cannot provide fairness certification in this scenario, indicating that there is no nearby fair distribution and thus the fairness of the model trained on the highly ``unfaired" distribution is generally low. 
        In other words, if the training distribution $\gP$ is unfair (typical case) and there is no feasible fairness constrained distribution $\gQ$ within a small distance to $\gP$, fairness cannot be certified. 
        
        This definition follows the intuition of typical real-world scenarios:
        The real-world training dataset is usually biased due to the limitation in data curation and collection processes, which causes the model to be unfair.
        Thus, when the trained models are evaluated on the real-world fairness constrained test distribution or ideal fair distribution, we hope that the model does not encode the training bias which would lead to low test performance.
        That is to say, the model performance on fairness constrained distribution is indeed a witness of the model's intrinsic fairness.

        We can further constrain that  the subpopulation of $\gP$ and $\gQ$ parameterized by $X_s$ and $Y$ does not change, which results in the following ``\textit{\shiftingone}'' fairness certification.
        
        \begin{problem}[Certified Fairness with \shiftingonecapital]
            Under the same setting as \Cref{prob:certified-fairness-shifting-two}, we call $\bar\ell$ a fairness certificate against \shiftingone, if $\bar\ell$ upper bounds
            \begin{equation*}
                \small
                \begin{aligned}
                & \max_{\gQ} \, \E_{(X,Y)\sim\gQ} [\ell(h_\theta(X), Y)] \\
                \quad \mathrm{s.t.} \quad & \dist(\gP,\gQ) \le \rho, \quad \gP_{s,y} = \gQ_{s,y} \, \forall s\in[S],y\in [C], \quad  \gQ \text{ is a fair distribution}, 
                \end{aligned}
            \end{equation*}
            where $\gP_{s,y}$ and $\gQ_{s,y}$ are the subpopulations of $\gP$ and $\gQ$ on the support $\{(X,Y): X\in\gX,X_s=s,Y=y\}$ respectively, and $\dist(\cdot,\cdot)$ is a predetermined distribution distance metric.
            \label{prob:certified-fairness-shifting-one}
        \end{problem}
        
        The definition adds an additional constraint between $\gP$ and $\gQ$ that each subpopulation, partitioned by the sensitive attribute $X_s$ and label $Y$, does not change.
        This constraint corresponds to the scenario where the distribution shifting between training and test distributions only happens on the proportions of different sensitive attributes and labels, and within each subpopulation the shifting is negligible.
        
        In addition, to model the real-world test distribution, we may further request that the test distribution $\gQ$ is not too skewed regarding the sensitive attribute $X_s$ by adding constraint (\ref{eqn:skewness}).
        We will show that this constraint can also be integrated into our fairness certification framework flexibly in \Cref{sec:exp_more_cons}.
        \begin{equation}
            \small
            \forall i \in [S], \forall j \in [S], \left| \Pr_{(X,Y)\sim\gQ} [X_s = i] -  \Pr_{(X,Y)\sim\gQ} [X_s = j] \right| \le \Delta_S.
            \label{eqn:skewness}
        \end{equation}
        \vspace{-1.2em}

\textbf{Connections to Other Fairness Measurements.}      
Though not explicitly stated, our goal of certifying the performance on a fair distribution $\gQ$ relates to certifying established fairness definitions in the literature. Consider the following example: Suppose Problem \ref{prob:certified-fairness-shifting-one} is feasible and returns a classifier $h_{\theta}$ that achieves certified fairness per group and per label class $\bar{l}:=\Pr_{(X,Y)\sim\gQ}[h_{\theta}(X) \neq Y|Y = y, X_s = i] \leq \epsilon$  on $\gQ$. 
We will then have the following proposition:
\begin{proposition}
   $h_{\theta}$ achieves $\epsilon$-Demographic Parity (DP) \cite{zemel2013learning} and $\epsilon$-Equalized Odds (EO) \cite{hardt2016equality}:
   \squishlist
       \item $\epsilon$-DP:
       \begin{small}$
\left|\Pr_{\gQ}[h_{\theta}(X) = 1|X_s = i]-
\Pr_{\gQ}[h_{\theta}(X) = 1|X_s = j]\right| \leq \epsilon, ~\forall i,j.
$\end{small}
\item $\epsilon$-EO: \begin{small}$ 
\left|\Pr_{\gQ}[h_{\theta}(X) = 1|Y = y, X_s = i]-
\Pr_{\gQ}[h_{\theta}(X) = 1|Y = y, X_s = j]\right| \leq \epsilon, \forall y, i,j.
$\end{small}
   \squishend
\label{prop:fairness}
\end{proposition}

\begin{remark}
The detailed proof is omitted to \cref{proof-fairness-notion}. 
(1)~When $\epsilon = 0$, \Cref{prop:fairness} can guarantee perfect DP and EO simultaneously. We achieve so because we evaluate with a fair distribution $\gQ$, where “fair distribution” stands for “equalized base rate” and according to \cite[Theorem 1.1, page 5]{kleinberg2016inherent} both DP and EO are achievable for this fair distribution. 
This observation in fact motivated us to identify the fair distribution $\gQ$ for the evaluation since it is this fair distribution that allows the fairness measures to hold at the same time. 
Therefore, another way to interpret our framework is: given a model, we provide a framework that certifies worst-case         ``unfairness'' bound in the context where perfect fairness is achievable. Such a worse-case bound serves as the gap to a perfectly fair model and could be a good indicator of the model’s fairness level.
(2)~In practice, $\epsilon$ is not necessarily zero. Therefore, \Cref{prop:fairness} only provides an upper lower bound of DP and EO, namely $\epsilon$-DP and $\epsilon$-EO, instead of absolute DP and EO. The approximate fairness guarantee renders our results more general. 
Meanwhile, there is a higher flexiblity in simultaneously satisfying approximate fairness metrics (for example when DP $= 0$, but EO = $\epsilon$, which is plausible for a proper range of epsilon, regardless of the distribution $\gQ$ being fair or not). But again, similar to (1), $\epsilon$-DP and $\epsilon$-EO can be achieved at the same time easily since the test distribution satisfies base rate parity. 

The bounds in \Cref{prop:fairness} are tight. 
Consider the distribution $\gQ$ with binary classes and binary sensitive attributes (i.e., $Y,X_s \in \{0,1\}$). When the distribution $\gQ$ and classifier $h_\theta$ satisfy the conditions that $\mathrm{Pr}_{\gQ} [h_\theta(X) \neq Y | Y=0, X_s=0] = \epsilon, \mathrm{Pr}_{\gQ} [h_\theta(X) \neq Y | Y=0, X_s=1] = 0$ and $\mathrm{Pr}_{\gQ}[Y=0]=1, \mathrm{Pr}_{\gQ}[Y=1]=0$, the bounds in Proposition 1 are tight. From $\mathrm{Pr}_{\gQ}[Y=0]=1, \mathrm{Pr}_{\gQ}[Y=1]=0$, we can observe that $\epsilon$-DP is equivalent to $\epsilon$-EO. From $\mathrm{Pr}_{\gQ} [h_\theta(X) \neq Y | Y=0, X_s=0] = \epsilon, \mathrm{Pr}_{\gQ} [h_\theta(X) \neq Y | Y=0, X_s=1] = 0$ and $\mathrm{Pr}_{\gQ} [h_\theta(X) \neq Y | Y=0, X_s=i] = \mathrm{Pr}_{\gQ} [h_\theta(X) = 1 | Y=0, X_s=i]$ for $i \in \{0,1\}$, we know that $\epsilon$-EO holds with tightness since $\left|\Pr_{\gQ}[h_{\theta}(X) = 1|Y = 0, X_s = 0]-\Pr_{\gQ}[h_{\theta}(X) = 1|Y = 0, X_s = 1]\right| = \epsilon$. To this point, we show that both bounds in Proposition 1 are tight.
\end{remark}

  \vspace{-2mm}
\section{Fairness Certification Framework}
    \label{sec:method}
    \allowdisplaybreaks
    \vspace{-2mm}
    
    We will introduce our fairness certification framework which efficiently computes the fairness certificate defined in \Cref{subsec:formulation}.
    We first introduce our framework for \textit{\shiftingone}~(\Cref{prob:certified-fairness-shifting-one}) which is less complex and shows our core methodology, then  \textit{\shiftingtwo} case~(\Cref{prob:certified-fairness-shifting-two}).
    
    Our framework focuses on using the Hellinger distance to bound the distributional distance in \Cref{prob:certified-fairness-shifting-one,prob:certified-fairness-shifting-two}.
    The Hellinger distance $H(\gP,\gQ)$ is defined in \Cref{def:hellinger}~(in \Cref{adxsubsec:hellinger}).
    The Hellinger distance has some nice properties, e.g., $H(\gP,\gQ)\in [0,1]$, and $H(\gP, \gQ) = 0$ if and only if $\gP = \gQ$ and the maximum value of 1 is attained when $\gP$ and $\gQ$ have disjoint support.
    The Hellinger distance is a type of $f$-divergences which are widely studied in ML distributional robustness literature~\cite{weber2022certifying,duchi2019variance} and in the context of distributionally robust optimization~\cite{ben2013robust,lam2016robust,duchi2021statistics}.
    Also, using Hellinger distance enables our certification framework to generalize to \textit{total variation distance (or statistic distance)} $\delta(\gP,\gQ)$\footnote{$\delta(\gP,\gQ) = \sup_{A\in\gF} |\gP(A)-\gQ(A)|$ where $\gF$ is a $\sigma$-algebra of subsets of the sample space $\Omega$.} directly with the connection, $H^2(\gP,\gQ) \le \delta(\gP,\gQ) \le \sqrt{2}H(\gP,\gQ)$~(\cite{steerneman1983total}, Equation 1).
    We leave the extension of our framework to other distance metrics as future work.
    
   \vspace{-0.7em}
    \subsection{Core Idea: Subpopulation Decomposition}
        \label{subsec:core-idea}
  \vspace{-0.7em}
    
        The core idea in our framework is (finite) subpopulation decomposition.
        Consider a generic optimization problem for computing the loss upper bound on a constrained test distribution $\gQ$, given training distribution $\gP$ and trained model $h_\theta(\cdot)$,
        we first characterize model $h_\theta(\cdot)$ based on some statistics, e.g., mean and variance for loss of the model: $h_\theta(\cdot)$ satisfies $e_j(\gP, h_\theta) \le v_j$, $1\le j \le L$.
        Then we characterize the properties~(e.g., fair base rate) of the test distribution $\gQ$: $g_j(\gQ) \le u_j$, $1 \le j \le M$.
        As a result, we can upper bound the loss of $h_\theta(\cdot)$ on $\gQ$ as the following optimization:
        \begin{equation}
            \small
            \max_{\gQ,\theta} \E_{(X,Y)\sim\gQ} [\ell(h_\theta(X), Y)] 
            \quad \mathrm{s.t.} \quad
            H(\gP,\gQ) \le \rho, \quad 
            e_j(\gP, h_\theta) \le v_j \, \forall j \in [L], \quad 
            g_j(\gQ) \le u_j \, \forall j\in [M].
            \label{eq:generic-prob}
        \end{equation}
        Now we decompose the space $\gZ := \gX\times\gY$ to $N$ partitions: $\gZ := \biguplus \gZ_i$, where $\gZ$ is the support of both $\gP$ and $\gQ$.
        Then, we denote $\gP$ conditioned on $\gZ_i$ by $\gP_i$ and similarly $\gQ$ conditioned on $\gZ_i$ by $\gQ_i$.
        As a result, we can write $\gP = \sum_{i\in[N]} p_i\gP_i$ and $\gQ = \sum_{i\in[N]} q_i\gQ_i$.
        Since $\gP$ is known, $p_i$'s are known.
        In contrast, both $\gQ_i$ and $q_i$'s are optimizable.
        Our key observation is that
        \vspace{-0.5em}
        \begin{equation}
            \small \label{eq:hellinger-prop}
            H(\gP,\gQ) \le \rho \iff 1-\rho^2 - \sum_{i=1}^N \sqrt{p_iq_i} (1 - H(\gP_i,\gQ_i)^2) \le 0
        \end{equation}
        which leads to the following theorem.
        \begin{theorem}
            The following constrained optimization upper bounds \Cref{eq:generic-prob}:
            \vspace{-0.5em}
            \begin{subequations}
                \label{cons-opt-generic}
                \begin{small}
                \begin{align}
                    \max_{\gQ_i, q_i,\rho_i,\theta} \quad & \sum_{i=1}^N q_i \E_{(X,Y)\sim\gQ_i} [\ell(h_\theta(X), Y)] \\
                    \mathrm{s.t.} \quad  & 1 - \rho^2 - \sum_{i=1}^N \sqrt{p_iq_i} (1 - \rho_i^2) \le 0, \label{cons-opt-generic-con1} \\
                    & H(\gP_i,\gQ_i) \le \rho_i \quad \forall i\in [N], \quad 
                    \sum_{i=1}^N q_i = 1, \quad q_i \ge 0 \quad \forall i\in [N], \quad \rho_i \ge 0 \quad \forall i\in [N],  \label{cons-opt-generic-con2}\\
                    & e_j'(\{\gP_i\}_{i\in [N]}, \{p_i\}_{i\in [N]}, h_\theta) \le v'_j \, \forall j \in [L], \quad 
                    g_j'(\{\gQ_i\}_{i\in [N]}, \{q_i\}_{i\in [N]}) \le u_j' \, \forall j\in [M],
                \end{align}
                \end{small}
            \end{subequations}
            if $e_j(\gP, h_\theta) \le v_j$ implies $e_j'(\{\gP_i\}_{i\in[N]}, \{p_i\}_{i\in[N]}, h_\theta) \le v_j'$ for any $j\in [L]$, and
            $g_j(\gQ) \le u_j$ implies $g_j'(\{\gQ_i\}_{i\in [N]}, \{q_i\}_{i\in [N]}) \le u_j'$ for any $j \in [M]$.
            \label{thm:generic}
        \end{theorem}
        
        In Problem~\ref{eq:generic-prob}, the challenge is to deal with the fair base rate constraint. Our core technique in \Cref{thm:generic} is subpopulation decomposition.
        At a high level, thanks to the disjoint support among different subpopulations, we get \Cref{eq:hellinger-prop}. 
        This equation gives us an equivalence relationship between distribution-level~(namely, $\gP$ and $\gQ$) distance constraint and subpopulation-level~(namely, $\gP_i$'s and $\gQ_i$'s) distance constraint. 
        As a result, we can rewrite the original problem (\ref{eq:generic-prob}) using sub-population as decision variables as in \Cref{cons-opt-generic-con1} and then imposing the unity constraint (\Cref{cons-opt-generic-con2}) to get \Cref{thm:generic}.
        We provide a detailed proof in \Cref{adxsubsec:pf-thm-1}.
        Although the optimization problem (\Cref{cons-opt-generic}) may look more complicated then the original \Cref{eq:generic-prob}, 
        this optimization simplifies the challenging fair base rate constraint, allows us to upper bound each subpopulation loss $\E_{(X,Y)\sim\gQ_i} [\ell(h_\theta(X), Y)]$ individually, and hence makes the whole optimization tractable.
    
 \vspace{-0.7em}
    \subsection{Certified Fairness with \shiftingonecapital}
        \label{subsec:shifting-one}
 \vspace{-0.7em}
    
        
        For the \shiftingone case, we instantiate \Cref{thm:generic} and obtain the following fairness certificate.

        \begin{theorem}
            Given a distance bound $\rho > 0$, the following constrained optimization, which is \textbf{convex}, when feasible, provides a \textbf{tight} fairness certificate for \Cref{prob:certified-fairness-shifting-one}:
            \begin{subequations}
                \label{eq:general_opt_prob_label_shifting_thm}
                \begin{small}
                \begin{align*}
                    \max_{k_s,r_y} \quad  \sum_{s=1}^S \sum_{y=1}^{C} k_s r_y E_{s,y}, \quad 
                    \mathrm{s.t.} \quad & \sum_{s=1}^S k_s = 1, \quad \sum_{y=1}^C r_y = 1, \quad k_s \ge 0 \quad \forall s\in [S], \quad r_y \ge 0 \quad \forall y\in [C],
                    \\
                    &  1-\rho^2-\sum_{s=1}^S \sum_{y=1}^C \sqrt{p_{s,y}k_s r_y} \le 0,
                \end{align*}
                \end{small}
            \end{subequations}
            where $E_{s,y} := \E_{(X,Y)\sim\gP_{s,y}} [\ell(h_\theta(X),Y)]$ and $p_{s,y} := \Pr_{(X,Y)\in \gP}[X_s=s,Y=y]$ are constants.
            \label{thm:shiftingone}
        \end{theorem}
        
        \begin{proof}[Proof sketch]
            We decompose distribution $\gP$ and $\gQ$ to $\gP_{s,y}$'s and $\gQ_{s,y}$'s according to their sensitive attribute and label values.
            In \shiftingone, $\Pr(X_o|Y,X_s)$ is fixed, i.e., $\gP_{s,y} = \gQ_{s,y}$, which means $\E_{(X,Y)\sim\gQ_{s,y}} [\ell(h_\theta(X),Y)] = E_{s,y}$ and $\rho_{s,y} = H(\gP_{s,y}, \gQ_{s,y}) = 0$.
            We plug these properties into \Cref{thm:generic}.
            Then, denoting $q_{s,y}$ to $\Pr_{(X,Y)\sim\gQ}[X_s=s,Y=y]$, we can represent the fairness constraint in \Cref{def:fair-dist} as $q_{s_0,y_0} = \left(\sum_{s=1}^S q_{s,y_0}\right)\left(\sum_{y=1}^C q_{s_0,y}\right)$ for any $s_0\in [S]$ and $y_0\in [C]$.
            Next, we parameterize $q_{s,y}$ with $k_sr_y$.
            Such parameterization simplifies the fairness constraint and allow us to prove the convexity of the resulting optimization. 
            Since all the constraints are encoded equivalently, the problem formulation provides a {tight} certification.
            Detailed proof in \Cref{adxsubsec:pf-thm-2}.
        \end{proof}
        
        As \Cref{thm:shiftingone} suggests, we can exploit the expectation information $E_{s,y} = \E_{(X,Y)\sim\gP_{s,y}}[\ell(h_\theta(X),Y)]$ and density information $p_{s,y} = \Pr_{(X,Y)\sim\gP} [X_s=s,Y=y]$ of each $\gP$'s subpopulation to provide a tight fairness certificate in \shiftingone.
        The convex optimization problem with $(S+C)$ variables can be efficiently solved by off-the-shelf packages.
    
 \vspace{-0.7em}
    \subsection{Certified Fairness with \shiftingtwocapital}
        \label{subsec:shifting-two}
\vspace{-0.7em}
    
        For the \shiftingtwo case, we leverage \Cref{thm:generic} and the parameterization trick $q_{s,y} := k_s r_y$ used in \Cref{thm:shiftingone} to reduce \Cref{prob:certified-fairness-shifting-two} to the following constrained optimization.
        
        \begin{lemma}
           Given a distance bound $\rho > 0$, the following constrained optimization, when feasible, provides a \textbf{tight} fairness certificate for \Cref{prob:certified-fairness-shifting-two}:
           \vspace{-0.5em}
           \begin{subequations}
                \label{eq:general_opt_prob_general_shifting_A}
                \begin{small}
                \begin{align}
                    \max_{k_s,r_y,\gQ,\rho_{s,y}} \quad & \sum_{s=1}^S \sum_{y=1}^C k_s r_y \E_{(X,Y)\sim\gQ_{s,y}} [\ell(h_\theta(X), Y)]
                    \label{general_opt_prob_general_shifting_obj_A} \\
                    \mathrm{s.t.} \quad & \sum_{s=1}^S k_s = 1, \quad \sum_{y=1}^C r_y = 1, \quad k_s \ge 0 \quad \forall s \in [S], \quad r_y \ge 0 \quad \forall y \in [C],
                    \label{general_opt_prob_general_shifting_con1_A} \\
                    & \sum_{s=1}^S \sum_{y=1}^C \sqrt{p_{s,y}k_sr_y} (1 - \rho_{s,y}^2) \ge 1 - \rho^2
                    \label{general_opt_prob_general_shifting_con2_A} \\
                    & H(\gP_{s,y}, \gQ_{s,y}) \le \rho_{s,y} \quad \forall s \in [S], y\in [C]
                    \label{general_opt_prob_general_shifting_con3_A},
                \end{align}
                \end{small}
           \end{subequations}
           \label{lemma:shiftingtwo-A}
           where $p_{s,y} := \Pr_{(X,Y)\in \gP}[X_s=s,Y=y]$ is a fixed constant.
           The $\gP_{s,y}$ and $\gQ_{s,y}$ are the subpopulations of $\gP$ and $\gQ$ on the support $\{(X,Y): X\in\gX,X_s=s,Y=y\}$ respectively.
        \end{lemma}
        
        \vspace{-1em}
        \begin{proof}[Proof sketch]
            We show that \Cref{general_opt_prob_general_shifting_con1_A} ensures a parameterization of $q_{s,y} = \Pr_{(X,Y)\in\gQ} [X_s=s,Y=y]$ that satisfies fairness constraints on $\gQ$.
            Then, leveraging \Cref{thm:generic} we prove that the constrained optimization provides a fairness certificate.
            Since all the constraints are either kept or equivalently encoded, this resulting certification is \textit{tight}.
            Detailed proof in \Cref{adxsubsec:pf-lemma-4-2}.
        \end{proof}
        
        \vspace{-1em}
        Now the main obstacle is to solve the non-convex optimization in Problem \ref{eq:general_opt_prob_general_shifting_A}. 
        Here, as the first step, we upper bound the loss of $h_\theta(\cdot)$ within each shifted subpopulation $\gQ_{s,y}$, i.e., upper bound $\E_{(X,Y)\sim\gQ_{s,y}}[\ell(h_\theta(X),Y)]$ in \Cref{general_opt_prob_general_shifting_obj_A}, by \Cref{thm:gramian-bound} in \Cref{adxsec:gramian-bound}~\cite{weber2022certifying}.
        Then, we apply variable transformations to make some decision variables convex.
        For the remaining decision variables, we observe that they are non-convex but bounded.
        Hence, we propose the technique of grid-based sub-problem construction. Concretely, we divide the feasible region regarding non-convex variables into small grids and consider the optimization problem in each region individually. For each sub-problem, we relax the objective by pushing the values of non-convex variables to the boundary of the current grid and then solve the convex optimization sub-problems.
            Concretely, the following theorem states our computable certificate for \Cref{prob:certified-fairness-shifting-two}, with detailed proof in \Cref{adxsubsec:pf-thm-3}.

        \begin{theorem}
           \label{thm:shiftingtwo}
            If for any $s\in [S]$ and $y\in [Y]$, $H(\gP_{s,y},\gQ_{s,y}) \le \bar\gamma_{s,y}$ and $0 \le \sup_{(X,Y)\in \gX\times\gY} \ell(h_\theta(X),Y) \le M$, given a distance bound $\rho > 0$, for any region granularity $T \in \sN_+$, the following expression provides a fairness certificate for \Cref{prob:certified-fairness-shifting-two}:
           \begin{equation}
                \small
                \bar\ell = \max_{\{i_s \in [T]: s\in [S]\}, \{j_y \in [T]: y\in [C]\}} \tC\left( \left\{\left[\frac{i_s-1}{T}, \frac{i_s}{T} \right]\right\}_{s=1}^S, 
                \left\{\left[\frac{j_y-1}{T}, \frac{j_y}{T}\right]\right\}_{y=1}^C \right), \text{ where }
                \label{eq:thm-general-shifting-1}
            \end{equation}
            \begin{subequations}
                \label{eq:general_opt_prob_general_shifting_C}
                \begin{small}
                \begin{align}
                    \hspace{-2em} & \tC\left(\{[\underline{k_s}, \overline{k_s}]\}_{s=1}^S, \{[\underline{r_y}, \overline{r_y}]\}_{y=1}^C\right)   = \max_{x_{s,y}} \sum_{s=1}^S \sum_{y=1}^C
                    \left(\overline{k_s} \overline{r_y} \left( E_{s,y} + C_{s,y} \right)_+
                    + \underline{k_s} \underline{r_y} \left( E_{s,y} + C_{s,y} \right)_- \right. \nonumber \\
                    &  \hspace{1em}
                    \left. + 2\overline{k_s} \overline{r_y} \sqrt{x_{s,y}(1-x_{s,y})} \sqrt{V_{s,y}} - \underline{k_s} \underline{r_y} x_{s,y} (C_{s,y})_+ - \overline{k_s} \overline{r_y} x_{s,y} (C_{s,y})_- \right)
                    \label{general_opt_prob_general_shifting_obj_C} \\
                  \qquad \quad \mathrm{s.t.} \quad & \sum_{s=1}^S \underline{k_s} \le 1, \quad \sum_{s=1}^S \overline{k_s} \ge 1, \quad \sum_{y=1}^C \underline{r_y} \le 1, \quad \sum_{y=1}^C \overline{r_y} \ge 1,
                    \label{general_opt_prob_general_shifting_con1_C} \\
                   \qquad \quad & \sum_{s=1}^S \sum_{y=1}^C \sqrt{p_{s,y}\overline{k_s}\overline{r_y}x_{s,y}} \ge 1 - \rho^2, \quad
                    (1-\bar\gamma_{s,y}^2)^2 \le x_{s,y} \le 1\quad \forall s\in [S], y\in[C],
                    \label{general_opt_prob_general_shifting_con3_C} 
                \end{align}
                \end{small}
           \end{subequations}
           where $(\cdot)_+ = \max\{\cdot,0\}$, $(\cdot)_- = \min\{\cdot,0\}$; $E_{s,y} = \E_{(X,Y)\sim\gP_{s,y}}[\ell(h_\theta(X),Y)]$, $V_{s,y}=\sV_{(X,Y)\sim\gP_{s,y}}[\ell(h_\theta(X),Y)]$, $p_{s,y} = \Pr_{(X,Y)\sim\gP}[X_s=s,Y=y]$, $C_{s,y} = M-E_{s,y}-\frac{V_{s,y}}{M-E_{s,y}}$, and $\bar\gamma_{s,y}^2 = 1 - (1 + (M-E_{s,y})^2/V_{s,y})^{-\frac12}$.
           \Cref{eq:thm-general-shifting-1} only takes $\tC$'s value when it is feasible, and each $\tC$ queried by \Cref{eq:thm-general-shifting-1} is a \textbf{convex optimization}.
        \end{theorem}

        \paragraph{Implications.}
        \Cref{thm:shiftingtwo} provides a fairness certificate for \Cref{prob:certified-fairness-shifting-two} under two assumptions:
        (1)~The loss function is bounded~(by $M$). This assumption holds for several typical losses such as 0-1 loss and JSD loss.
        (2)~The distribution shift between training and test distribution within each subpopulation is bounded by $\bar\gamma_{s,y}$, where $\bar\gamma_{s,y}$ is determined by the model's statistics on  $\gP$.
        In practice, this additional distance bound assumption generally holds, since $\bar\gamma_{s,y} \gg \rho$ for common choices of $\rho$.

        In \Cref{thm:shiftingtwo}, we exploit three types of statistics of $h_\theta(\cdot)$ on $\gP$ to compute the fairness certificates: the expectation $E_{s,y} = \E_{(X,Y)\sim\gP_{s,y}} [\ell(h_\theta(X),Y)]$, the variance $V_{s,y} = \sV_{(X,Y)\sim\gP_{s,y}} [\ell(h_\theta(X),Y)]$, and the density $p_{s,y} = \Pr_{(X,Y)\sim\gP} [X_s=s,Y=y]$, all of which are at the subpopulation level and a high-confidence estimation of them based on finite samples are tractable (\Cref{subsec:finite-sampling}).

        Using \Cref{thm:shiftingtwo}, after determining the region granularity $T$, we can provide a fairness certificate for \Cref{prob:certified-fairness-shifting-two} by solving $T^{SC}$ convex optimization problems, each of which has $SC$ decision variables.
        Note that the computation cost is independent of $h_\theta$, and 
        therefore we can numerically compute the certificate for large DNN models used in practice.
        Specifically, when $S=2$~(binary sensitive attribute) or $C=2$~(binary classification) which is common in the fairness evaluation setting, we can construct the region for only one dimension $k_1$ or $r_1$, and use $1-k_1$ or $1-r_1$ for the other dimension.
        Thus, for the typical setting $S=2,C=2$, we only need to solve $T^2$ convex optimization problems.

        Note that for \Cref{prob:certified-fairness-shifting-one}, our certificate in \Cref{thm:shiftingone} is tight, whereas for \Cref{prob:certified-fairness-shifting-two}, our certificate in \Cref{thm:shiftingtwo} is not.
        This is because in \Cref{prob:certified-fairness-shifting-two}, extra distribution shift exists within each subpopulation, i.e., $\Pr(X_o|Y,X_s)$ changes from $\gP$ to $\gQ$, and to bound such shift, we need to leverage Thm.~2.2 in \cite{weber2022certifying} which has no tightness guarantee.
        Future work providing tighter bounds than \cite{weber2022certifying} can be seamlessly incorporated into our framework to tighten our fairness certificate for \Cref{prob:certified-fairness-shifting-two}.

        \vspace{-0.7em}
    \subsection{Dealing with Finite Sampling Error}
        \label{subsec:finite-sampling}
      \vspace{-0.7em}
        
        In \Cref{subsec:shifting-one} and \Cref{subsec:shifting-two}, we present \Cref{thm:shiftingone} and \Cref{thm:shiftingtwo} that provide computable fairness certificates for \shiftingone and \shiftingtwo scenarios respectively.
        In these theorems, we need to know the quantities related to the training distribution and trained $\gP$ and model $h_\theta(\cdot)$:
        \begin{equation}
            \small
            E_{s,y} = \underset{(X,Y)\sim\gP_{s,y}}{\E} [\ell(h_\theta(X),Y)], 
            V_{s,y} = \underset{(X,Y)\sim\gP_{s,y}}{\sV} [\ell(h_\theta(X),Y)],
            p_{s,y} = \Pr_{(X,Y)\sim\gP} [X_s=s,Y=y].
        \end{equation}
        \Cref{subsec:shifting-two} further requires $C_{s,y}$ and $\bar\gamma_{s,y}$ which are functions of $E_{s,y}$ and $V_{s,y}$.
        However, a practical challenge is that common training distributions do not have an analytical expression that allows us to precisely compute these quantities.
        Indeed, we only have access to a finite number of individually drawn samples, i.e., the training dataset, from  $\gP$.
        Thus, we will provide high-confidence bounds for $E_{s,y}$, $V_{s,y}$, and $p_{s,y}$ in \Cref{lemma:finite-sampling-conf-interval}~(stated in \Cref{adxsubsec:finite-sampling-conf-interval}).  

        For \Cref{thm:shiftingone}, we can replace  $E_{s,y}$ in the objective by the upper bounds of $E_{s,y}$ and replace the concrete quantities of $p_{s,y}$ by interval constraints and the unit constraint $\sum_s \sum_y p_{s,y}=1$, which again yields a convex optimization that can be effectively solved.
        For \Cref{thm:shiftingtwo}, we compute  the confidence intervals of $C_{s,y}$ and $\rho_{s,y}$, then plug in either the lower bounds or the upper bounds to the objective~(\ref{general_opt_prob_general_shifting_obj_C}) based on the coefficient, and finally replace the concrete quantities of $p_{s,y}$ by interval constraints and the unit constraint $\sum_s \sum_y p_{s,y}=1$.
        The resulting optimization is proved to be convex and provides an upper bound for any possible values of $E_{s,y}$, $V_{s,y}$, and $p_{s,y}$ within the confidence intervals. 
        We defer the statement of \Cref{thm:shiftingone} and \Cref{thm:shiftingtwo} considering finite sampling error to \Cref{adxsubsec:statement-finite-sampling}.
        To this point, we have presented our framework for computing high-confidence fairness certificates given access to model $h_\theta(\cdot)$ and a finite number of samples drawn from $\gP$.

  \vspace{-2mm}
\section{Experiments}
\vspace{-2mm}
In this section, we evaluate the certified fairness under both \textit{sensitive shifting} and \textit{general shifting} scenarios on six real-world datasets. We observe that under the sensitive shifting, our certified fairness bound is \textit{tight} (\Cref{sec:exp_sensitive_shifting}); while the bound is less tight under general shifting (\Cref{sec:exp_general_shifting}) which depends on the tightness of generalization bounds within each subpopulation (details in \Cref{subsec:shifting-two}).
In addition, we show that our certification framework can flexibly integrate more constraints on $\gQ$, leading to a tighter fairness certification (\Cref{sec:exp_more_cons}).
Finally, we compare our certified fairness bound with existing distributional robustness bound~\cite{sinha2017certifying} (\cref{sec:exp_comp_w_gen}), since both consider a shifted distribution while our bound is optimized with an additional fairness constraint which is challenging to be directly integrated to the existing distributional robustness optimization. We show that with the fairness constraint and our optimization approach, our bound is much tighter. 

\begin{figure}[thb]
\vspace{-3mm}
\subfigure{
    \rotatebox{90}{\hspace{-1.7em}\footnotesize{Error}}
    \begin{minipage}{0.24\linewidth}
    \centerline{\footnotesize{\quad ADULT}}
 	\vspace{1pt}
 	\centerline{\includegraphics[width=1.0\textwidth]{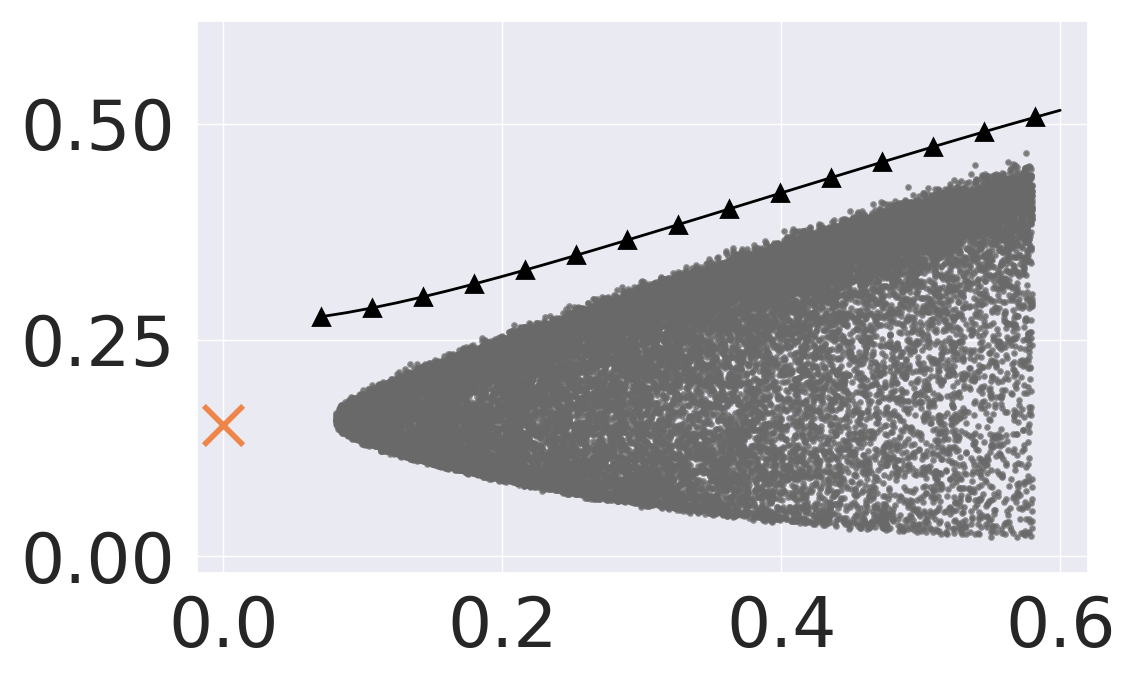}}
 	\vspace{-1.3em}
 \end{minipage}
 \begin{minipage}{0.24\linewidth}
    \centerline{\footnotesize{\quad COMPAS}}
 	\vspace{1pt}
 	\centerline{\includegraphics[width=1.0\textwidth]{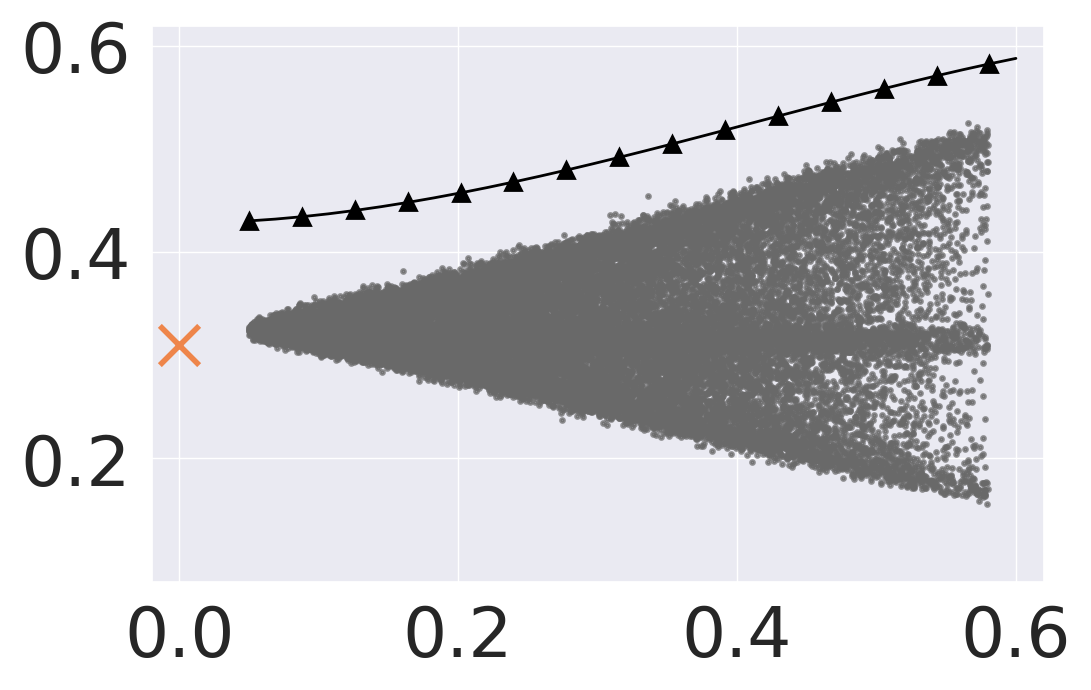}}
 	\vspace{-1.3em}
 \end{minipage}
 \begin{minipage}{0.24\linewidth}
    \centerline{\footnotesize{\quad HEALTH}}
 	\vspace{1pt}
 	\centerline{\includegraphics[width=1.0\textwidth]{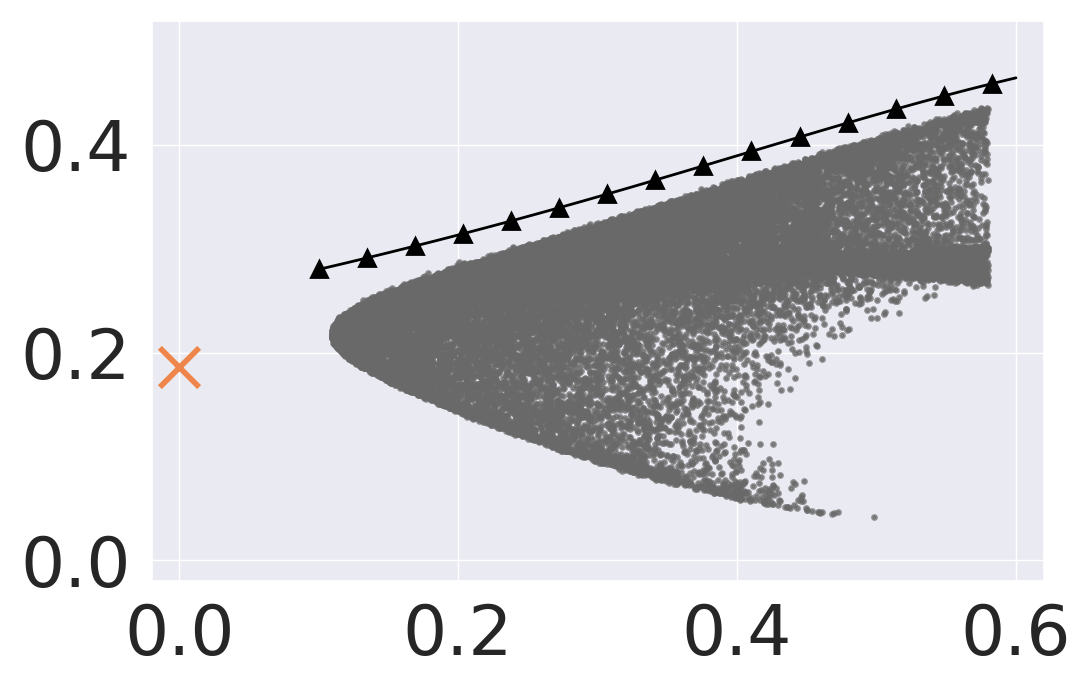}}
 	\vspace{-1.3em}
 \end{minipage}
 \begin{minipage}{0.24\linewidth}
    \centerline{\footnotesize{\quad LAW SCHOOL}}
 	\vspace{1pt}
 	\centerline{\includegraphics[width=1.0\textwidth]{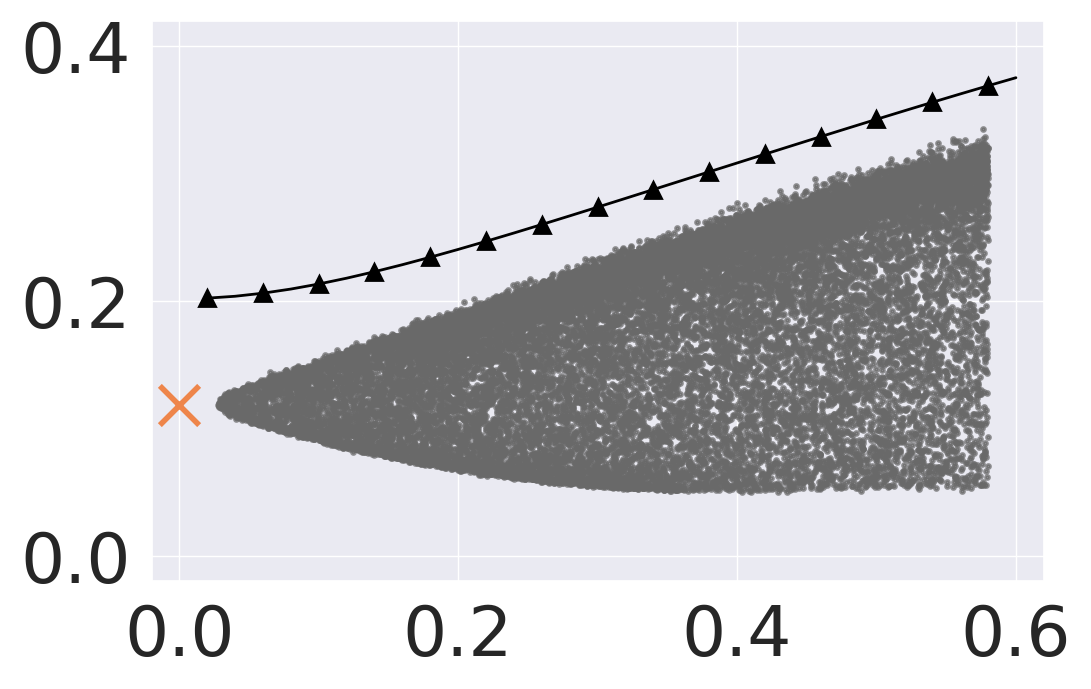}}
 	\vspace{-1.3em}
 \end{minipage}
}
\vspace{-0.7em}
\subfigure{
    \rotatebox{90}{\hspace{-1.1em}\footnotesize{BCE Loss}}
    \begin{minipage}{0.24\linewidth}
 	\centerline{\includegraphics[width=1.0\textwidth]{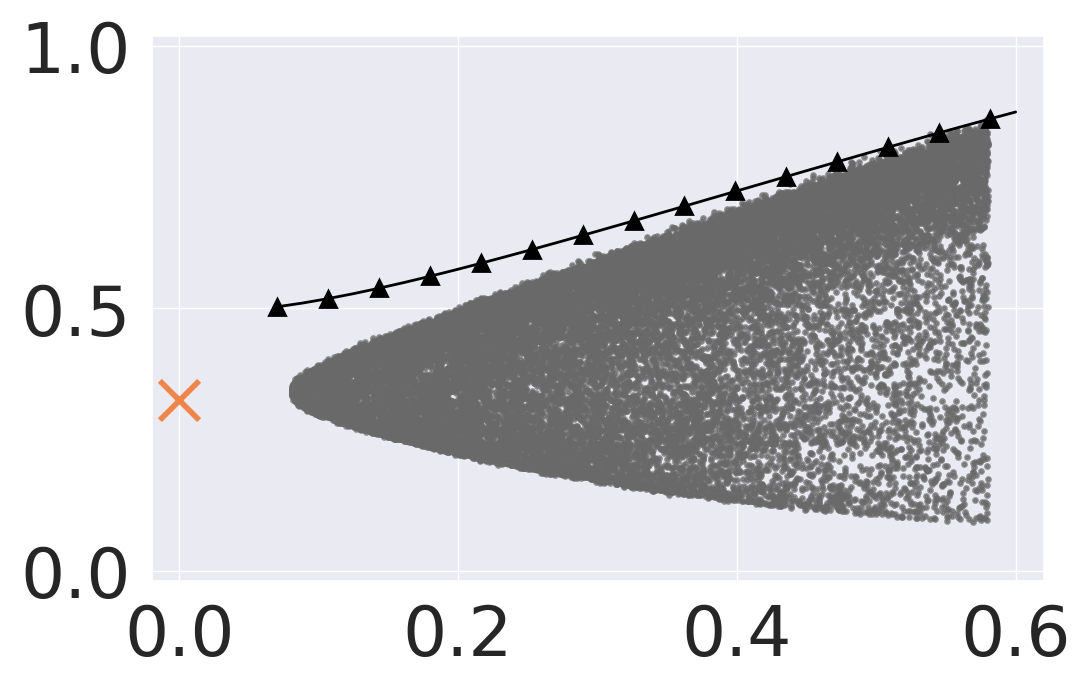}}
  	\vspace{-0.5em}
 	\centerline{\footnotesize{~~~Hellinger Distance}}
 	 \vspace{-0.5em}
 \end{minipage}
\begin{minipage}{0.24\linewidth}
 	\centerline{\includegraphics[width=1.0\textwidth]{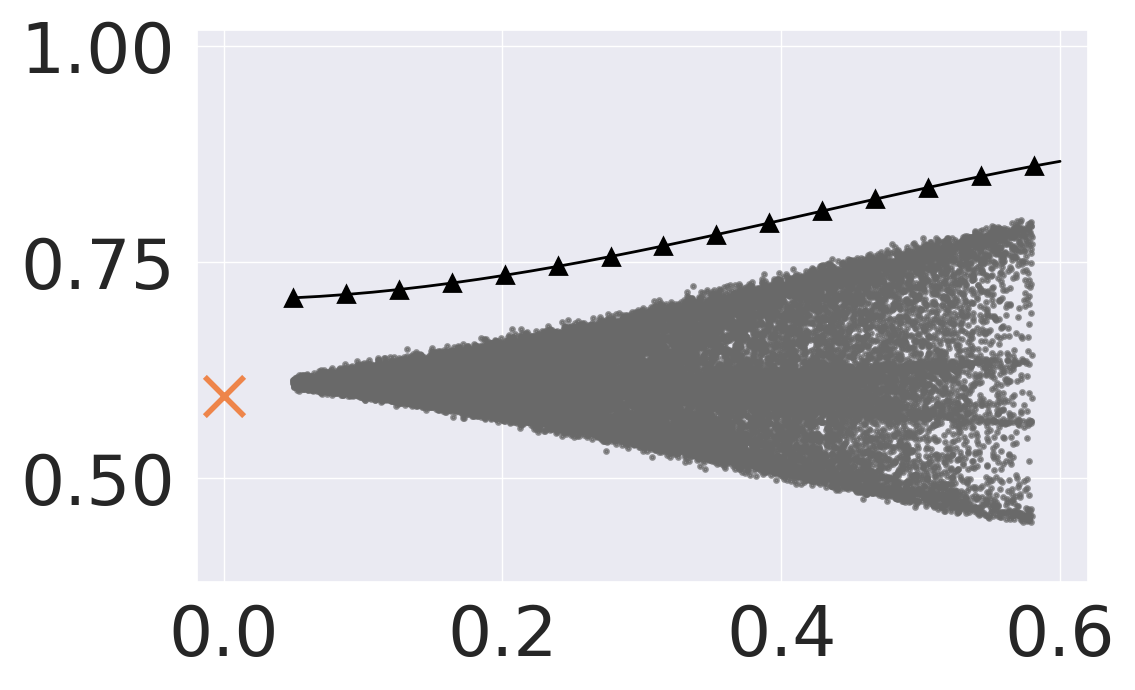}}
\vspace{-0.5em}
 	\centerline{\footnotesize{~~~Hellinger Distance}}
 	 \vspace{-0.5em}
 \end{minipage}
\begin{minipage}{0.24\linewidth}
 	\centerline{\includegraphics[width=1.0\textwidth]{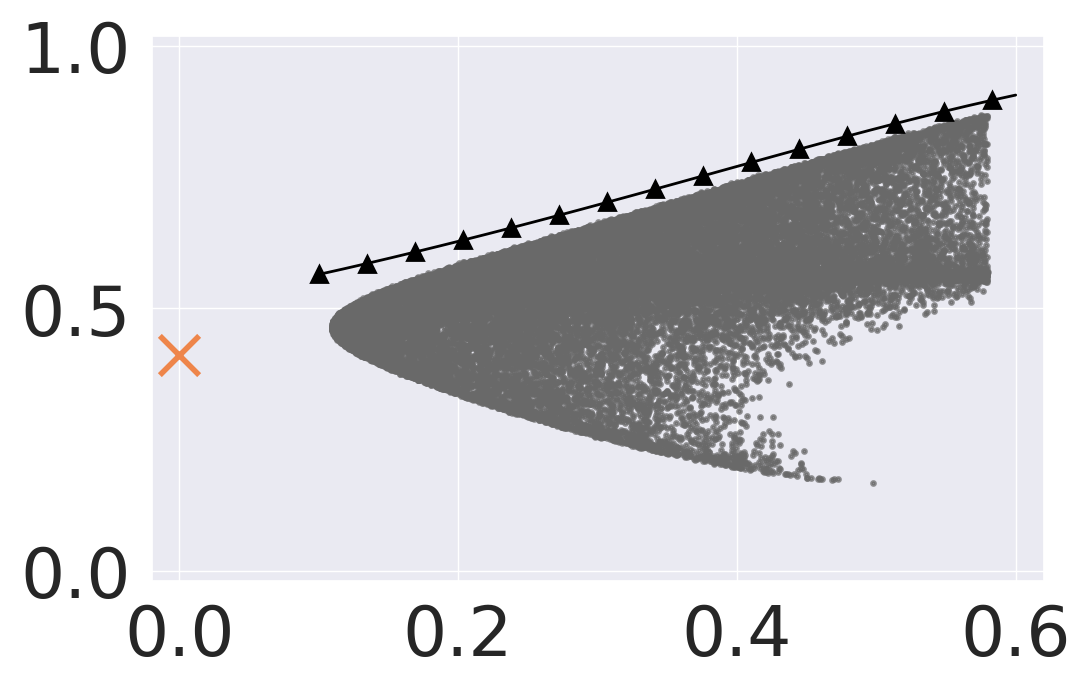}}
\vspace{-0.5em}
 	\centerline{\footnotesize{~~Hellinger Distance}}
 	 \vspace{-0.5em}
 \end{minipage}
\begin{minipage}{0.24\linewidth}
 	\centerline{\includegraphics[width=1.0\textwidth]{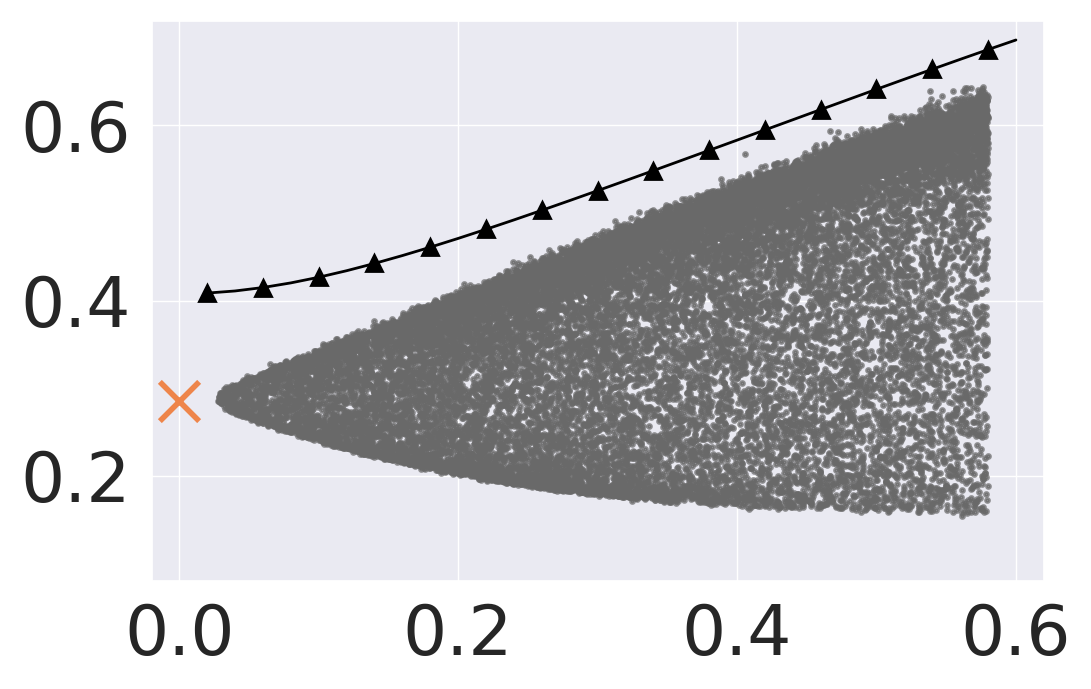}}
\vspace{-0.5em}
 	\centerline{\footnotesize{~~Hellinger Distance}}
 \vspace{-0.5em}
 \end{minipage}
}
\subfigure{
\centerline{\includegraphics[width=0.7\textwidth]{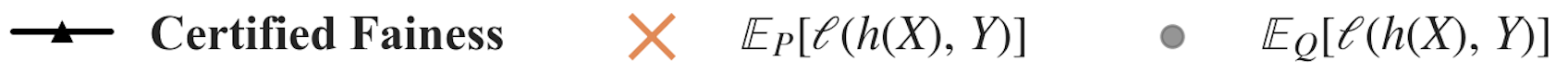}}}
\vspace{-1.8em}
\caption{\small Certified fairness with \shiftingone. 
Grey points are results on generated distributions ($\gQ$) and the black line is our fairness certificate based on \Cref{thm:shiftingone}. We observe that our fairness certificate is usually tight.}
\label{fig:shiftingone}
\vspace{-1.2em}
\end{figure}

\paragraph{Dataset \& Model.} We validate our certified fairness on \textit{six} real-world datasets: Adult~\cite{asuncion2007uci}, Compas~\cite{angwin2016machine},  Health~\cite{healthdataset}, Lawschool~\cite{wightman1998lsac}, Crime~\cite{asuncion2007uci}, and German~\cite{asuncion2007uci}.
Details on the datasets and data processing steps are provided in \Cref{sec:datasets}.
Following the standard setup of fairness evaluation in the literature~\cite{ruoss2020learning,roh2021sample,marity2021does,shekhar2021adaptive}, we consider the scenario that the sensitive attributes and labels take binary values.
The ReLU network composed of 2 hidden layers of size 20 is used for all datasets.

\paragraph{Fairness Certification.} We perform vanilla model training and then 
leverage our fairness certification framework to calculate the fairness certificate.
Concretely, we input the trained model information on $\gP$
and the framework would output the fairness certification for both \shiftingone and \shiftingtwo scenarios following \Cref{thm:shiftingone} and \Cref{thm:shiftingtwo}, respectively.

Code, model, and all experimental data are publicly available at \url{https://github.com/AI-secure/Certified-Fairness}.



\vspace{-0.7em}
\subsection{Certified Fairness with \shiftingonecapital}
\label{sec:exp_sensitive_shifting}
\vspace{-0.7em}

\paragraph{Generating Fair Distributions.}
To evaluate how well our certificates capture the fairness risk in practice, we compare our certification bound with the empirical loss evaluated on randomly generated $30,000$ fairness constrained distributions $\gQ$ shifted from  $\gP$.
The detailed steps for generating fairness constrained distributions $\gQ$ are provided in \Cref{sec:fair_gen}.
Under \shiftingone, since each subpopulation divided by the sensitive attribute and label does not change (\Cref{subsec:formulation}),
we tune only the portion of each subpopulation $q_{s,y}$ satisfying the base rate fairness constraint, and then sample from each subpopulation of $\gP$ individually according to the proportion $q_{s,y}$. In this way, our protocols can generate distributions with different combinations of subpopulation portions. 
If the classifier is biased toward one subpopulation (i.e., it achieves high accuracy in the group but low accuracy in others), the worst-case accuracy on generated distribution is low since the portion of the biased subpopulation in the generated distribution can be low; 
in contrast, a fair classifier which performs uniformly well for each group can achieve high worst-case accuracy (high certified fairness). 
Therefore, we believe that our protocols can demonstrate real-world training distribution bias as well as reflect the model’s unfairness and certification tightness in real-world scenarios.

\vspace{-1em}

\paragraph{Results.}
We report the classification error (Error) and BCE loss as the evaluation metric.
\Cref{fig:shiftingone} illustrates the certified fairness on Adult, Compas, Health, and Lawschool under \shiftingone. More results on two relatively small datasets (Crime, German) are shown in \Cref{sec:app_results}.
From the results, we see that our certified fairness is tight in practice.

\vspace{-0.7em}
\subsection{Certified Fairness with \shiftingtwocapital}
\label{sec:exp_general_shifting}
\vspace{-0.7em}

In the \shiftingtwo scenario, we similarly randomly generate $30,000$ fair distributions $\gQ$ shifted from  $\gP$.
Different from \shiftingone, the distribution conditioned on sensitive attribute $X_s$ and label $Y$ can also change in this scenario.
Therefore, we construct another distribution $\gQ^\prime$ disjoint with  $\gP$ on non-sensitive attributes and mix $\gP$ and $\gQ^\prime$ in each subpopulation individually guided by mixing parameters satisfying fair base rate constraint.
Detailed generation steps are given in \Cref{sec:fair_gen}.
Since the fairness certification for \shiftingtwo requires bounded loss, we select classification error (Error) and Jensen-Shannon loss (JSD Loss) as the evaluation metric.
\Cref{fig:shiftingtwo} illustrates the certified fairness with classification error metric under \shiftingtwo. Results of JSD loss and more results on two relatively small datasets (Crime, German) are  in \Cref{sec:app_results}.

 \vspace{-0.7em}
\subsection{Certified Fairness with Additional Non-Skewness Constraints}
\label{sec:exp_more_cons}
 \vspace{-0.7em}
 
In \Cref{subsec:formulation}, we discussed that to represent different real-world scenarios we can add more constraints such as \Cref{eqn:skewness} to prevent the skewness of $\gQ$, which can be flexibly incorporated into our certificate framework.
Concretely, for \shiftingone, we only need to add one more box constraint\footnote{Note that such modification is only viable when sentive attributes take binary values, which is the typical scenario in the literature of fairness evaluation~\cite{ruoss2020learning,roh2021sample,marity2021does,shekhar2021adaptive}.\label{specify}} $0.5-\Delta_s/2 \le k_s \le 0.5 + \Delta_s/2$ where $\Delta_s$ is a parameter controlling the skewness of $\gQ$, which still guarantees convexity.
For general shifting, we only need to modify the region partition step\footref{specify}, where we split $[0.5-\Delta_s/2, 0.5 + \Delta_s/2]$ instead of $[0,1]$.
The certification results with additional constraints are in \Cref{fig:different_k_sensitive,fig:different_k_general}, which suggests that if the added constraints are strict (i.e., smaller $\Delta_s$), the bound is tighter.
More constraints w.r.t. labels can also be handled by our framework and the corresponding results as well as results on more datasets are in \Cref{sec:app_results_nonskew}.

\begin{figure}[t]
\subfigure{
    \rotatebox{90}{\hspace{-0.9em}\footnotesize{Error}}
    \begin{minipage}{0.24\linewidth}
    \centerline{\footnotesize{\quad ADULT}}
 	\vspace{1pt}
 	\centerline{\includegraphics[width=1.0\textwidth]{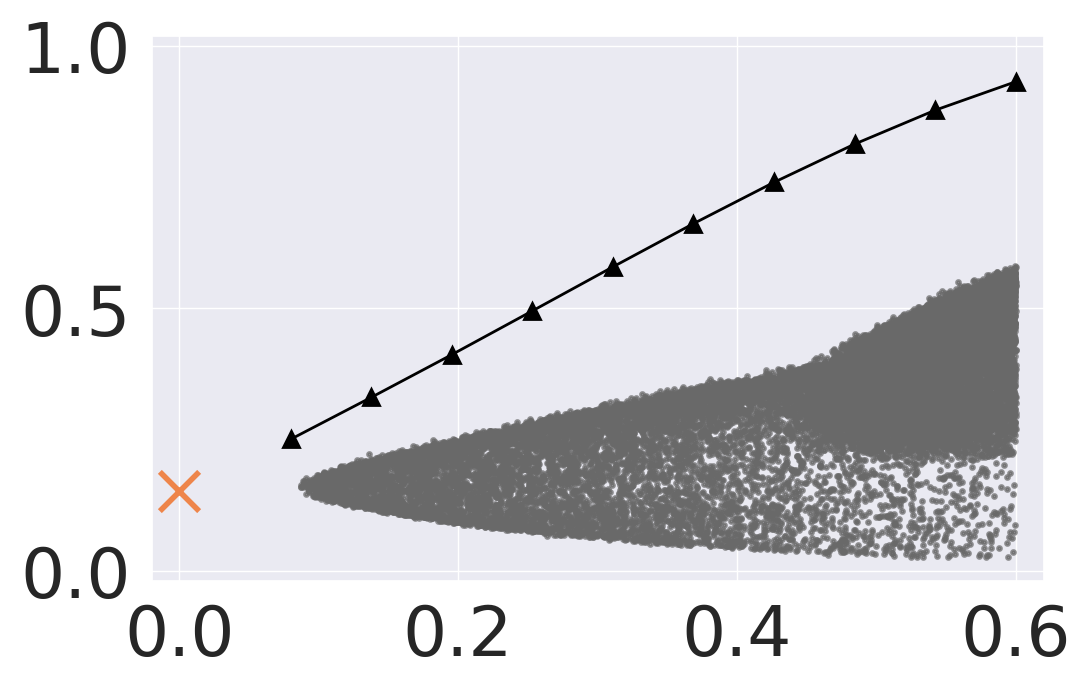}}
 	\vspace{-0.5em}
 	\centerline{\footnotesize{~~~Hellinger Distance}}
 	\vspace{-1.1em}
 \end{minipage}
 \begin{minipage}{0.24\linewidth}
    \centerline{\footnotesize{\quad COMPAS}}
 	\vspace{1pt}
 	\centerline{\includegraphics[width=1.0\textwidth]{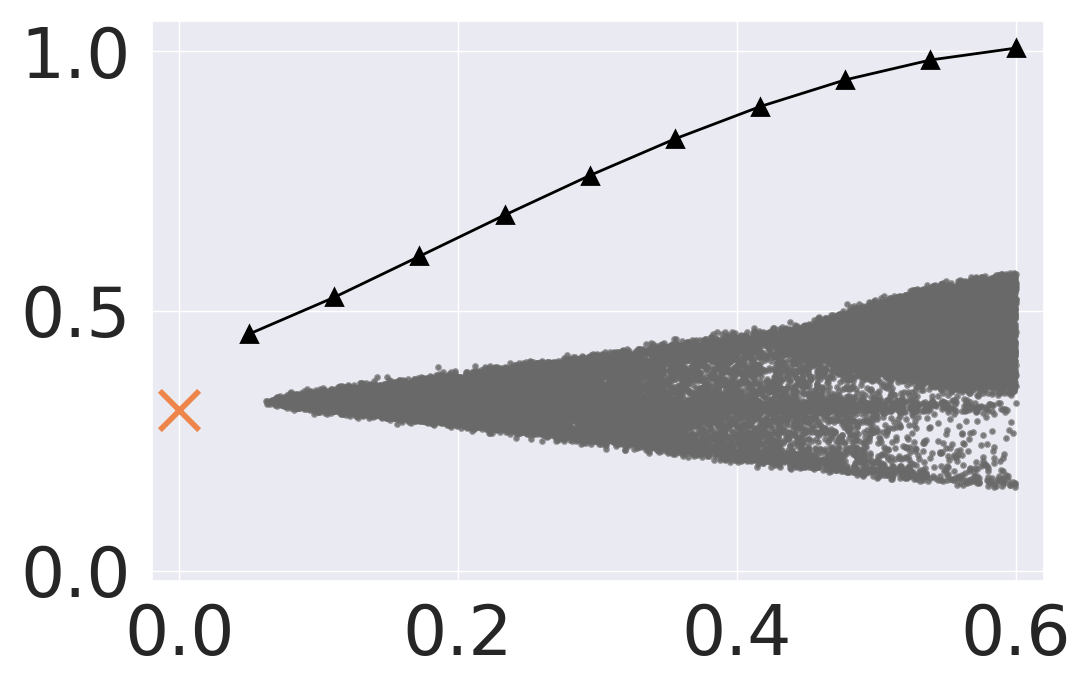}}
 	\vspace{-0.5em}
 	\centerline{\footnotesize{~~~Hellinger Distance}}
 	\vspace{-1.1em}
 \end{minipage}
 \begin{minipage}{0.24\linewidth}
    \centerline{\footnotesize{\quad HEALTH}}
 	\vspace{1pt}
 	\centerline{\includegraphics[width=1.0\textwidth]{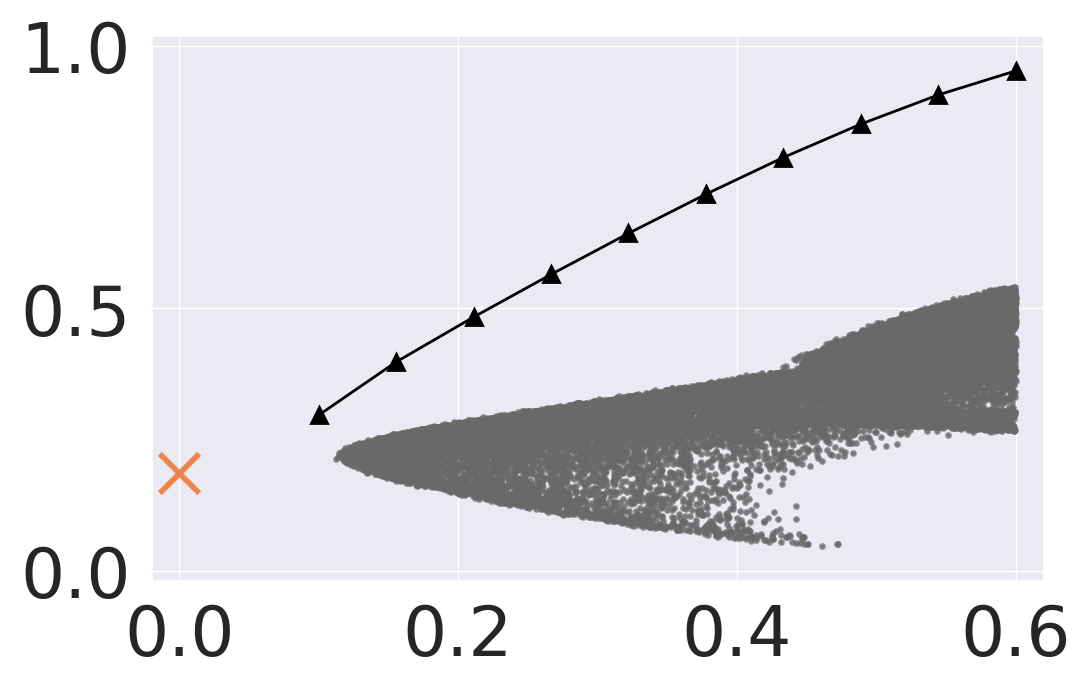}}
 	\vspace{-0.5em}
 	\centerline{\footnotesize{~~~Hellinger Distance}}
 	\vspace{-1.1em}
 \end{minipage}
 \begin{minipage}{0.24\linewidth}
    \centerline{\footnotesize{\quad LAW SCHOOL}}
 	\vspace{1pt}
 	\centerline{\includegraphics[width=1.0\textwidth]{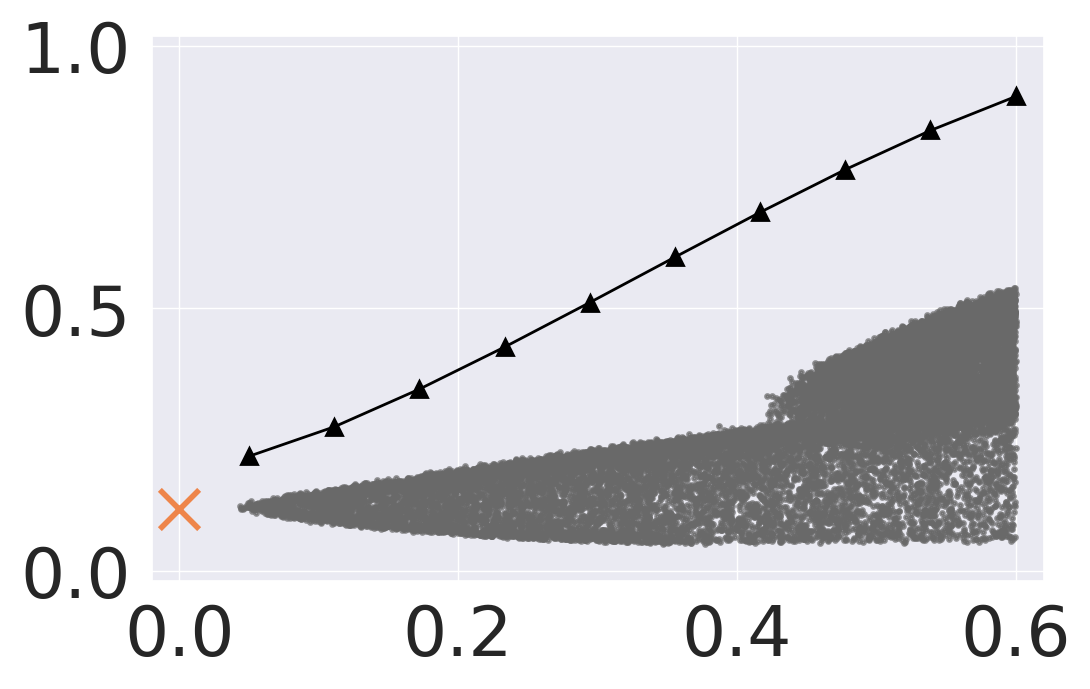}}
 	\vspace{-0.5em}
 	\centerline{\footnotesize{~~~Hellinger Distance}}
 	\vspace{-1.1em}
 \end{minipage}
}

\subfigure{
\centerline{\includegraphics[width=0.7\textwidth]{Figures/legend_1.png}}}
\vspace{-1.9em}
\caption{\small Certified fairness with \shiftingtwo. 
Grey points are results on generated distributions ($\gQ$) and the black line is our fairness certificate based on \Cref{thm:shiftingtwo}. We observe that our fairness certificate is non-trivial.
}
\label{fig:shiftingtwo}
\vspace{-1.5em}
\end{figure}
\begin{figure}[t]
    \centering
    \subfigure[Sentitive Shifting]{
    \includegraphics[width=0.32\linewidth]{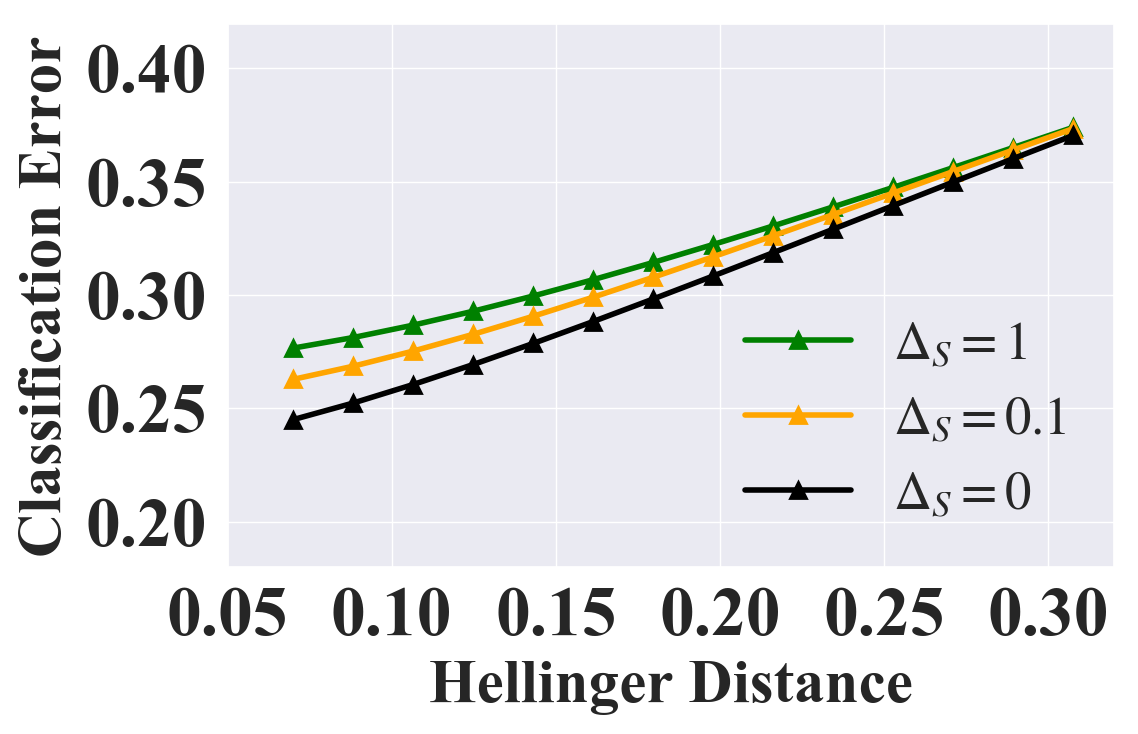}
    \vspace{-1em}
    \label{fig:different_k_sensitive}}\hfill
    \subfigure[General Shifting]{
    \includegraphics[width=0.32\linewidth]{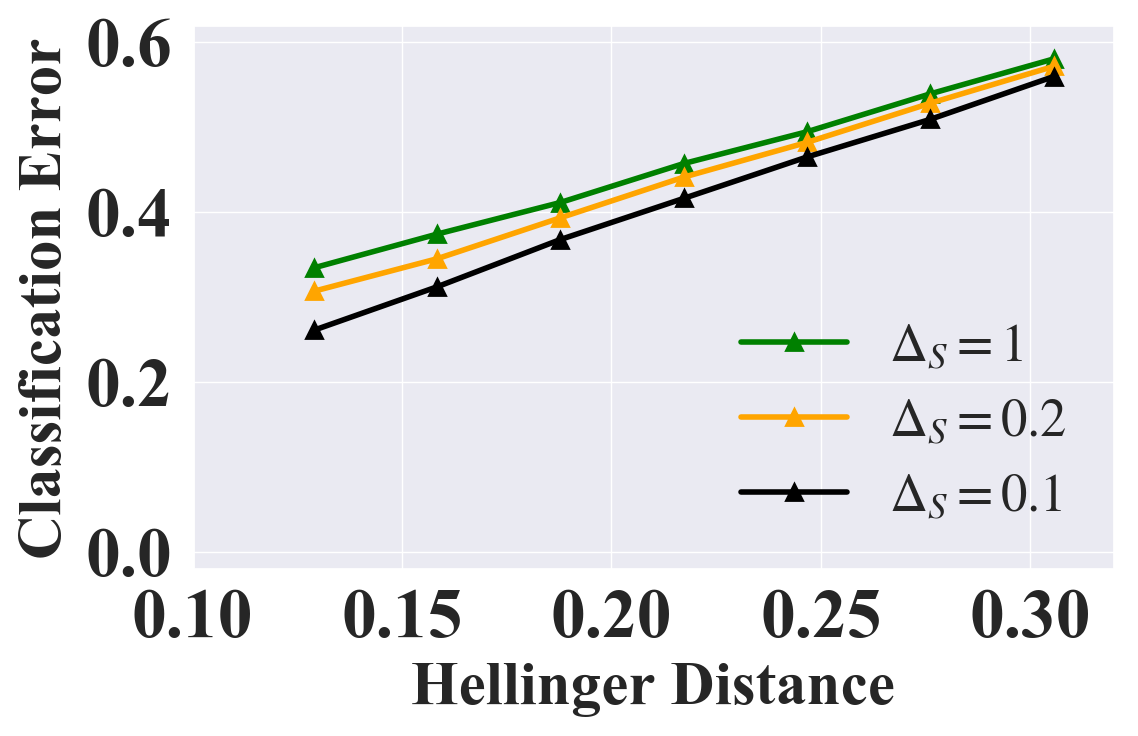}
    \vspace{-1em}
    \label{fig:different_k_general}}\hfill
    \subfigure[Comparison with WRM]{
    \includegraphics[width=0.32\linewidth]{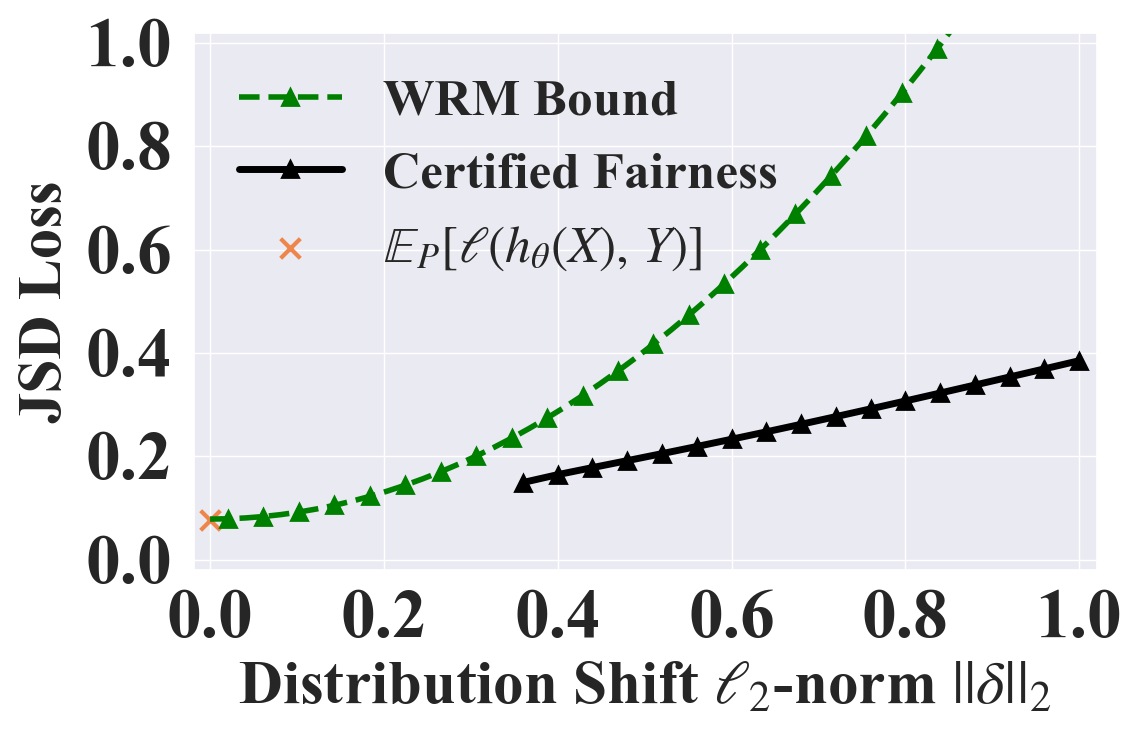} 
    \vspace{-1em}
    \label{fig:comparison}}\hfill
    \vspace{-1.0em}
    \caption{\small
    Certified fairness with additional non-skewness constraints on Adult dataset is shown in (a) (b). $\Delta_s$ controls the skewness of  $\gQ$ ($| \Pr_{(X,Y)\sim \gQ}[X_s=0]-\Pr_{(X,Y)\sim \gQ}[X_s=1] | \le \Delta_s$).
    More analysis in \Cref{sec:exp_more_cons}.
    In (c), we compare our certified fairness bound with the distributional robustness bound \cite{sinha2017certifying}.
    More analysis in \Cref{sec:exp_comp_w_gen}.
    }
    \vspace{-1.7em}
\end{figure}

 	

\vspace{-0.7em}
\subsection{Comparison with Distributional Robustness Bound}
\label{sec:exp_comp_w_gen}
\vspace{-0.7em}
To the best of our knowledge, there is no existing work providing \textit{certified fairness} on the end-to-end model performance. 
Thus, we try to compare our bound with the distributional robustness bound since both consider certain distribution shifts. However, it is challenging to directly integrate the fairness constraints into existing bounds.
Therefore, we compare with the state-of-the-art distributional robustness certification WRM~\cite{sinha2017certifying}, which solves the similar optimization problem as ours except for the fairness constraint.
For fair comparison, we construct a synthetic dataset following~\cite{sinha2017certifying}, on which there is a one-to-one correspondence between the Hellinger and Wasserstein distance used by WRM.
We randomly select one dimension as the sensitive attribute.
Since WRM has additional assumptions on smoothness of models and losses, we use JSD loss and a small ELU network with 2 hidden layers of size 4 and 2 following their setting.
More implementation details are in \Cref{sec:detail_wrm}.
Results in \Cref{fig:comparison} suggest that  1) our certified fairness bound is much tighter than WRM given the additional fairness distribution constraint and our optimization framework; 2) with additional fairness constraint, our certificate problem could be infeasible under very small distribution distances since the fairness constrained distribution $\gQ$ does not exist near the skewed original distribution $\gP$; 3) with the fairness constraint, we provide non-trivial fairness certification bound even when the distribution shift is large. 

  \vspace{-3mm}
\section{Conclusion}
\vspace{-3mm}
In this paper, we provide the first \textit{fairness certification} on end-to-end model performance, based on a fairness constrained distribution which has bounded distribution distance from the training distribution. 
We show that our fairness certification has strong connections with existing fairness notions such as group parity, and we provide an effective framework to calculate the certification under different scenarios. We provide both theoretical and empirical analysis of our fairness certification.


\paragraph{Acknowledgements.}
MK, LL, and BL are partially supported by 
the NSF grant No.1910100, NSF CNS No.2046726, C3 AI, and the Alfred P. Sloan Foundation.
 YL is partially supported by the NSF grants IIS-2143895 and IIS-2040800.

\newpage
    
\bibliographystyle{plain}
\bibliography{references}

\begin{thebibliography}{10}

\bibitem{albarghouthi2017fairsquare}
Aws Albarghouthi, Loris D'Antoni, Samuel Drews, and Aditya~V Nori.
\newblock Fairsquare: probabilistic verification of program fairness.
\newblock {\em Proceedings of the ACM on Programming Languages},
  1(OOPSLA):1--30, 2017.

\bibitem{angwin2016machine}
Julia Angwin, Jeff Larson, Surya Mattu, and Lauren Kirchner.
\newblock Machine bias.
\newblock In {\em Ethics of Data and Analytics}, pages 254--264. Auerbach
  Publications, 2016.

\bibitem{asuncion2007uci}
Arthur Asuncion and David Newman.
\newblock Uci machine learning repository, 2007.

\bibitem{balunovic2021fair}
Mislav Balunovi{\'c}, Anian Ruoss, and Martin Vechev.
\newblock Fair normalizing flows.
\newblock {\em arXiv preprint arXiv:2106.05937}, 2021.

\bibitem{barocas2016big}
Solon Barocas and Andrew~D Selbst.
\newblock Big data's disparate impact.
\newblock {\em Calif. L. Rev.}, 104:671, 2016.

\bibitem{bastani2019probabilistic}
Osbert Bastani, Xin Zhang, and Armando Solar-Lezama.
\newblock Probabilistic verification of fairness properties via concentration.
\newblock {\em Proceedings of the ACM on Programming Languages},
  3(OOPSLA):1--27, 2019.

\bibitem{ben2013robust}
Aharon Ben-Tal, Dick Den~Hertog, Anja De~Waegenaere, Bertrand Melenberg, and
  Gijs Rennen.
\newblock Robust solutions of optimization problems affected by uncertain
  probabilities.
\newblock {\em Management Science}, 59(2):341--357, 2013.

\bibitem{berk2021fairness}
Richard Berk, Hoda Heidari, Shahin Jabbari, Michael Kearns, and Aaron Roth.
\newblock Fairness in criminal justice risk assessments: The state of the art.
\newblock {\em Sociological Methods \& Research}, 50(1):3--44, 2021.

\bibitem{chen2022fairness}
Yatong Chen, Reilly Raab, Jialu Wang, and Yang Liu.
\newblock Fairness transferability subject to bounded distribution shift.
\newblock {\em Advances in Neural Information Processing Systems}, 2022.

\bibitem{choi2020group}
YooJung Choi, Meihua Dang, and Guy Van~den Broeck.
\newblock Group fairness by probabilistic modeling with latent fair decisions.
\newblock {\em arXiv preprint arXiv:2009.09031}, 2020.

\bibitem{chouldechova2017fair}
Alexandra Chouldechova.
\newblock Fair prediction with disparate impact: A study of bias in recidivism
  prediction instruments.
\newblock {\em Big data}, 5(2):153--163, 2017.

\bibitem{creager2019flexibly}
Elliot Creager, David Madras, J{\"o}rn-Henrik Jacobsen, Marissa Weis, Kevin
  Swersky, Toniann Pitassi, and Richard Zemel.
\newblock Flexibly fair representation learning by disentanglement.
\newblock In {\em International conference on machine learning}, pages
  1436--1445. PMLR, 2019.

\bibitem{datta2015automated}
Amit Datta, Michael~Carl Tschantz, and Anupam Datta.
\newblock Automated experiments on ad privacy settings.
\newblock {\em Proceedings on privacy enhancing technologies}, 2015(1):92--112,
  2015.

\bibitem{duchi2019variance}
John Duchi and Hongseok Namkoong.
\newblock Variance-based regularization with convex objectives.
\newblock {\em The Journal of Machine Learning Research}, 20(1):2450--2504,
  2019.

\bibitem{duchi2021statistics}
John~C Duchi, Peter~W Glynn, and Hongseok Namkoong.
\newblock Statistics of robust optimization: A generalized empirical likelihood
  approach.
\newblock {\em Mathematics of Operations Research}, 2021.

\bibitem{edwards2015censoring}
Harrison Edwards and Amos Storkey.
\newblock Censoring representations with an adversary.
\newblock {\em arXiv preprint arXiv:1511.05897}, 2015.

\bibitem{gliner1994reviewing}
Jeffrey~A Gliner.
\newblock Reviewing qualitative research: Proposed criteria for fairness and
  rigor.
\newblock {\em The Occupational Therapy Journal of Research}, 14(2):78--92,
  1994.

\bibitem{hardt2016equality}
Moritz Hardt, Eric Price, and Nati Srebro.
\newblock Equality of opportunity in supervised learning.
\newblock {\em Advances in neural information processing systems}, 29, 2016.

\bibitem{healthdataset}
Kaggle Inc.
\newblock Heritage health prize kaggle.
\newblock \url{https://www.kaggle.com/c/hhp}, 2022.

\bibitem{jin2021transferability}
Xisen Jin, Francesco Barbieri, Brendan Kennedy, Aida~Mostafazadeh Davani,
  Leonardo Neves, and Xiang Ren.
\newblock On transferability of bias mitigation effects in language model
  fine-tuning.
\newblock In {\em Proceedings of the 2021 Conference of the North American
  Chapter of the Association for Computational Linguistics: Human Language
  Technologies, {NAACL-HLT} 2021, Online, June 6-11, 2021}, pages 3770--3783.
  Association for Computational Linguistics, 2021.

\bibitem{john2020verifying}
Philips~George John, Deepak Vijaykeerthy, and Diptikalyan Saha.
\newblock Verifying individual fairness in machine learning models.
\newblock In {\em Conference on Uncertainty in Artificial Intelligence}, pages
  749--758. PMLR, 2020.

\bibitem{kehrenberg2020null}
Thomas Kehrenberg, Myles Bartlett, Oliver Thomas, and Novi Quadrianto.
\newblock Null-sampling for interpretable and fair representations.
\newblock In {\em European Conference on Computer Vision}, pages 565--580.
  Springer, 2020.

\bibitem{kleinberg2016inherent}
Jon Kleinberg, Sendhil Mullainathan, and Manish Raghavan.
\newblock Inherent trade-offs in the fair determination of risk scores.
\newblock {\em arXiv preprint arXiv:1609.05807}, 2016.

\bibitem{lakkaraju2017selective}
Himabindu Lakkaraju, Jon Kleinberg, Jure Leskovec, Jens Ludwig, and Sendhil
  Mullainathan.
\newblock The selective labels problem: Evaluating algorithmic predictions in
  the presence of unobservables.
\newblock In {\em Proceedings of the 23rd ACM SIGKDD International Conference
  on Knowledge Discovery and Data Mining}, pages 275--284, 2017.

\bibitem{lam2016robust}
Henry Lam.
\newblock Robust sensitivity analysis for stochastic systems.
\newblock {\em Mathematics of Operations Research}, 41(4):1248--1275, 2016.

\bibitem{liao2019learning}
Jiachun Liao, Chong Huang, Peter Kairouz, and Lalitha Sankar.
\newblock Learning generative adversarial representations (gap) under fairness
  and censoring constraints.
\newblock {\em arXiv preprint arXiv:1910.00411}, 1, 2019.

\bibitem{liu2019implicit}
Lydia~T Liu, Max Simchowitz, and Moritz Hardt.
\newblock The implicit fairness criterion of unconstrained learning.
\newblock In {\em International Conference on Machine Learning}, pages
  4051--4060. PMLR, 2019.

\bibitem{locatello2019fairness}
Francesco Locatello, Gabriele Abbati, Thomas Rainforth, Stefan Bauer, Bernhard
  Sch{\"o}lkopf, and Olivier Bachem.
\newblock On the fairness of disentangled representations.
\newblock {\em Advances in Neural Information Processing Systems}, 32, 2019.

\bibitem{louizos2015variational}
Christos Louizos, Kevin Swersky, Yujia Li, Max Welling, and Richard Zemel.
\newblock The variational fair autoencoder.
\newblock {\em arXiv preprint arXiv:1511.00830}, 2015.

\bibitem{madras2018learning}
David Madras, Elliot Creager, Toniann Pitassi, and Richard Zemel.
\newblock Learning adversarially fair and transferable representations.
\newblock In {\em International Conference on Machine Learning}, pages
  3384--3393. PMLR, 2018.

\bibitem{marity2021does}
Subha Maity, Debarghya Mukherjee, Mikhail Yurochkin, and Yuekai Sun.
\newblock Does enforcing fairness mitigate biases caused by subpopulation
  shift?
\newblock In M.~Ranzato, A.~Beygelzimer, Y.~Dauphin, P.S. Liang, and J.~Wortman
  Vaughan, editors, {\em Advances in Neural Information Processing Systems},
  volume~34, pages 25773--25784. Curran Associates, Inc., 2021.

\bibitem{maurer2009empirical}
Andreas Maurer and Massimiliano Pontil.
\newblock Empirical bernstein bounds and sample variance penalization.
\newblock {\em arXiv preprint arXiv:0907.3740}, 2009.

\bibitem{McNamara2019costs}
Daniel McNamara, Cheng~Soon Ong, and Robert~C. Williamson.
\newblock Costs and benefits of fair representation learning.
\newblock In {\em Proceedings of the 2019 AAAI/ACM Conference on AI, Ethics,
  and Society}, AIES '19, page 263–270, New York, NY, USA, 2019. Association
  for Computing Machinery.

\bibitem{oneto2020learning}
Luca Oneto, Michele Donini, Massimiliano Pontil, and Andreas Maurer.
\newblock Learning fair and transferable representations with theoretical
  guarantees.
\newblock In {\em 2020 IEEE 7th International Conference on Data Science and
  Advanced Analytics (DSAA)}, pages 30--39, 2020.

\bibitem{paszke2019pytorch}
Adam Paszke, Sam Gross, Francisco Massa, Adam Lerer, James Bradbury, Gregory
  Chanan, Trevor Killeen, Zeming Lin, Natalia Gimelshein, Luca Antiga, et~al.
\newblock Pytorch: An imperative style, high-performance deep learning library.
\newblock {\em Advances in neural information processing systems}, 32, 2019.

\bibitem{peychev2021latent}
Momchil Peychev, Anian Ruoss, Mislav Balunovi{\'c}, Maximilian Baader, and
  Martin Vechev.
\newblock Latent space smoothing for individually fair representations.
\newblock {\em arXiv preprint arXiv:2111.13650}, 2021.

\bibitem{raab2021unintended}
Reilly Raab and Yang Liu.
\newblock Unintended selection: Persistent qualification rate disparities and
  interventions.
\newblock {\em Advances in Neural Information Processing Systems}, 34, 2021.

\bibitem{roh2021sample}
Yuji Roh, Kangwook Lee, Steven Whang, and Changho Suh.
\newblock Sample selection for fair and robust training.
\newblock {\em Advances in Neural Information Processing Systems}, 34, 2021.

\bibitem{ruoss2020learning}
Anian Ruoss, Mislav Balunovic, Marc Fischer, and Martin Vechev.
\newblock Learning certified individually fair representations.
\newblock {\em Advances in Neural Information Processing Systems},
  33:7584--7596, 2020.

\bibitem{sarhan2020fairness}
Mhd~Hasan Sarhan, Nassir Navab, Abouzar Eslami, and Shadi Albarqouni.
\newblock Fairness by learning orthogonal disentangled representations.
\newblock In {\em European Conference on Computer Vision}, pages 746--761.
  Springer, 2020.

\bibitem{segal2021fairness}
Shahar Segal, Yossi Adi, Benny Pinkas, Carsten Baum, Chaya Ganesh, and Joseph
  Keshet.
\newblock Fairness in the eyes of the data: Certifying machine-learning models.
\newblock In {\em Proceedings of the 2021 AAAI/ACM Conference on AI, Ethics,
  and Society}, pages 926--935, 2021.

\bibitem{shekhar2021adaptive}
Shubhanshu Shekhar, Greg Fields, Mohammad Ghavamzadeh, and Tara Javidi.
\newblock Adaptive sampling for minimax fair classification.
\newblock In M.~Ranzato, A.~Beygelzimer, Y.~Dauphin, P.S. Liang, and J.~Wortman
  Vaughan, editors, {\em Advances in Neural Information Processing Systems},
  volume~34, pages 24535--24544. Curran Associates, Inc., 2021.

\bibitem{sinha2017certifying}
Aman Sinha, Hongseok Namkoong, Riccardo Volpi, and John Duchi.
\newblock Certifying some distributional robustness with principled adversarial
  training.
\newblock {\em arXiv preprint arXiv:1710.10571}, 2017.

\bibitem{song2019learning}
Jiaming Song, Pratyusha Kalluri, Aditya Grover, Shengjia Zhao, and Stefano
  Ermon.
\newblock Learning controllable fair representations.
\newblock In {\em The 22nd International Conference on Artificial Intelligence
  and Statistics}, pages 2164--2173. PMLR, 2019.

\bibitem{steerneman1983total}
Ton Steerneman.
\newblock On the total variation and hellinger distance between signed
  measures; an application to product measures.
\newblock {\em Proceedings of the American Mathematical Society},
  88(4):684--688, 1983.

\bibitem{urban2020perfectly}
Caterina Urban, Maria Christakis, Valentin W{\"u}stholz, and Fuyuan Zhang.
\newblock Perfectly parallel fairness certification of neural networks.
\newblock {\em Proceedings of the ACM on Programming Languages},
  4(OOPSLA):1--30, 2020.

\bibitem{weber2022certifying}
Maurice Weber, Linyi Li, Boxin Wang, Zhikuan Zhao, Bo~Li, and Ce~Zhang.
\newblock Certifying out-of-domain generalization for blackbox functions.
\newblock In {\em International Conference on Machine Learning}. PMLR, 2022.

\bibitem{wightman1998lsac}
Linda~F Wightman.
\newblock Lsac national longitudinal bar passage study. lsac research report
  series.
\newblock {\em Law School Admission Council}, 1998.

\bibitem{winter2019learning}
Robin Winter, Floriane Montanari, Frank No{\'e}, and Djork-Arn{\'e} Clevert.
\newblock Learning continuous and data-driven molecular descriptors by
  translating equivalent chemical representations.
\newblock {\em Chemical science}, 10(6):1692--1701, 2019.

\bibitem{yeom2020individual}
Samuel Yeom and Matt Fredrikson.
\newblock Individual fairness revisited: Transferring techniques from
  adversarial robustness.
\newblock {\em arXiv preprint arXiv:2002.07738}, 2020.

\bibitem{zemel2013learning}
Rich Zemel, Yu~Wu, Kevin Swersky, Toni Pitassi, and Cynthia Dwork.
\newblock Learning fair representations.
\newblock In {\em International conference on machine learning}, pages
  325--333. PMLR, 2013.

\bibitem{zhang2020fair}
Xueru Zhang, Ruibo Tu, Yang Liu, Mingyan Liu, Hedvig Kjellstrom, Kun Zhang, and
  Cheng Zhang.
\newblock How do fair decisions fare in long-term qualification?
\newblock {\em Advances in Neural Information Processing Systems},
  33:18457--18469, 2020.

\bibitem{zhao2019conditional}
Han Zhao, Amanda Coston, Tameem Adel, and Geoffrey~J Gordon.
\newblock Conditional learning of fair representations.
\newblock {\em arXiv preprint arXiv:1910.07162}, 2019.

\bibitem{zhao2019inherent}
Han Zhao and Geoff Gordon.
\newblock Inherent tradeoffs in learning fair representations.
\newblock {\em Advances in neural information processing systems}, 32, 2019.

\end{thebibliography}

\newpage
\appendix

\part*{Appendices} 
\DoToC

\newpage

\section{Broader Impact}
    \label{sec:boarder-impact}
    
    This paper aims to calculate a \textit{fairness certificate} under some distributional fairness constraints on the performance of an end-to-end ML model.
    We believe that the rigorous fairness certificates provided by our framework will significantly benefit and advance social fairness in the era of deep learning.
    Especially, such fairness certificate can be directly used to measure the fairness of an ML model regardless the target domain, which means that it will measure the unique property of the model itself with theoretical guarantees, and thus help people understand the risks of existing ML models.
    As a result, the ML community may develop ML training algorithms that explicitly reduce the fairness risks by regularizing on this fairness certificate.
    
    A possible negative societal impact may stem from the misunderstanding or inaccurate interpretation of our fairness certificate.
    As a first step towards distributional fairness certification, we define the fairness through the lens of worst-case performance loss on a fairness constrained distribution.
    This fairness definition may not explicitly imply an absoluate fairness guarantee under some other criterion.
    For example, it does not imply that for any possible individual input, the ML model will give fair prediction.
    We tried our best in \Cref{sec:background-and-formulatiuon} to define the certification goal, and the practitioners may need to understand this goal well to avoid misinterpretation or misuse of our fairness certification.
    
\section{Omitted Background}

    We illustrate omitted background in this appendix.

    \subsection{Hellinger Distance}
        \label{adxsubsec:hellinger}
        
        As illustrated in the beginning of \Cref{sec:method}, our framework uses Hellinger distance to bound the distributional distance.
        A formal definition of Hellinger distance is as below.
        
        \begin{definition}[Hellinger Distance]
        Let $\gP$ and $\gQ$ be distributions on $\gZ := \gX\times\gY$ that are absolutely continuous with respect to a reference measure $\mu$ with $\gP,\gQ \ll \mu$. The Hellinger distance between $\gP$ and $\gQ$ is defined as 
        \begin{equation}
            H(\gP,\gQ) := \sqrt{\frac{1}{2}\int_{\gZ} \left(\sqrt{p(z)} - \sqrt{q(z)} \right)^2 \d\mu(z)}
        \end{equation}
        where $p=\frac{d\gP}{d\mu}$ and $q=\frac{d\gQ}{d\mu}$ are the Radon-Nikodym derivatives of $\gP$ and $\gQ$ with respect to $\mu$, respectively. The Hellinger distance is independent of the choice of the reference measure $\mu$.
        \label{def:hellinger}
    \end{definition}
    
    Representative properties for the Hellinger distance are discussed in \Cref{sec:method}.

    \subsection{Thm. 2.2 in \texorpdfstring{\cite{weber2022certifying}}{Weber et al}}
    \label{adxsec:gramian-bound}
    
    As mentioned in \Cref{subsec:shifting-two}, we leverage Thm. 2.2 from \cite{weber2022certifying} to upper bound the expected loss of $h_\theta(\cdot)$ in each shifted subpopulation $\gQ_{s,y}$.
    Here we restate Thm. 2.2 for completeness.
    
    \begin{theorem}[Thm. 2.2, \cite{weber2022certifying}]
    Let $\gP'$ and $\gQ'$ denote two distributions supported on $\gX\times\gY$,
    suppose that $0\le \ell(h_\theta(X),Y) \le M$,
            then
            \begin{equation}
                \begin{aligned}
                    & \max_{\gQ', \theta} \E_{(X,Y)\sim\gQ'} [\ell(h_\theta(X),Y)] \quad \mathrm{s.t.} \quad H(\gP',\gQ') \le \rho \\
                    \le & \E_{(X,Y)\sim\gP'} [\ell(h_\theta(X),Y)] + 2C_\rho \sqrt{\sV_{(X,Y)\sim\gP'}[\ell(h_\theta(X),Y)]} + \\
                    & \rho^2 (2-\rho^2) 
                    \left(
                    M - \E_{(X,Y)\sim\gP'}[\ell(h_\theta(X),Y)] - \dfrac{\sV_{(X,Y)\sim\gP'}[\ell(h_\theta(X),Y)]}{M - \E_{(X,Y)\sim\gP'}[\ell(h_\theta(X),Y)]}
                    \right),
                \end{aligned}
            \end{equation}
            where $C_\rho = \sqrt{\rho^2 (1 - \rho^2)^2 (2 - \rho^2)}$, for any given distance bound $\rho > 0$ that satisfies
            \begin{equation}
                \rho^2 \le 1 - \left( 1 + \dfrac{(M - \E_{(X,Y)\sim\gP'}[\ell(h_\theta(X),Y)])^2}{\sV_{(X,Y)\sim\gP'}[\ell(h_\theta(X),Y)]} \right)^{-1/2}.
            \end{equation}
            \label{thm:gramian-bound}
    \end{theorem}
    
    This theorem provides a closed-form expression that upper bounds the mean loss of $h_\theta(\cdot)$ on shifted distribution~(namely $\E_{\gQ'} [\ell(h_\theta(X),Y)]$), given bounded Hellinger distance $H(\cP,\gQ)$ and the mean $E$ and variance $V$ of loss on $\gP$ under two mild conditions: 
    (1)~the function is positive and bounded~(denote the upper bound by $M$); and (2)~the distance $H(\gP,\gQ)$ is not too large~(specifically, $H(\gP,\gQ)^2 \le \bar\gamma^2 := 1 - (1+(M-E)^2/V)^{-\frac12}$).
    Since \Cref{thm:gramian-bound} holds for arbitrary models and loss functions $\ell(h_\theta(\cdot),\cdot)$ as long as the function value is bounded by $[0, M]$, using \Cref{thm:gramian-bound} allows us to provide a generic and succinct fairness certificate in \Cref{thm:shiftingtwo} for \shiftingtwo case that holds for generic models including DNNs without engaging complex model architectures.
    Indeed, we only need to query the mean and variance under $\gP$ for the given model to compute the certificate in \Cref{thm:gramian-bound}, and this benefit is also inherited by our certification framework expressed by \Cref{thm:shiftingtwo}.
    Note that there is no tightness guarantee for this bound yet, which is also inherited by our \Cref{thm:shiftingtwo}.
    

\section{Proofs of Main Results}

    \label{adxsec:main-result}
    
    This appendix entails the complete proofs for \Cref{prop:fairness}, \Cref{thm:generic}, \Cref{thm:shiftingone}, \Cref{lemma:shiftingtwo-A}, and \Cref{thm:shiftingtwo} in the main text.
    For complex proofs such as that for \Cref{thm:shiftingtwo}, we also provide high-level illustration before going into the formal proof.

    \subsection{Proof of \texorpdfstring{\Cref{prop:fairness}}{Proposition 1}}
    \label{proof-fairness-notion}

        \begin{proof}[Proof of \Cref{prop:fairness}]
            Since each term $\Pr_{(X,Y)\sim\gQ}[h_{\theta}(X) \neq Y|Y = y, X_s = i ]$ is within $[0, \epsilon]$,
            we consider two cases: $y \neq 1$ and $y = 1$.
            If $y \neq 1$, $\Pr_{(X,Y)\sim\gQ} [h_\theta(X) = 1 | Y=y, X_s=i] \le \Pr_{(X,Y)\sim\gQ} [h_\theta(X)\neq Y | Y=y, X_s=i] \le \epsilon$ and so will be their differences for $X_s=i$ and $X_s=j$.
            If $y = 1$, $\Pr_{(X,Y)\sim\gQ} [h_\theta(X) = 1 | Y=y, X_s=i] = 1 - \Pr_{(X,Y)\sim\gQ} [h_\theta(X) \neq Y | Y=y, X_s=i] \in [1-\epsilon,1]$, and also the differences for $X_s=i$ and $X_s=j$ are always within $\epsilon$.
            This proves $\epsilon$-EO.

        Now consider DP.
        We notice that for any $a$,
        \begin{equation}
            \begin{small}
                \Pr_{(X,Y)\sim\gQ} [h_\theta(X)=1|X_s=a]
                =
                \sum_{y=1}^C \Pr_{(X,Y)\sim\gQ} [h_\theta(X)=1|Y=y,X_s=a] \cdot \Pr_{(X,Y)\sim\gQ} [Y=y|X_s=a].
            \end{small}
        \end{equation}
        Thus,
        \begin{align*}
            &  \Big|\Pr_{(X,Y)\sim\gQ}[h_{\theta}(X) = 1|X_s = i]-
        \Pr_{(X,Y)\sim\gQ}[h_{\theta}(X) = 1|X_s = j]\Big| \\
            \overset{(*)}{\le} & \sum_{y=1}^C \Big|\Pr_{(X,Y)\sim\gQ}[h_{\theta}(X) = 1|Y=y,X_s = i] - \Pr_{(X,Y)\sim\gQ}[h_{\theta}(X) = 1|Y=y,X_s = j]\Big| \\ & \hspace{20em} \cdot \Pr_{(X,Y)\sim\gQ}[Y=y|X_s = i] \\
            \le & \sum_{y=1}^C \epsilon \Pr_{(X,Y)\sim\gQ}[Y=y|X_s = i] = \epsilon
        \end{align*}
        which proves $\epsilon$-DP,
        where $(*)$ leverages the fair base rate property of $\gQ$ which gives $\Pr_{(X,Y)\sim\gQ} [Y=y|X_s=i] = \Pr_{(X,Y)\sim\gQ} [Y=y|X_s=j]$.

        \end{proof}

    \subsection{Proof of \texorpdfstring{\Cref{thm:generic}}{Theorem 1}}
        \label{adxsubsec:pf-thm-1}
        
        \begin{proof}[Proof of \Cref{thm:generic}]
            We first prove the key \cref{eq:hellinger-prop}.
                \begin{align}
                    H(\gP,\gQ) \le \rho &
                    \iff H^2(\gP,\gQ) \le \rho^2 \nonumber \\
                    & \iff \dfrac12 \int_{\gZ} \left(\sqrt{p(z)} - \sqrt{q(z)}\right)^2 \d \mu(z) \le \rho^2 \nonumber \\
                    & \iff \dfrac12 \left(\int_{\gZ} p(z)\d\mu(z) + \int_{\gZ} q(z)\d\mu(z)\right) - \int_{\gZ} \sqrt{p(z)q(z)} \d\mu(z) \le \rho^2 \nonumber \\
                    & \iff \int_{\gZ} \sqrt{p(z)q(z)} \d\mu(z) \ge 1 - \rho^2 \nonumber \\
                    & \iff \sum_{i=1}^N \int_{\gZ_i} \sqrt{p_iq_i} \cdot \sqrt{p_i(z)q_i(z)} \d \mu(z) \ge 1 - \rho^2 \nonumber \\
                    & \iff \sum_{i=1}^N \sqrt{p_iq_i} \left(1 - H^2(\gP_i,\gQ_i)  \right) \ge 1 - \rho^2 
                    \label{eq:helling-prop-pf}
                \end{align}
            where $p_i(\cdot)$ and $q_i(\cdot)$ are density functions of subpopulation distributions $\gP_i$ and $\gQ_i$ respectively.
            
            Then, we show that any feasible solution of \Cref{eq:generic-prob} satisfies the constraints in \Cref{cons-opt-generic}.
            We let $\gQ^\star$ and $\theta^\star$ denote a feasible solution of \Cref{eq:generic-prob}, i.e.,
            \begin{equation}
                H(\gP,\gQ^\star) \le \rho,
                \quad
                e_j(\gP,h_{\theta^\star}) \le v_j \, \forall j \in [L],
                \quad
                g_j(\gQ^\star) \le u_j \, \forall j \in [M].
            \end{equation}
            We let $\{q_i^\star\}_{i=1}^N$ denote the proportions of $\gQ^\star$ within each support partition $\gZ_i$, and $\{\gQ_i^\star\}_{i=1}^N$ the $\gQ^\star$ in each subpopulation. 
            By \Cref{eq:helling-prop-pf}, we have $1 - \rho^2 - \sum_{i=1}^N \sqrt{p_iq_i^\star} (1 - \rho_i^2) \le 0$ where $\rho_i = H^2(\gP_i, \gQ_i^\star)$.
            Note that by definition, $\sum_{i=1}^N q_i^\star = 1$ and $\forall i\in [N], q_i^\star \ge 0, \rho_i \ge 0$.
            Furthermore, 
            by the implication relations stated in \Cref{thm:generic}, for any $j\in [L]$, $e_j'(\{\gP_i\}_{i=1}^N, \{p_i\}_{i=1}^N, h_{\theta^\star}) \le v_j'$;
            and for any $j\in[M]$, $g_j'(\{\gQ_i^\star\}_{i=1}^N, \{q_i^\star\}_{i=1}^N) \le u_j'$.
            To this point, we have shown $\gQ^\star$ and $\theta^\star$ satisfy all constraints in \Cref{cons-opt-generic}, i.e., $\gQ^\star$ and $\theta^\star$ is a feasible solution of \Cref{cons-opt-generic}.
            Since \Cref{cons-opt-generic} expresses the optimal~(maximum) solution, \Cref{cons-opt-generic}~(in \Cref{thm:generic}) $\ge$ \Cref{eq:generic-prob}.
        \end{proof}
    
    \subsection{Proof of \texorpdfstring{\Cref{thm:shiftingone}}{Theorem 2}}
        \label{adxsubsec:pf-thm-2}
        
        \begin{proof}[Proof of \Cref{thm:shiftingone}]
            The proof of \Cref{thm:shiftingone} is composed of three parts:
            (1)~the optimization problem provides a fairness certificate for \Cref{prob:certified-fairness-shifting-one};
            (2)~the certificate is tight;
            and (3)~the optimization problem is convex.
            
            \begin{enumerate}[label={(\bfseries\arabic*)},leftmargin=*]
                \item 
                
                Suppose the maximum of \Cref{prob:certified-fairness-shifting-one} is attained with the test distribution $\gQ^\star$ in the \shiftingone setting,
                then we decompose both $\gP$ and $\gQ^\star$ according to both the sensitive attribute and the label:
                \begin{equation}
                    \gP = \sum_{s=1}^S \sum_{y=1}^C p_{s,y} \gP_{s,y},
                    \quad
                    \gQ^\star = \sum_{s=1}^S \sum_{y=1}^C q^\star_{s,y} \gQ^\star_{s,y}.
                \end{equation}
                Since $\gQ^\star$ is a fair base rate distribution, for any $i,j\in [S]$, $b^{\gQ^\star}_{i,y} = b^{\gQ^\star}_{j,y}$ where $b^{\gQ^\star}_{s,y} = \Pr_{(X,Y)\sim\gQ^\star} [Y=y|X_s=s]$.
                As a result,
                $\Pr_{(X,Y)\sim\gQ^\star} [Y=y|X_s=s] = \Pr_{(X,Y)\sim\gQ^\star} [Y=y]$.
                Now we define
                \begin{equation}
                    k_s^\star := \Pr_{(X,Y)\sim\gQ^\star} [X_s=s],
                    \quad
                    r^\star_y := \Pr_{(X,Y)\sim\gQ^\star} [Y=y],
                \end{equation}
                and then
                \begin{equation}
                    q_{s,y}^\star = \Pr_{(X,Y)\sim\gQ^\star} [X_s=s,Y=y]
                    = \Pr_{(X,Y)\sim\gQ^\star} [X_s=s] \cdot \Pr_{(X,Y)\sim\gQ^\star} [Y=y|X_s=s] = k_s^\star r^\star_y.
                    \label{pf-thm-2-eq1}
                \end{equation}
                By the distance constraint in \Cref{prob:certified-fairness-shifting-one}~(namely $H(\gP,\gQ^\star) \le \rho$) and \Cref{eq:helling-prop-pf}, we have
                \begin{equation}
                    \sum_{s=1}^S \sum_{y=1}^C \sqrt{p_{s,y} q^\star_{s,y}} \left(1 - H^2(\gP_{s,y}, \gQ^\star_{s,y}) \right) \ge 1 - \rho^2.
                \end{equation}
                Since there is only \shiftingone, $H^2(\gP_{s,y},\gQ^\star_{s,y}) = 0$, given \Cref{pf-thm-2-eq1}, we have
                \begin{equation}
                    \sum_{s=1}^S \sum_{y=1}^C \sqrt{p_{s,y} k_s^\star r_y^\star} \ge 1 - \rho^2.
                \end{equation}
                Now, we can observe that the $k_s^\star$ and $r_y^\star$ induced by $\gQ^\star$ satisfy all constraints of \Cref{prob:certified-fairness-shifting-one}.
                For the objective, 
                    \begin{align*}
                    &\text{Objective in \Cref{thm:shiftingone}} \\
                    =& \sum_{s=1}^S \sum_{y=1}^C k_s^\star r_s^\star \E_{(X,Y)\sim\gP_{s,y}} [\ell(h_\theta(X),Y)] \\
                    = & \sum_{s=1}^S \sum_{y=1}^C q^\star_{s,y} \E_{(X,Y)\sim\gQ^\star_{s,y}} [\ell(h_\theta(X),Y)] & \text{(by \Cref{pf-thm-2-eq1} and $H^2(\gP_{s,y},\gQ^\star_{s,y}) = 0$)} \\
                    = & \E_{(X,Y)\sim\gQ^\star} [\ell(h_\theta(X),Y)] \\
                    = & \text{Optimal value of \Cref{prob:certified-fairness-shifting-one}}.
                    \end{align*}
                Therefore, the \emph{optimal} value of \Cref{thm:shiftingone} will be larger or equal to the optimal value of \Cref{prob:certified-fairness-shifting-one} which concludes the proof of the first part.
                
                \item 
                Suppose the optimal value of \Cref{thm:shiftingone} is attained with $k_s^\star$ and $r_y^\star$.
                We then construct $\gQ^\star = \sum_{s=1}^S \sum_{y=1}^C k_s^\star r_y^\star \gP_{s,y}$.
                We now inspect each constraint of \Cref{prob:certified-fairness-shifting-one}.
                The constraint $\dist(\gP,\gQ^\star) \le \rho$ is satisfied because $1-\rho^2 - \sum_{s=1}^S \sum_{y=1}^C \sqrt{p_{s,y}k_s^\star r_y^\star} \le 0$ is satisfied as a constraint of \Cref{thm:shiftingone}.
                Apparently, $\gP_{s,y} = \gQ^\star_{s,y}$.
                Then, $\gQ^\star$ is a fair base rate distribution because 
                \begin{equation}
                    b_{s,y}^{\gQ^\star} = \Pr_{(X,Y)\sim\gQ^\star} [Y=y|X_s=s] = \dfrac{k_s^\star r_y^\star}{k_s^\star} = r_y^\star
                    \label{pf-thm2-eq3}
                \end{equation}
                is a constant across all $s\in [S]$.
                Thus, $\gQ^\star$ satisfies all constraints of \Cref{prob:certified-fairness-shifting-one} and 
                \begin{equation}
                    \begin{aligned}
                        & \text{Optimal objective of \Cref{prob:certified-fairness-shifting-one}} \\
                        \ge & \E_{(X,Y)\sim\gQ^\star} [\ell(h_\theta(X),Y)] \\
                        = & \sum_{s=1}^S \sum_{y=1}^C k_s^\star r_y^\star \E_{(X,Y)\sim\gP_{s,y}} [ \ell(h_\theta(X),Y)] \\
                        = & \sum_{s=1}^S \sum_{y=1}^C k_s^\star r_y^\star E_{s,y} = \text{Optimal objective of \Cref{thm:shiftingone}}.
                    \end{aligned}
                \end{equation}
                Combining with the conclusion of the first part, we know optimal values of \Cref{thm:shiftingone} and \Cref{prob:certified-fairness-shifting-one} match, i.e., the certificate is tight.
                
                \item Inspecting the problem definition in \Cref{thm:shiftingone}, we find the objective and all constraints but the last one are linear.
                Therefore, to prove the convexity of the optimization problem, we only need to show that the last constraint
                \begin{equation}
                    1 - \rho^2 - \sum_{s=1}^S \sum_{y=1}^C \sqrt{p_{s,y}k_sr_y} \le 0
                    \label{pf-thm2-eq2}
                \end{equation}
                is convex with respect to $k_s$ and $r_y$.
                Given two arbitrary feasible pairs of $k_s$ and $r_y$ satisfying \Cref{pf-thm2-eq2}, namely $(k_s^a, r_y^a)$ and $(k_s^b, r_y^b)$, we only need to show that $(k_s^m, r_y^m)$ also satisfies \Cref{pf-thm2-eq2}, where $k_s^m = (k_s^a + k_s^b)/2$, $r_y^m = (r_y^a + r_y^b)/2$.
                Indeed,
                    \begin{align*}
                        & 1 - \rho^2 - \sum_{s=1}^S \sum_{y=1}^C \sqrt{p_{s,y} k_s^m r_y^m} \\
                        = & 1 - \rho^2 - \dfrac12 \sum_{s=1}^S \sum_{y=1}^C \sqrt{p_{s,y}} \cdot \sqrt{k_s^a + k_s^b} \cdot \sqrt{r_y^a + r_y^b} \\
                        \le & 1 - \rho^2 - \dfrac12 \sum_{s=1}^S \sum_{y=1}^C \sqrt{p_{s,y}} \cdot \left(\sqrt{k_s^a r_y^a} + \sqrt{k_s^b r_y^b}\right) & \text{(Cauchy's inequality)} \\
                        = & \dfrac12 \left( 1 - \rho^2 \sum_{s=1}^{S} \sum_{y=1}^C \sqrt{p_{s,y}k_s^a r_y^a} \right) + \dfrac12 \left( 1 - \rho^2 \sum_{s=1}^{S} \sum_{y=1}^C \sqrt{p_{s,y}k_s^b r_y^b} \right) \\
                        \le & 0.
                    \end{align*}
            \end{enumerate}
        \end{proof}
    
    \subsection{Proof of \texorpdfstring{\Cref{lemma:shiftingtwo-A}}{Lemma 3.1}}
        \label{adxsubsec:pf-lemma-4-2}
        
        \begin{proof}[Proof of \Cref{lemma:shiftingtwo-A}]
            The proof of \Cref{lemma:shiftingtwo-A} is composed of two parts: (1)~the optimization problem provides a fairness certificate for \Cref{prob:certified-fairness-shifting-two}; and (2)~the certificate is tight.
            The high-level proof sketch is similar to the proof of \Cref{thm:shiftingone}.
            
            \begin{enumerate}[label={(\bfseries\arabic*)},leftmargin=*]
                \item 
                Suppose that the maximum of \Cref{prob:certified-fairness-shifting-two} is attained with the test distribution $\gQ^\star$ under the \shiftingtwo setting, then we decompose both $\gP$ and $\gQ^\star$ according to both the sensitive attribute and the label:
                \begin{equation}
                    \gP = \sum_{s=1}^S \sum_{y=1}^C p_{s,y}\gP_{s,y},
                    \quad
                    \gQ^\star = \sum_{s=1}^S \sum_{y=1}^C q^\star_{s,y} \gQ^\star_{s,y}.
                \end{equation}
                Unlike \shiftingone setting, in \shiftingtwo setting, here the subpopulation of $\gQ^\star$ is $\gQ^\star_{s,y}$ instead of $\gP_{s,y}$ due to the existence of distribution shifting within each subpopulation.
                
                Following the same argument as in the first part proof of \Cref{thm:shiftingone}, since $\gQ^\star$ is a fair base rate distribution, we can define
                \begin{equation}
                    k_s^\star := \Pr_{(X,Y)\sim\gQ^\star} [X_s = s],
                    \quad
                    r_y^\star := \Pr_{(X,Y)\sim\gQ^\star} [Y=y],
                \end{equation}
                and write 
                \begin{equation}
                    \gQ^\star := \sum_{s=1}^S \sum_{y=1}^C k_s^\star r_y^\star \gQ^\star_{s,y}
                \end{equation}
                since $q_{s,y}^\star = k_s^\star r_y^\star$.
                We also define $\rho_{s,y}^\star = H(\gP_{s,y}, \gQ^\star_{s,y})$.
                Now we show these $k_s^\star, r_y^\star, \gQ_{s,y}^\star, \rho_{s,y}^\star$ along with model parameter $\theta$ constitute a feasible point of \Cref{eq:general_opt_prob_general_shifting_A}, and the objectives of \Cref{eq:general_opt_prob_general_shifting_A} and \Cref{prob:certified-fairness-shifting-one} are the same given $\gQ^\star$.
                \begin{itemize}[leftmargin=*]
                    \item (Feasibility) \\
                    There are three constraints in \Cref{eq:general_opt_prob_general_shifting_A}.
                    By the definition of $k_s^\star$ and $r_y^\star$, naturally \Cref{general_opt_prob_general_shifting_con1_A} is satisfied.
                    Then, according to \Cref{eq:helling-prop-pf} and the definifition of $\rho^\star_{s,y}$ above, \Cref{general_opt_prob_general_shifting_con2_A} and \Cref{general_opt_prob_general_shifting_con3_A} are satisfied.
                    
                    \item (Objective Equality) \\
                    \begin{equation}
                        \begin{aligned}
                            \text{\Cref{general_opt_prob_general_shifting_obj_A}}
                            & = \sum_{s=1}^S \sum_{y=1}^C k_s^\star r_y^\star \E_{(X,Y)\sim\gQ^\star_{s,y}} [\ell(h_\theta(X),Y)]  \\
                            & = \sum_{s=1}^S \sum_{y=1}^C q_{s,y}^\star \E_{(X,Y)\sim\gQ^\star_{s,y}} [\ell(h_\theta(X),Y)]  
                            \\
                            & = \E_{(X,Y)\sim\gQ^\star} [\ell(h_\theta(X),Y)]  
                            = \text{Optimal value of \Cref{prob:certified-fairness-shifting-two}}.
                        \end{aligned}
                    \end{equation}
                \end{itemize}
                As a result, the optimal value of \Cref{eq:general_opt_prob_general_shifting_A} is larger than or equal to the optimal value of \Cref{prob:certified-fairness-shifting-two}, and hence the optimization problem encoded by \Cref{eq:general_opt_prob_general_shifting_A} provides a fairness certificate.
            
                \item 
                To prove the tightness of the certificate, we only need to show that the optimal value of the optimization problem in \Cref{eq:general_opt_prob_general_shifting_A} is also attainable by the original \Cref{prob:certified-fairness-shifting-two}.
                
                Suppose that the optimal objective of \Cref{eq:general_opt_prob_general_shifting_A} is achieved by optimizable parameters $k_s^\star, r_y^\star, \gQ^\star$, and $\rho_{s,y}^\star$.
                Then, we construct $\gQ^\dagger = \sum_{s=1}^S \sum_{y=1}^C k_s^\star r_y^\star \gQ^\star_{s,y}$.
                We first show that $\gQ^\dagger$ is a feasible point of \Cref{prob:certified-fairness-shifting-two}, and then show that the objective given $\gQ^\dagger$ is equal to the optimal objective of \Cref{eq:general_opt_prob_general_shifting_A}.
                \begin{itemize}[leftmargin=*]
                    \item (Feasibility) \\
                    There are two constraints in \Cref{prob:certified-fairness-shifting-two}: the bounded distance constraint and the fair base rate constraint.
                    The bounded distance constraint is satisfied due to applying \Cref{eq:helling-prop-pf} along with \Cref{general_opt_prob_general_shifting_con2_A,general_opt_prob_general_shifting_con3_A}.
                    The fair base rate constraint is satisfied following the same deduction as in \Cref{pf-thm2-eq3}.
                    
                    \item (Objective Equality) \\
                    \begin{equation*}
                    \begin{aligned}
                        \text{Objective \Cref{prob:certified-fairness-shifting-two}}
                        & =
                        \E_{(X,Y)\sim\gQ^\dagger} [\ell(h_\theta(X),Y)]
                        =
                        \sum_{s=1}^S \sum_{y=1}^C k_s^\star r_y^\star \E_{(X,Y)\sim\gQ^\star_{s,y}} [\ell(h_\theta(X),Y)] \\
                        & = 
                        \text{Optimal value of \Cref{eq:general_opt_prob_general_shifting_A}}.
                    \end{aligned}
                    \end{equation*}
                \end{itemize}
                Thus, the optimal value of the optimization problem in \Cref{eq:general_opt_prob_general_shifting_A} is attainable also by the original \Cref{prob:certified-fairness-shifting-two} which concludes the tightness proof.
            \end{enumerate}
        \end{proof}
    
    \subsection{Proof of \texorpdfstring{\Cref{thm:shiftingtwo}}{Theorem 3}}
        \label{adxsubsec:pf-thm-3}
        
        \paragraph{High-Level Illustration.}
        The starting point of our proof is \Cref{lemma:shiftingtwo-A}, where we have shown a fairness certificate for \Cref{prob:certified-fairness-shifting-two}~(\shiftingtwo setting).
        Then, we plug in Thm.~2.2 in \cite{weber2022certifying}~(stated as \Cref{thm:gramian-bound} in \Cref{adxsec:gramian-bound}) to upper bound the expected loss within each sub-population.
        Now, we get an optimization problem involving $k_s$, $r_y$, and $\rho_{s,y}$ that upper bounds the optimization problem in \Cref{lemma:shiftingtwo-A}.
        In this optimization problem, we find $k_s$ and $r_y$ are bounded in $[0,1]$, and once these two variables are fixed, the optimization with respect to $x_{s,y} := (1 - \rho^2_{s,y})^2$ becomes convex.
        Using this observation, we propose to partition the feasible space of $k_s$ and $r_y$ into sub-regions and solve the convex optimization within each region bearing some degree relaxation, which yields \Cref{thm:shiftingtwo}.
        
        \begin{proof}[Proof of \Cref{thm:shiftingtwo}]
            The proof is done stage-wise: starting from \Cref{lemma:shiftingtwo-A}, we apply relaxation and derive a subsequent optimization problem that upper bounds the previous one stage by stage, until we get the final expression in \Cref{thm:shiftingtwo}.
            
            To demonstrate the proof, we first define the optimization problems at each stage, then prove the relaxations between each adjacent stage, and finally show that the last optimization problem contains a finite number of $\tC$'s values where each $\tC$ is a convex optimization, so that the final optimization problem provides a computable fairness certificate.
            
            We define these quantities, for $s\in[S],y\in[C]$:
            \begin{equation}
                \label{pf-thm-3-macros}
                \begin{aligned}
                & E_{s,y} = \E_{(X,Y)\sim\gP_{s,y}} [\ell(h_\theta(X),Y)],
                \quad
                V_{s,y} = \sV_{(X,Y)\sim\gP_{s,y}} [\ell(h_\theta(X),Y)], \\
                &
                p_{s,y} = \Pr_{(X,Y)\sim\gP} [X_s=s,Y=y], 
                \quad
                C_{s,y} = M - E_{s,y} - \frac{V_{s,y}}{M - E_{s,y}}, \\
                & \bar\gamma_{s,y}^2 = 1 - (1 + (M-E_{s,y})^2/V_{s,y})^{-\frac12}.
                \end{aligned}
            \end{equation}
            Given $\rho > 0$ and the above quantities, the optimization problem definitions are:
            \begin{itemize}[leftmargin=*]
                \item \Cref{lemma:shiftingtwo-A}:
                \begin{subequations}
                    \label{eq:pf_thm3_A}
                    \begin{align}
                        \max_{k_s,r_y,\gQ,\rho_{s,y}} \quad & \sum_{s=1}^S \sum_{y=1}^C k_s r_y \E_{(X,Y)\sim\gQ_{s,y}} [\ell(h_\theta(X), Y)]
                        \label{pf_thm3_obj_A} \\
                        \mathrm{s.t.} \quad & \sum_{s=1}^S k_s = 1, \quad \sum_{y=1}^C r_y = 1, \quad k_s \ge 0 \quad \forall s \in [S], \quad r_y \ge 0 \quad \forall y \in [C],
                        \label{pf_thm3_con1_A} \\
                        & \sum_{s=1}^S \sum_{y=1}^C \sqrt{p_{s,y}k_sr_y} (1 - \rho_{s,y}^2) \ge 1 - \rho^2
                        \label{pf_thm3_con2_A} \\
                        & H(\gP_{s,y}, \gQ_{s,y}) \le \rho_{s,y} \quad \forall s \in [S], y\in [C]
                        \label{pf_thm3_con3_A}.
                    \end{align}
               \end{subequations}
                
                \item After applying \Cref{thm:gramian-bound}:
                \begin{subequations}
                    \label{eq:pf_thm3_B}
                    \begin{align}
                        \max_{k_s,r_y,\rho_{s,y}} \quad & \sum_{s=1}^S \sum_{y=1}^C k_s r_y \left(
                        E_{s,y} + 2\sqrt{\rho_{s,y}^2 (1-\rho_{s,y}^2)^2 (2-\rho_{s,y}^2)} \sqrt{V_{s,y}} + \rho_{s,y}^2 (2-\rho_{s,y}^2) C_{s,y} \right)
                        \label{pf_thm3_obj_B} \\
                        \mathrm{s.t.} \quad & \sum_{s=1}^S k_s = 1, \quad \sum_{y=1}^C r_y = 1, \quad k_s \ge 0 \quad \forall s \in [S], \quad r_y \ge 0 \quad \forall y \in [C],
                        \label{pf_thm3_con1_B} \\
                        & \sum_{s=1}^S \sum_{y=1}^C \sqrt{p_{s,y}k_sr_y} (1 - \rho_{s,y}^2) \ge 1 - \rho^2,
                        \label{pf_thm3_con2_B} \\
                        & 0 \le \rho_{s,y} \le \bar\gamma_{s,y}. \label{pf_thm3_con3_B}
                    \end{align}
                \end{subequations}
                
                \item After variable transform $x_{s,y} := (1-\rho^2_{s,y})^2$:
                \begin{subequations}
                    \label{eq:pf_thm3_C}
                    \begin{align}
                        \max_{k_s,r_y,x_{s,y}} \quad & \sum_{s=1}^S \sum_{y=1}^C k_s r_y \left(
                        E_{s,y} + 2\sqrt{x_{s,y} (1 - x_{s,y})} \sqrt{V_{s,y}} + (1-x_{s,y}) C_{s,y} \right)
                        \label{pf_thm3_obj_C} \\
                        \mathrm{s.t.} \quad & \sum_{s=1}^S k_s = 1, \quad \sum_{y=1}^C r_y = 1, \quad k_s \ge 0 \quad \forall s \in [S], \quad r_y \ge 0 \quad \forall y \in [C],
                        \label{pf_thm3_con1_C} \\
                        & \sum_{s=1}^S \sum_{y=1}^C \sqrt{p_{s,y}k_sr_yx_{s,y}} \ge 1 - \rho^2,
                        \label{pf_thm3_con2_C} \\
                        & (1 - \bar\gamma_{s,y}^2)^2 \le x_{s,y} \le 1 \quad \forall s \in [S], y \in [C].
                        \label{pf_thm3_con3_C} 
                    \end{align}
                \end{subequations}
                
                \item After feasible region partitioning on $k_s$ and $r_y$:
                \begin{subequations}
                    \label{eq:pf_thm3_D}
                    \begin{align}
                        & \max_{\{i_s \in [T]: s\in [S]\}, \{j_y \in [T]: y\in [C]\}} \tC'\left( \left\{\left[\frac{i_s-1}{T}, \frac{i_s}{T} \right]\right\}_{s=1}^S, 
                        \left\{\left[\frac{j_y-1}{T}, \frac{j_y}{T}\right]\right\}_{y=1}^C \right), \text{ where }
                        \label{pf_thm3_obj_D} \\
                        & \tC'\left(\{[\underline{k_s}, \overline{k_s}]\}_{s=1}^S, \{[\underline{r_y}, \overline{r_y}]\}_{y=1}^C\right) =  \label{pf_thm3_con1_D}  \\
                        & \max_{\underline{k_s}\le k_s\le \overline{k_s}, \underline{r_y}\le r_y\le \overline{r_y}, x_{s,y}} \sum_{s=1}^S \sum_{y=1}^C
                        k_s r_y \left(
                        E_{s,y} + 2\sqrt{x_{s,y} (1 - x_{s,y})} \sqrt{V_{s,y}} + (1-x_{s,y}) C_{s,y} \right)
                       \nonumber \\
                        \mathrm{s.t.} \quad & \sum_{s=1}^S k_s = 1, \quad \sum_{y=1}^C r_y = 1,
                        \label{pf_thm3_con2_D} \\
                        & \sum_{s=1}^S \sum_{y=1}^C \sqrt{p_{s,y}k_sr_yx_{s,y}} \ge 1 - \rho^2,
                        \label{pf_thm3_con3_D} \\
                        & (1 - \bar\gamma_{s,y}^2)^2 \le x_{s,y} \le 1 \quad \forall s \in [S], y \in [C].
                        \label{pf_thm3_con4_D}  
                    \end{align}
               \end{subequations}
               
               \item Final quantity in \Cref{thm:shiftingtwo}:
                \begin{subequations}
                    \label{eq:pf_thm3_E}
                    \begin{align}
                        & \max_{\{i_s \in [T]: s\in [S]\}, \{j_y \in [T]: y\in [C]\}} \tC\left( \left\{\left[\frac{i_s-1}{T}, \frac{i_s}{T} \right]\right\}_{s=1}^S, 
                        \left\{\left[\frac{j_y-1}{T}, \frac{j_y}{T}\right]\right\}_{y=1}^C \right), \text{ where } \label{pf_thm3_obj_E} \\
                        & \tC\left(\{[\underline{k_s}, \overline{k_s}]\}_{s=1}^S, \{[\underline{r_y}, \overline{r_y}]\}_{y=1}^C\right)   = \max_{x_{s,y}} \sum_{s=1}^S \sum_{y=1}^C
                        \left(\overline{k_s} \overline{r_y} \left( E_{s,y} + C_{s,y} \right)_+ + \right.  \label{pf_thm3_con1_E} \\
                        &   \hspace{-2em}
                        \left. 
                         \underline{k_s} \underline{r_y} \left( E_{s,y} + C_{s,y} \right)_- + 2\overline{k_s} \overline{r_y} \sqrt{x_{s,y}(1-x_{s,y})} \sqrt{V_{s,y}} - \underline{k_s} \underline{r_y} x_{s,y} (C_{s,y})_+ - \overline{k_s} \overline{r_y} x_{s,y} (C_{s,y})_- \right)
                        \nonumber \\
                        \mathrm{s.t.} \quad & \sum_{s=1}^S \underline{k_s} \le 1, \quad \sum_{s=1}^S \overline{k_s} \ge 1, \quad \sum_{y=1}^C \underline{r_y} \le 1, \quad \sum_{y=1}^C \overline{r_y} \ge 1,
                        \label{pf_thm3_con2_E} \\
                        & \sum_{s=1}^S \sum_{y=1}^C \sqrt{p_{s,y}\overline{k_s}\overline{r_y}x_{s,y}} \ge 1 - \rho^2, \quad
                        \label{pf_thm3_con3_E} \\
                        & 
                        (1-\bar\gamma_{s,y}^2)^2 \le x_{s,y} \le 1\quad \forall s\in [S], y\in[C].
                        \label{pf_thm3_con4_E}
                    \end{align}
                \end{subequations}
            \end{itemize}
            
            We have this relation:
            \begin{equation}
                \text{\Cref{prob:certified-fairness-shifting-two}} 
                \underbrace{\le}_{\text{\Cref{lemma:shiftingtwo-A}}}
                \text{(\ref{eq:pf_thm3_A})}
                \underbrace{\overset{\substack{\text{(when $\ell(h_\theta(X),Y) \in [0,M]$}\\ \text{and $H(\gP_{s,y},\gQ_{s,y}) \le \bar\gamma_{s,y}$)}}}{\le}}_{\text{\bf(A)}}
                \text{(\ref{eq:pf_thm3_B})}
                \underbrace{=}_{\text{\bf(B)}}
                \text{(\ref{eq:pf_thm3_C})}
                \underbrace{=}_{\text{\bf(C)}}
                \text{(\ref{eq:pf_thm3_D})}
                \underbrace{\le}_{\text{\bf(D)}}
                \text{(\ref{eq:pf_thm3_E})}.
            \end{equation}
            Thus, when $H(\gP_{s,y},\gQ_{s,y}) \le \bar\gamma_{s,y}$ and $\sup_{(X,Y)\in\gX\times\gY} \ell(h_\theta(X),Y) \le M$, given that $\ell$ is a non-negative loss by \Cref{subsec:prelim}, we can see \Cref{eq:pf_thm3_E}, i.e., the expression in \Cref{thm:shiftingtwo}'s statement, upper bounds \Cref{prob:certified-fairness-shifting-two}, i.e., provides a fairness certificate for \Cref{prob:certified-fairness-shifting-two}.
            The proofs of these equalities/inequalities are in the following parts labeled by \textbf{(A)}, \textbf{(B)}, \textbf{(C)}, and \textbf{(D)} respectively.
            
            Now we show that each $\tC$ queried by \Cref{eq:thm-general-shifting-1}~(or equally \Cref{pf_thm3_obj_E}) is a convex optimization.
            Inspecting $\tC$'s objective, with respect to the optimizable variable $x_{s,y}$, we find that the only non-linear term in the objective is $\sum_{s=1}^S \sum_{y=1}^C 2\overline{k_s} \overline{r_y} \sqrt{V_{s,y}} \sqrt{x_{s,y}(1-x_{s,y})}$.
            Consider the function $f(x) = \sqrt{x(1-x)}$.
            Define $g(y) = \sqrt{y}$ and $h(x) = x(1-x)$, and then $f(x) = g(h(x))$.
            Thus, $f'(x) = g'(h(x))h'(x)$ and $f''(x) = g''(h(x)) h'(x)^2 + g'(h(x)) h''(x)$.
            Notice that $g''(h(x)) \le 0$, $g'(h(x)) > 0$, and $h''(x) < 0$ for $x\in(0,1]$.
            Thus, $f''(x) \le 0$.
            Since $f$ is twice differentiable in $(0,1]$, we can conclude that $f$ is concave and so does the objective of \Cref{eq:thm-general-shifting-1}.
            Inspecting $\tC$'s constraints, we observe that the only non-linear constraint is $\sum_{s=1}^S \sum_{y=1}^C \sqrt{p_{s,y}\overline{k_s}\overline{r_y}x_{s,y}} \ge 1 - \rho^2$.
            Due to the concavity of function $x\mapsto \sqrt{x}$, we have
            $\sqrt{p_{s,y}\overline{k_s}\overline{r_y}(x_{s,y}^a + x_{s,y}^b)/2} \ge \frac12 \left( \sqrt{p_{s,y}\overline{k_s}\overline{r_y}x_{s,y}^a}+ \sqrt{p_{s,y}\overline{k_s}\overline{r_y}x_{s,y}^b}\right)$ for any two feasible points $x_{s,y}^a$ and $x_{s,y}^b$.
            Thus, this non-linear constraint defines a convex region.
            To this point, we have shown that $\tC$'s objective is concave and $\tC$'s constraints are convex, given that $\tC$ is a maximization problem, $\tC$ is a convex optimization.
        \end{proof}
            
            Under the assumptions that $\ell(h_\theta(X),Y) \in [0,M]$ and $H(\gP_{s,y},\gQ_{s,y}) \le \bar\gamma_{s,y}$:
            \begin{enumerate}[label={\bfseries(\Alph*)},leftmargin=*]
                \item \emph{Proof of \Cref{eq:pf_thm3_A} $\le$ \Cref{eq:pf_thm3_B}.} \\
                Given \Cref{pf_thm3_con3_A}, for each $\gQ_{s,y}$, applying \Cref{thm:gramian-bound}, we get
                \begin{equation}
                    \E_{(X,Y)\sim\gQ_{s,y}}[\ell(h_\theta(X),Y)] \le E_{s,y} + 2\sqrt{\rho_{s,y}^2 (1 - \rho_{s,y}^2)^2 (2 - \rho_{s,y}^2)} \sqrt{V_{s,y}} + \rho_{s,y}^2 (2-\rho_{s,y}^2) C_{s,y}.
                \end{equation}
                Plugging this inequality into all $\E_{(X,Y)\sim\gQ_{s,y}}[\ell(h_\theta(X),Y)]$ in \Cref{pf_thm3_obj_A}, we obtain \Cref{eq:pf_thm3_B}.
                \hfill$\qedsymbol$
                
                \item \emph{Proof of \Cref{eq:pf_thm3_B} $=$ \Cref{eq:pf_thm3_C}.} \\
                By \Cref{pf_thm3_con3_B}, $\rho_{s,y} \in [0,1]$.
                Therefore, $x_{s,y} := (1 - \rho^2_{s,y})^2$ is a one-to-one mapping, and we can use $x_{s,y}$ to parameterize $\rho_{s,y}$, which yields \Cref{eq:pf_thm3_C}.
                \hfill$\qedsymbol$
                
                \item \emph{Proof of \Cref{eq:pf_thm3_C} $=$ \Cref{eq:pf_thm3_D}.} \\
                From \Cref{pf_thm3_con1_C}, we notice that the feasible range of $k_s$ and $r_y$ is subsumed by $[0,1]$.
                We now partition this region $[0,1]$ for each variable to $T$ sub-regions: $[(i-1)/T, i/T]$, $i\in [T]$, and then consider the maximum value across all the combinations of each sub-region for variables $k_s$ and $r_y$, when feasible.
                As a result, \Cref{eq:pf_thm3_C} can be written as the maximum over all such sub-problems where $k_s$'s and $r_y$'s enumerate all possible sub-region combinations, which is exactly encoded by \Cref{eq:pf_thm3_D}. 
                \hfill$\qedsymbol$
                
                \item \emph{Proof of \Cref{eq:pf_thm3_D} $\le$ \Cref{eq:pf_thm3_E}.} \\
                We only need to show that when $\tC'\left(\{[\underline{k_s}, \overline{k_s}]\}_{s=1}^S, \{[\underline{r_y}, \overline{r_y}]\}_{y=1}^C\right)$ is feasible,
                \begin{equation}
                    \tC'\left(\{[\underline{k_s}, \overline{k_s}]\}_{s=1}^S, \{[\underline{r_y}, \overline{r_y}]\}_{y=1}^C\right)
                    \le 
                    \tC\left(\{[\underline{k_s}, \overline{k_s}]\}_{s=1}^S, \{[\underline{r_y}, \overline{r_y}]\}_{y=1}^C\right).
                    \label{eq:pf-thm3-last}
                \end{equation}
                Since both $\tC'$ and $\tC$ are maximization problem, we only need to show that the objective of $\tC$ upper bounds that of $\tC'$, and the constraints of $\tC'$ are equal or relaxations of those of $\tC$.
                
                For the objective, given that $\underline{k_s} \le k_s \le \overline{k_s}$ and $\underline{r_y} \le r_y \le \overline{r_y}$, for any $x_{s,y}$, 
                We observe that 
                \begin{equation}
                    \begin{aligned}
                        k_s r_y (E_{s,y} + C_{s,y}) & \le \overline{k_s} \overline{r_y} \left( E_{s,y} + C_{s,y} \right)_+ 
                        + \underline{k_s} \underline{r_y} \left( E_{s,y} + C_{s,y} \right)_-, \\
                        k_s r_y \cdot \left( 2\sqrt{x_{s,y} (1 - x_{s,y})} \sqrt{V_{s,y}} \right) & \le \overline{k_s} \overline{r_y} \cdot \left( 2\sqrt{x_{s,y} (1 - x_{s,y})} \sqrt{V_{s,y}} \right), \\
                         - k_s r_y C_{s,y} x_{s,y} & \le - \underline{k_s} \underline{r_y} x_{s,y} (C_{s,y})_+ - \overline{k_s} \overline{r_y} x_{s,y} (C_{s,y})_-,
                    \end{aligned}
                \end{equation}
                and by summing up all these terms for all $s\in [S]$ and $y\in [C]$, the LHS would be the objective of $\tC'$ and the RHS would be the objective of $\tC$.
                Hence, $\tC$'s objective upper bounds that of $\tC'$.
                
                For the constraints,
                similarly, given that $\underline{k_s} \le k_s \le \overline{k_s}$ and $\underline{r_y} \le r_y \le \overline{r_y}$, we have
                \begin{equation*}
                    \small
                    \begin{aligned}
                        \text{(\ref{pf_thm3_con2_D})} \sum_{s=1}^S k_s=1, \, \sum_{y=1}^C r_y=1
                        & \Longrightarrow \sum_{s=1}^S \underline{k_s}\le 1, \, \sum_{s=1}^S \overline{k_s}\ge 1, \, \sum_{y=1}^C \underline{r_y} \le 1, \, \sum_{y=1}^C \overline{r_y} \ge 1 \text{(\ref{pf_thm3_con2_E})}, \\
                        \text{(\ref{pf_thm3_con3_D})} \sum_{s=1}^S \sum_{y=1}^C \sqrt{p_{s,y}k_sr_yx_{s,y}} \ge 1 - \rho^2 & \Longrightarrow \sum_{s=1}^S \sum_{y=1}^C \sqrt{p_{s,y}\overline{k_s}\overline{r_y}x_{s,y}} \ge 1 - \rho^2 \text{(\ref{pf_thm3_con3_E})}, \\
                        \text{(\ref{pf_thm3_con4_D})} & \text{ is as same as } \text{(\ref{pf_thm3_con4_E})},
                    \end{aligned}
                \end{equation*}
                which implies that all feasible solutions of $\tC'$ are also feasible for $\tC$.
                Combining with the fact that for any solution of $\tC'$, its objective value $\tC$ is greater than or equal to that of $\tC'$ as shown above, we have \Cref{eq:pf-thm3-last} which concludes the proof.
                \hfill$\qedsymbol$
            \end{enumerate}

\section{Omitted Theorem Statements and Proofs for Finite Sampling Error}

    \label{adxsec:finite-sampling}
    
    \subsection{Finite Sampling Confidence Intervals}
        \label{adxsubsec:finite-sampling-conf-interval}
        
        \begin{lemma}
            Let $\hat P$ be set of i.i.d. finite samples from $\gP$, and let $\hat P_{s,y} := \{(X_i, Y_i) \in \hat P: (X_i)_s = s, Y_i = y\}$ for any $s \in [S], y\in [C]$. 
            Let $\ell: (\hat{y},y) \rightarrow [0,M]$ be a loss function.
            We define $\hat{L}_n=\frac{1}{|\hat P_{s,y}|} \sum_{(X_i, Y_i) \in \hat P_{s,y}} \ell(h_\theta(X_i), Y_i)$, 
            $s_n^2=\frac{1}{n(n-1)}\sum_{1\le i < j \le n}^n\left(\ell(h_\theta(X_i),Y)-\ell(h_\theta(X_j),Y)\right)^2$,
            and $\hat P_{s,y} := \{(X_i, Y_i) \in \hat P: (X_i)_s = s, Y_i = y\}$.
            Then for $\delta > 0$, with respect to the random draw of $\hat P$ from $\gP$, we have 
            \begin{align}
                & \Pr\left(
                \hat{L}_n - M\sqrt{\frac{\ln(2/\delta)}{2|\hat P_{s,y}|}} \le \underset{(X,Y)\sim\gP_{s,y}}{\E}[\ell \left( h_\theta(X),Y \right)] \le \hat{L}_n + M\sqrt{\frac{\ln(2/\delta)}{2|\hat P_{s,y}|}}
                \right) \ge 1 - \delta, \label{mean_finite_sampling} \\
                & \Pr\left(
                \sqrt{s_n^2} - M \sqrt{\dfrac{2\ln(2/\delta)}{|\hat P_{s,y}|-1}} \le \sqrt{\underset{(X,Y)\sim\gP_{s,y}}{\sV}[\ell \left( h_\theta(X),Y \right)]} \le \sqrt{s_n^2} + M \sqrt{\dfrac{2\ln(2/\delta)}{|\hat P_{s,y}|-1}}
                \right) \ge 1 - \delta, \label{variance_finite_sampling} \\
                & \Pr\left(
                \dfrac{|\hat P_{s,y}|}{|\hat P|} - \sqrt{\dfrac{\ln(2/\delta)}{2|\hat P|}} \le \Pr_{(X,Y)\sim\gP} [X_s=s, Y=y] \le \dfrac{|\hat P_{s,y}|}{|\hat P|} + \sqrt{\dfrac{\ln(2/\delta)}{2|\hat P|}}
                \right) \ge 1 - \delta. \label{proportion_finite_sampling}
            \end{align}
            \label{lemma:finite-sampling-conf-interval}
        \end{lemma}
    
        \begin{proof}[Proof of \Cref{lemma:finite-sampling-conf-interval}]
        We can get \Cref{variance_finite_sampling} according to Theorem 10 in \cite{maurer2009empirical}. Here, we will provide proofs for \Cref{mean_finite_sampling} and \Cref{proportion_finite_sampling}, respectively. The general idea is to use Hoeffding's inequality to get the high-confidence interval.
        
        We will prove \Cref{mean_finite_sampling} first. From Hoeffding's inequality, for all $t>0$, we have:
        \begin{equation}
            \Pr\left(\left| \hat{L}_n - \underset{(X,Y)\sim\gP_{s,y}}{\E}[\ell \left( h_\theta(X),Y \right)] \right| \ge t\right) \le 2 \exp \left( -\dfrac{2|\hat P_{s,y}|^2 t^2}{|\hat P_{s,y}|M^2} \right)
            \label{mean_hoeffding}
        \end{equation}
        Since we want to get an interval with confidence $1-\delta$, we let $2 \exp \left( -\dfrac{2|\hat P_{s,y}|^2 t^2}{|\hat P_{s,y}|M^2} \right)=\delta$, from which we can derive that
        \begin{equation}
            t = M \sqrt{\dfrac{\ln(2/\delta)}{2|\hat P_{s,y}|}} \label{mean_fs_t}
        \end{equation}
        Plugging \Cref{mean_fs_t} into \Cref{mean_hoeffding}, we can get:
        \begin{equation}
            \Pr\left(
            \hat{L}_n - M\sqrt{\frac{\ln(2/\delta)}{2|\hat P_{s,y}|}} \le \underset{(X,Y)\sim\gP_{s,y}}{\E}[\ell \left( h_\theta(X),Y \right)] \le \hat{L}_n + M\sqrt{\frac{\ln(2/\delta)}{2|\hat P_{s,y}|}}
            \right) \ge 1 - \delta
        \end{equation}
        
        Then we will prove \Cref{proportion_finite_sampling}. 
        From Hoeffding's inequality, for all $t>0$, we have:
        \begin{equation}
            \Pr\left(\left|\dfrac{|\hat P_{s,y}|}{|\hat P|}-\Pr_{(X,Y)\sim\gP} [X_s=s, Y=y]\right| \ge t\right) \le 2 \exp \left( -\dfrac{2|\hat P|^2 t^2}{|\hat P|} \right) \label{proportion_hoeffding}
        \end{equation}
        Since we want to get an interval with confidence $1-\delta$, we let $2 \exp \left( -\dfrac{2|\hat P|^2 t^2}{|\hat P|} \right)=\delta$, from which we can derive that
        \begin{equation}
            t =  \sqrt{\dfrac{\ln(2/\delta)}{2|\hat P|}} \label{proportion_t}
        \end{equation}
        Plugging \Cref{proportion_t} into \Cref{proportion_hoeffding}, we can get:
        \begin{equation}
            \Pr\left(
            \dfrac{|\hat P_{s,y}|}{|\hat P|} - \sqrt{\dfrac{\ln(2/\delta)}{2|\hat P|}} \le \Pr_{(X,Y)\sim\gP} [X_s=s, Y=y] \le \dfrac{|\hat P_{s,y}|}{|\hat P|} + \sqrt{\dfrac{\ln(2/\delta)}{2|\hat P|}}
            \right) \ge 1 - \delta
        \end{equation}
        \end{proof}
        
    \subsection{Fairness Certification Statements with Finite Sampling}
        \label{adxsubsec:statement-finite-sampling}
    \begin{theorem}[\Cref{thm:shiftingone} with finite sampling]
        Given a distance bound $\rho > 0$ and any $\delta > 0$, the following constrained optimization, which is \textbf{convex}, when feasible, provides a fairness certificate for \Cref{prob:certified-fairness-shifting-one} with probability at least $1-2SC\delta$: 
        \begin{subequations}
            \label{eq:general_opt_prob_label_shifting_thm_fs}
            \begin{align}
                \max_{k_s,r_y,p_{s,y}}  \quad & \sum_{s=1}^S \sum_{y=1}^{C} k_s r_y \overline{E_{s,y}} \quad \label{eq:shiftingone_objective}\\
                \mathrm{s.t.} \quad & \sum_{s=1}^S k_s = 1, \quad \sum_{y=1}^C r_y = 1, \quad k_s \ge 0 \quad \forall s\in [S], \quad r_y \ge 0 \quad \forall y\in [C], \label{eq:shiftingone_kr_fs}
                \\
                &  1-\rho^2-\sum_{s=1}^S \sum_{y=1}^C \sqrt{p_{s,y}k_s r_y} \le 0, \label{eq:shiftingone_dis_fs} \\
                & \underline{p_{s,y}} \le p_{s,y} \le \overline{p_{s,y}}, \quad \forall s\in [S], \quad \forall y\in [C] \label{eq:shiftingone_p_fs} \\
                & \sum_{s=1}^S \sum_{y=1}^{C}p_{s,y}=1 \label{eq:shiftingone_ps_fs}
            \end{align}
        \end{subequations}
        where $\overline{E_{s,y}} := \hat{L}_n + M \sqrt{\ln(2/\delta)/\left(2|\hat P_{s,y}|\right)}$, $\underline{p_{s,y}} := |\hat P_{s,y}|/|\hat P| - \sqrt{\ln(2/\delta)/\left(2|\hat P|\right)}$, $\overline{p_{s,y}} := |\hat P_{s,y}|/|\hat P| + \sqrt{\ln(2/\delta)/\left(2|\hat P|\right)}$ are constants computed with \Cref{lemma:finite-sampling-conf-interval}.
        \label{thm:shiftingone_fs}
    \end{theorem}
    
    \begin{theorem}
       \label{thm:shiftingtwo_fs}
        If for any $s\in [S]$ and $y\in [Y]$, $H(\gP_{s,y},\gQ_{s,y}) \le \bar\gamma_{s,y}$ and $0 \le \sup_{(X,Y)\in \gX\times\gY} \ell(h_\theta(X),Y) \le M$, given a distance bound $\rho > 0$ and any $\delta>0$, for any region granularity $T \in \sN_+$, the following expression provides a fairness certificate for \Cref{prob:certified-fairness-shifting-two} with probability at least $1-3SC\delta$: 
       \begin{equation}
            \bar\ell = \max_{\{i_s \in [T]: s\in [S]\}, \{j_y \in [T]: y\in [C]\}} \tC\left( \left\{\left[\frac{i_s-1}{T}, \frac{i_s}{T} \right]\right\}_{s=1}^S, 
            \left\{\left[\frac{j_y-1}{T}, \frac{j_y}{T}\right]\right\}_{y=1}^C \right), \text{ where }
            \label{eq:thm-general-shifting-1_fs}
        \end{equation}
        \begin{subequations}
            \label{eq:general_opt_prob_general_shifting_C_fs}
            \begin{align}
                \hspace{-3em} & \tC\left(\{[\underline{k_s}, \overline{k_s}]\}_{s=1}^S, \{[\underline{r_y}, \overline{r_y}]\}_{y=1}^C\right)   = \max_{x_{s,y},p_{s,y}} \sum_{s=1}^S \sum_{y=1}^C
                \left(\overline{k_s} \overline{r_y} \left( \overline{E_{s,y}} + \overline{C_{s,y}} \right)_+
                + \underline{k_s} \underline{r_y} \left( \overline{E_{s,y}} + \overline{C_{s,y}} \right)_- \right.  \nonumber  \\
                &  \hspace{0em}
                \left. +2\overline{k_s} \overline{r_y} \sqrt{x_{s,y}(1-x_{s,y})} \sqrt{\overline{V_{s,y}}} - \underline{k_s} \underline{r_y} x_{s,y} (\underline{C_{s,y}})_+ - \overline{k_s} \overline{r_y} x_{s,y} (\underline{C_{s,y}})_- \right)
                \label{general_opt_prob_general_shifting_obj_C_fs} \\
              \mathrm{s.t.} \quad & \sum_{s=1}^S \underline{k_s} \le 1, \quad \sum_{s=1}^S \overline{k_s} \ge 1, \quad \sum_{y=1}^C \underline{r_y} \le 1, \quad \sum_{y=1}^C \overline{r_y} \ge 1,
                \label{general_opt_prob_general_shifting_con1_C_fs} \\
               \qquad \quad & \sum_{s=1}^S \sum_{y=1}^C \sqrt{p_{s,y}\overline{k_s}\overline{r_y}x_{s,y}} \ge 1 - \rho^2, \quad
                \left(1-\overline{\bar\gamma_{s,y}^2}\right)^2 \le x_{s,y} \le 1,
                \label{general_opt_prob_general_shifting_con3_C_fs} \\
                & \underline{p_{s,y}} \le p_{s,y} \le \overline{p_{s,y}}, \quad \sum_{s=1}^S \sum_{y=1}^{C}p_{s,y}=1 \label{general_opt_prob_general_shifting_con4_C_fs}
            \end{align}
       \end{subequations}
       where $(\cdot)_+ = \max\{\cdot,0\}$, $(\cdot)_- = \min\{\cdot,0\}$; 
       $\underline{E_{s,y}} := \hat{L}_n - M \sqrt{\ln(2/\delta)/\left(2|\hat P_{s,y}|\right)}$, $\overline{E_{s,y}} := \hat{L}_n + M \sqrt{\ln(2/\delta)/\left(2|\hat P_{s,y}|\right)}$, 
       $\underline{V_{s,y}}=\left(\sqrt{s_n^2} - M \sqrt{2\ln(2/\delta)/\left(|\hat P_{s,y}|-1\right)}\right)^2$, $\overline{V_{s,y}}=\left(\sqrt{s_n^2} + M \sqrt{2\ln(2/\delta)/\left(|\hat P_{s,y}|-1\right)}\right)^2$,
       $\underline{p_{s,y}} := |\hat P_{s,y}|/|\hat P| - \sqrt{\ln(2/\delta)/\left(2|\hat P|\right)}$, $\overline{p_{s,y}} := |\hat P_{s,y}|/|\hat P| + \sqrt{\ln(2/\delta)/\left(2|\hat P|\right)}$ computed with \Cref{lemma:finite-sampling-conf-interval}, and 
       $\underline{C_{s,y}} = M-\overline{E_{s,y}}-\overline{V_{s,y}}/(M-\overline{E_{s,y}})$, $\overline{C_{s,y}} = M-\underline{E_{s,y}}-\underline{V_{s,y}}/(M-\underline{E_{s,y}})$,
       $\overline{\bar\gamma_{s,y}^2} = 1 - (1 + (M-\underline{E_{s,y}})^2/\overline{V_{s,y}})^{-\frac12}$.
       \Cref{eq:thm-general-shifting-1_fs} only takes $\tC$'s value when it is feasible, and each $\tC$ queried by \Cref{eq:thm-general-shifting-1_fs} is a \textbf{convex optimization}.
    \end{theorem}
    
    \subsection{Proofs of Fairness Certification with Finite Sampling}
        \label{adxsubsec:pf-statement-finite-sampling}
        
        \paragraph{High-Level Illustration.}
        We use Hoeffding’s inequality to bound the finite sampling error of statistics and add the high confidence box constraints to the optimization problems, which can still be proved to be convex.
        
    \begin{proof}[Proof of \Cref{thm:shiftingone_fs}]
        The proof of \Cref{thm:shiftingone_fs} is composed of two parts:
            (1)~the optimization problem provides a fairness certificate for \Cref{prob:certified-fairness-shifting-one};
            (2)~the optimization problem is convex.
            
            \begin{enumerate}[label={(\bfseries\arabic*)},leftmargin=*]
                \item We prove that \Cref{thm:shiftingone_fs} provides a fairness certificate for \Cref{prob:certified-fairness-shifting-one} in this part.
                Since \Cref{thm:shiftingone} provides a fairness certificate for \Cref{prob:certified-fairness-shifting-one}, we only need to prove: (a) the feasible region of the optimization problem in \Cref{thm:shiftingone} is a subset of the feasible region of the optimization problem in \Cref{thm:shiftingone_fs}, and (b) the optimization objective in \Cref{thm:shiftingone} can be upper bounded by that in \Cref{thm:shiftingone_fs}.
                
                To prove (a), we first equivalently transform the optimization problem in \Cref{thm:shiftingone} into the following optimization problem by adding $p_{sy}$ to the decision variables:
                \begin{subequations}
                \label{eq:general_opt_prob_label_shifting_thm_equal}
                    \begin{align}
                        \max_{k_s,r_y,p_{s,y}}  \quad & \sum_{s=1}^S \sum_{y=1}^{C} k_s r_y E_{s,y} \quad \label{eq:shiftingone_objective_fs} \\
                        \mathrm{s.t.} \quad & \sum_{s=1}^S k_s = 1, \quad \sum_{y=1}^C r_y = 1, \quad k_s \ge 0 \quad \forall s\in [S], \quad r_y \ge 0 \quad \forall y\in [C], \label{eq:shiftingone_kr_equal}
                        \\
                        &  1-\rho^2-\sum_{s=1}^S \sum_{y=1}^C \sqrt{p_{s,y}k_s r_y} \le 0, \\
                        & p_{s,y} = |\hat P_{s,y}|/|\hat P|, \quad \forall s\in [S], \quad \forall y\in [C] \label{eq:shiftingone_p_equal} \\
                        & \sum_{s=1}^S \sum_{y=1}^{C}p_{s,y}=1 \label{eq:shiftingone_ps_equal}
                    \end{align}
                \end{subequations}
            
        For decision variables $k_{s,y}$ and $r_{s,y}$, optimization \ref{eq:general_opt_prob_label_shifting_thm_fs} and \ref{eq:general_opt_prob_label_shifting_thm_equal} has the same constraints (\Cref{eq:shiftingone_kr_fs} and \Cref{eq:shiftingone_kr_equal}). 
        For decision variables $p_{s,y}$, the feasible region of $p_{s,y}$ in optimization \ref{eq:general_opt_prob_label_shifting_thm_fs} (decided by \Cref{eq:shiftingone_p_fs,eq:shiftingone_ps_fs}) is a subset of the feasible region of $p_{s,y}$ in optimization \ref{eq:general_opt_prob_label_shifting_thm_equal} (decided by \Cref{eq:shiftingone_p_equal,eq:shiftingone_ps_equal}), since $\underline{p_{s,y}} \le |\hat P_{s,y}|/|\hat P| \le \overline{p_{s,y}}$.
        Therefore, the feasible region with respect to $k_{s,y}$, $r_{s,y}$, and $p_{s,y}$ of the optimization problem in \Cref{thm:shiftingone} is a subset of that in \Cref{thm:shiftingone_fs}. 
        
        To prove (b), we only need to show that the objective in \Cref{eq:shiftingone_objective} can be upper bounded by the objective in \Cref{eq:shiftingone_objective_fs}.
        The statement $\sum_{s=1}^S \sum_{y=1}^{C} k_s r_y E_{s,y} \le \sum_{s=1}^S \sum_{y=1}^{C} k_s r_y \overline{E_{s,y}}$ consistently holds because $E_{s,y} \le \overline{E_{s,y}}$ and $k_s,r_y \ge 0$.
        
        Combining the proofs of (a) and (b), we prove that \Cref{thm:shiftingone_fs} provides a fairness certificate for \Cref{prob:certified-fairness-shifting-one}.
        
        \item 
        Inspecting that the objective and all the constraints in optimization problem in \Cref{eq:general_opt_prob_label_shifting_thm_fs} are linear with respect to $k_s$, $r_{y}$, $p_{s,y}$ but the one in \Cref{eq:shiftingone_dis_fs}. Therefore, we only need to prove that the following constraint is convex with respect to $k_s$, $r_{y}$, $p_{s,y}$:
        \begin{equation}
            1-\rho^2-\sum_{s=1}^S \sum_{y=1}^C \sqrt{p_{s,y}k_s r_y} \le 0 \label{eq:dist}
        \end{equation}
        We define a function $f$ with respect to vector $\mathbf{p}:= [p_{s,y}]_{s \in [S], y \in [C]}$: $f(\mathbf{p})=1-\rho^2-\sum_{s=1}^S \sum_{y=1}^C \sqrt{p_{s,y}k_s r_y}$.
        Then we can derive that:
        \begin{equation}
            \dfrac{\partial^2 f}{\partial \mathbf{p}^2} = \sum_{s=1}^S \sum_{y=1}^C \dfrac{\sqrt{k_s r_y}}{4} {p_{s,y}}^{-\frac{3}{2}} \ge 0
        \end{equation}
        Therefore, the function $f$ is convex with respect to $p_{s,y}$. Similarly, we can prove the convexity with respect to $k_s$ and $r_y$. Finally, we can conclude that the constraint in \Cref{eq:dist} is convex with respect to $k_s$, $r_{y}$, $p_{s,y}$ and the optimization problem defined in \Cref{thm:shiftingone_fs} is convex.
    \end{enumerate}
    Since we use the union bound to bound $E_{s,y}$ and $p_{s,y}$ for all $s \in [S], y \in [C]$ simultaneously, the confidence is $1-2SC\delta$.
    \end{proof}
    \begin{proof}[Proof of \Cref{thm:shiftingtwo_fs}]
        The proof of \Cref{thm:shiftingtwo_fs} includes two parts:
            (1)~the optimization problem provides a fairness certificate for \Cref{prob:certified-fairness-shifting-two};
            (2)~each $\tC$ queried by \Cref{eq:thm-general-shifting-1_fs} is a convex optimization.
            
            \begin{enumerate}[label={(\bfseries\arabic*)},leftmargin=*]
                \item Since \Cref{thm:shiftingtwo} provides a fairness certificate for \Cref{prob:certified-fairness-shifting-two}, we only need to prove: (a) the feasible region of the optimization problem in \Cref{thm:shiftingtwo} is a subset of that in \Cref{thm:shiftingtwo_fs}, and (b) the optimization objective in \Cref{thm:shiftingtwo} can be upper bounded by that in \Cref{thm:shiftingtwo_fs}.
                
                To prove (a), we first equivalently transform the optimization problem in \Cref{thm:shiftingtwo} into the following optimization problem by adding $p_{sy}$ to the decision variables:
                \begin{subequations}
                    \label{eq:general_opt_prob_general_shifting_C_equal}
                    \begin{align}
                        \hspace{-4em} & \tC\left(\{[\underline{k_s}, \overline{k_s}]\}_{s=1}^S, \{[\underline{r_y}, \overline{r_y}]\}_{y=1}^C\right)   = \max_{x_{s,y},p_{s,y}} \sum_{s=1}^S \sum_{y=1}^C
                        \left(\overline{k_s} \overline{r_y} \left( E_{s,y} + C_{s,y} \right)_+
                        + \underline{k_s} \underline{r_y} \left( E_{s,y} + C_{s,y} \right)_- \right.  \nonumber  \\
                        &  \hspace{0em}
                        \left. + 2\overline{k_s} \overline{r_y} \sqrt{x_{s,y}(1-x_{s,y})} \sqrt{V_{s,y}} - \underline{k_s} \underline{r_y} x_{s,y} (C_{s,y})_+ - \overline{k_s} \overline{r_y} x_{s,y} (C_{s,y})_- \right)
                        \label{general_opt_prob_general_shifting_obj_C_equal} \\
                      \mathrm{s.t.} ~~ & \sum_{s=1}^S \underline{k_s} \le 1, \quad \sum_{s=1}^S \overline{k_s} \ge 1,  \sum_{y=1}^C \underline{r_y} \le 1, \quad \sum_{y=1}^C \overline{r_y} \ge 1,
                        \label{general_opt_prob_general_shifting_con1_C_equal} \\
                        & \sum_{s=1}^S \sum_{y=1}^C \sqrt{p_{s,y}\overline{k_s}\overline{r_y}x_{s,y}} \ge 1 - \rho^2, \quad
                        \left(1-\bar\gamma_{s,y}^2 \right)^2 \le x_{s,y} \le 1,
                        \label{general_opt_prob_general_shifting_con3_C_equal}  \\
                        & p_{s,y} = |\hat P_{s,y}|/|\hat P|, \quad \forall s\in [S], \quad \forall y\in [C] \label{eq:shiftingtwo_p_equal} \\
                        & \sum_{s=1}^S \sum_{y=1}^{C}p_{s,y}=1 \label{eq:shiftingtwo_ps_equal}
                    \end{align}
            \end{subequations}
        
        For decision varibales $x_{s,y}$, since $\sqrt{p_{s,y}\overline{k_s}\overline{r_y}x_{s,y}} \le \sqrt{\overline{p_{s,y}}\overline{k_s}\overline{r_y}x_{s,y}}$ and $\left(1-\bar\gamma_{s,y}^2\right)^2 \ge \left(1-\overline{\bar\gamma_{s,y}^2}\right)^2$, the feasible region of $x_{s,y}$ in \Cref{eq:general_opt_prob_general_shifting_C_equal} is a subset of that in \Cref{eq:general_opt_prob_general_shifting_C_fs}. For decision variables $p_{s,y}$, since $\underline{p_{s,y}} \le |\hat P_{s,y}|/|\hat P| \le \overline{p_{s,y}}$, the feasible region of $p_{s,y}$ in \Cref{eq:general_opt_prob_general_shifting_C_equal} is also a subset of that in \Cref{eq:general_opt_prob_general_shifting_C_fs}. Therefore, the feasible region of the optimization problem in \Cref{thm:shiftingtwo} is a subset of that in \Cref{thm:shiftingtwo_fs}.
        
        To prove (b), we only need to show that the objective in \Cref{general_opt_prob_general_shifting_obj_C_equal} can be upper bounded by the objective in \Cref{general_opt_prob_general_shifting_obj_C_fs}. Since $\underline{k_s},\overline{k_s},\underline{r_y},\overline{r_y} \ge 0$ and $0 \le x_{s,y} \le 1$ hold, we can observe that $\forall s \in [S], \forall y \in [C]$,
        \begin{subequations}
        \begin{footnotesize}
        \begin{align*}
            &\overline{k_s} \overline{r_y} \left( E_{s,y} + C_{s,y} \right)_+ + \underline{k_s} \underline{r_y} \left( E_{s,y} + C_{s,y} \right)_-  + 2\overline{k_s} \overline{r_y} \sqrt{x_{s,y}(1-x_{s,y})} \sqrt{V_{s,y}} 
             - \underline{k_s} \underline{r_y} x_{s,y} (C_{s,y})_+ - \\
             & \overline{k_s} \overline{r_y} x_{s,y} (C_{s,y})_- 
            ~~ \le ~~  \overline{k_s} \overline{r_y} \left( \overline{E_{s,y}} + \overline{C_{s,y}} \right)_+ 
                + \underline{k_s} \underline{r_y} \left( \overline{E_{s,y}} + \overline{C_{s,y}} \right)_- +
                 2\overline{k_s} \overline{r_y} \sqrt{x_{s,y}(1-x_{s,y})} \sqrt{\overline{V_{s,y}}} \\
                 & - \underline{k_s} \underline{r_y} x_{s,y} (\underline{C_{s,y}})_+ - \overline{k_s} \overline{r_y} x_{s,y} (\underline{C_{s,y}})_- 
        \end{align*}
        \end{footnotesize}
        \end{subequations}
        
        Therefore, we prove that the optimization in \Cref{thm:shiftingtwo_fs} provides a fairness certificate for \Cref{prob:certified-fairness-shifting-two}.
        
        \item We will prove that each $\tC$ queried by \Cref{eq:thm-general-shifting-1_fs} is a convex optimization with respect to decision variables $x_{s,y}$ and $p_{s,y}$ in this part.
        In the proof of \Cref{thm:shiftingtwo}, we provide the proof of convexity with respect to $x_{s,y}$, so we only need to prove that the optimization problem is convex with respect to $p_{s,y}$.
        We can observe that the constraints of $p_{s,y}$ in \Cref{general_opt_prob_general_shifting_con4_C_fs} is linear, and thus we only need to prove that $\sum_{s=1}^S \sum_{y=1}^C \sqrt{p_{s,y}\overline{k_s}\overline{r_y}x_{s,y}} \ge 1 - \rho^2$ (the constraint in \Cref{general_opt_prob_general_shifting_con3_C_fs}) is convex with respect to $p_{s,y}$. 
        Here, we define a function $f$ with respect to vector $\vp:= [p_{s,y}]_{s \in [S], y \in [C]}$: $f(\vp)=1-\rho^2-\sum_{s=1}^S \sum_{y=1}^C \sqrt{p_{s,y}\overline{k_s}\overline{ r_y}}$.
        Then we can derive that:
        \begin{equation}
            \left(\dfrac{\partial^2 f}{\partial \vp^2}\right)_{sy,s'y'} = \sum_{s=1}^S \sum_{y=1}^C \dfrac{\sqrt{\overline{k_s} \overline{r_y}}}{4} {p_{s,y}}^{-\frac{3}{2}} \cdot \1[s=s',y=y'] \ge 0
        \end{equation}
        Thus, the function $f$ is convex and $f(\vp) \le 0$ defines a convex set with respect to $p_{s,y}$. Then, we prove that the constraint $\sum_{s=1}^S \sum_{y=1}^C \sqrt{p_{s,y}\overline{k_s}\overline{r_y}x_{s,y}} \ge 1 - \rho^2$ is convex with respect to $p_{s,y}$.
            \end{enumerate}
    Since we use the union bound to bound $E_{s,y}$, $V_{s,y}$ and $p_{s,y}$ for all $s \in [S], y \in [C]$ simultaneously, the confidence is $1-3SC\delta$.
    \end{proof}

\section{Experiments}

\subsection{Datasets}
\label{sec:datasets}
We validate our certified fairness on \textit{six} real-world datasets: Adult~\cite{asuncion2007uci}, Compas~\cite{angwin2016machine},  Health~\cite{healthdataset}, Lawschool~\cite{wightman1998lsac}, Crime~\cite{asuncion2007uci}, and German~\cite{asuncion2007uci}.
All the used datasets contain categorical data.

In Adult dataset, we have 14 attributes of a person as input and try to predict whether the income of the person is over $50$k \$/year. The sensitive attribute in Adult is selected as the sex.

In Compas dataset, given the attributes of a criminal defendent, the task is to predict whether he/she will re-offend in two years. The sensitive attribute in Compas is selected as the race.

In Health dataset, given the physician records and and insurance claims of the patients, we try to predict ten-year mortality by binarizing the Charlson Index, taking the median value as a cutoff. The sensitive attribute in Health is selected as the age.

In Lawschool dataset, we try to predict whether a student passes the exam according to the appication records of different law schools. The sensitive attribute in Lawschool is the race.

In Crime dataset, we try to predict whether a specific community is above or below the median number of violent crimes per population. The sensitive attribute in Crime is selected as the race.

In German dataset, each person is classified as good or bad credit risks according to the set of attributes. The sensitive attribute in German is selected as the sex.

Following \cite{ruoss2020learning}, we consider the scenario where sensitive attributes and labels take binary values, and we also follow their standard data processing steps: (1) normalize the numerical values of all attributes with the mean value 0 and variance 1, (2) use one-hot encodings to represent categorical attributes, (3) drop instances and attributes with missing values, and (4) split the datasets into training set, validation set, and test set.

\textbf{Training Details.} We directly train a ReLU network composed of two hidden layers on the training set of six datasets respectively following the setting in \cite{ruoss2020learning}. Concretely, we train the models for 100 epochs with a batch size of 256. We adopt the binary cross-entropy loss and use the Adam optimizer with weight decay 0.01 and dynamic learning rate scheduling (ReduceLROnPlateau in \cite{paszke2019pytorch}) based on the loss on the validation set starting at 0.01 with the patience of 5 epochs.

\subsection{Fair Base Rate Distribution Generation Protocol}
\label{sec:fair_gen}
To evaluate how well our certificates capture the fairness risk in practice, we compare our certification bound with the empirical loss evaluated on randomly generated $30,000$ fairness constrained distributions $\gQ$ shifted from $\gP$.
Now, we introduce the protocols to generate fairness distributions $\gQ$ for sensitive shifting and general shifting, respectively.
Note that the protocols are only valid when the sensitive attributes and labels take binary values.

Fair base rate distributions generation steps in the \textit{sensitive shifting} scenario:
\begin{enumerate}[label={(\bfseries\arabic*)},leftmargin=*]
    \item Sample the proportions of subpopulations of the generated distribution $q_{0,0},q_{0,1},q_{1,0},q_{1,1}$: uniformly sample two real values in the interval $[0,1]$, and do the assignment: $q_{0,0}:=kr$, $q_{0,1}:=k(1-r)$, $q_{1,0}:=(1-k)r$, $q_{1,1}:=(1-k)(1-r)$.
    \item Determine the sample size of every subpopulation: first determine the subpopulation which requires the largest sample size, use all the samples in that subpopulation, and then calculate the sample size in other subpopulations according to the proportions.
    \item Uniformly sample in each subpopulation based on the sample size.
    \item Calculate the Hellinger distance $H(\gP,\gQ) =\sqrt{1 - \sum_{s=0}^1 \sum_{y=0}^1 \sqrt{p_{s,y}} \sqrt{q_{s,y}}}$.
    Suppose that the support of $\gP$ and $\gQ$ is $\gX \times \gY$ and the densities of $\gP$ and $\gQ$ with respect to a suitable measure are $f_{\gP}$ and $f_{\gQ}$, respectively. 
    Since we consider sensitive shifting here, we have $f_{\gQ_{s,y}} = \lambda_{s,y} f_{\gP_{s,y}}, \quad s,y \in \{0,1\}$ where $\lambda_{s,y}$ is a scalor. 
    The derivation of the distance calculation formula is shown as follows,
    \begin{subequations}
    \begin{align}
        H^2(\gP,\gQ) &= 1 - \iint_{\gX \times \gY} \sqrt{f_{\gP}(x,y)} \sqrt{f_{\gQ}(x,y)} \d x\d y \\
        &= 1 - \sum_{s=0}^1 \sum_{y=0}^1 \iint_{f_{\gP_{s,y}}(x,y)>0} \sqrt{f_{\gP_{s,y}}(x,y)} \sqrt{\lambda_{s,y} f_{\gP_{s,y}}} \d x\d y \\
        &= 1 - \sum_{s=0}^1 \sum_{y=0}^1 \sqrt{\lambda_{s,y}} \iint_{f_{\gP_{s,y}}(x,y)>0} f_{\gP_{s,y}}(x,y) \d x\d y \\
        &= 1 - \sum_{s=0}^1 \sum_{y=0}^1 \sqrt{\lambda_{s,y}} p_{s,y} \\
        &= 1 - \sum_{s=0}^1 \sum_{y=0}^1 \sqrt{p_{s,y}} \sqrt{q_{s,y}}.
    \end{align}
    \end{subequations}
\end{enumerate}

Fair base rate distribution generation steps in the \textit{general shifting} scenario:
\begin{enumerate}[label={(\bfseries\arabic*)},leftmargin=*]
    \item Construct a data distribution $\gQ^\prime$ that is disjoint with the training data distribution $\gP$ by changing the distribution of non-sensitive values given the sensitive attributes and labels.
    \item Sample mixing parameters $\alpha_{s,y}$ and $\alpha^\prime_{s,y}$ in the interval $[0,1]$ satisfying $\frac{p_{00}\alpha_{00}+q_{00}\alpha^\prime_{00}}{p_{01}\alpha_{01}+q_{01}\alpha^\prime_{01}} = \frac{p_{10}\alpha_{10}+q_{10}\alpha^\prime_{10}}{p_{11}\alpha_{11}+q_{11}\alpha^\prime_{11}}$~(base rate parity) and $p_{00}\alpha_{00}+q_{00}\alpha^\prime_{00}+p_{01}\alpha_{01}+q_{01}\alpha^\prime_{01}+p_{10}\alpha_{10}+q_{10}\alpha^\prime_{10}+p_{11}\alpha_{11}+q_{11}\alpha^\prime_{11}=1$. 
    \item Determine the proportion of every subpopulation in distribution $\gQ$: $q_{s,y}:=\alpha_{s,y} p_{s,y} + \alpha^\prime_{s,y} q^\prime_{s,y}, \quad s,y \in \{0,1\}$.
    \item Determine the sample size of every subpopulation in $\gP$ and $\gQ^\prime$: first determine the subpopulation which requires the largest sample size, use all the samples in that subpopulation, and then calculate the sample size in other subpopulations according to the proportions.
    \item Calculate the Hellinger distance between distribution $\gP$ and $\gQ$:
    $H(\gP,\gQ) = \sqrt{1-\sum_{s=0}^{1}\sum_{y=0}^1\sqrt{\alpha_{s,y}}p_{s,y}}$. Suppose that the support of $\gP$ and $\gQ$ is $\gX \times \gY$ and the densities of $\gP$ and $\gQ$ with respect to a suitable measure are $f_{\gP}$ and $f_{\gQ}$, respectively. The derivation of the distance calculation formula is shown as follows,
    \begin{subequations}
    \begin{align}
        H^2(\gP,\gQ) &= 1 - \iint_{\gX \times \gY} \sqrt{f_{\gP}(x,y)} \sqrt{f_{\gQ}(x,y)} \d x\d y \\
        &= 1 - \iint_{\gX \times \gY} \sqrt{\sum_{s=0}^1\sum_{y=0}^1  f_{\gP_{s,y}}(x,y)} \sqrt{\sum_{s=0}^1\sum_{y=0}^1 \left( \alpha_{s,y} f_{\gP_{s,y}}(x,y) + \alpha^\prime_{s,y}f_{\gQ^\prime_{s,y}}(x,y)\right) } \d x\d y \\
        &= 1 - \sum_{s=0}^1\sum_{y=0}^1 \sqrt{\alpha_{s,y}} \iint_{f_{\gP_{s,y}}(x,y)>0} f_{\gP_{s,y}}(x,y) \sqrt{1 + \dfrac{\alpha_{s,y}^\prime f_{\gQ_{s,y}}(x,y)}{\alpha_{s,y}f_{\gP_{s,y}}(x,y)}} \d x\d y \\
        &= 1 - \sum_{s=0}^1\sum_{y=0}^1 \sqrt{\alpha_{s,y}} \iint_{f_{\gP_{s,y}}(x,y)>0} f_{\gP_{s,y}}(x,y)\d x\d y\\
        &= 1 - \sum_{s=0}^1\sum_{y=0}^1 \sqrt{\alpha_{s,y}} p_{s,y}.
    \end{align}
    \end{subequations}
\end{enumerate}

\subsection{Implementation Details of Our Fairness Certification}
\label{sec:imple_details}
We conduct vanilla training and then calculate our certified fairness according to our certification framework.
Concretely, in the training step, we train a ReLU network composed of 2 hidden layers of size 20 for 100 epochs with binary cross entropy loss (BCE loss) using an Adam optimizer. The initial learning rate is 0.05 for Crime and German datasets, while for other datasets, the initial learning rate is set $0.001$. We reduce the learning rate with a factor of 0.5 on the plateau measured by the loss on the validation set with patience of 5 epochs. In the fairness certification step, we set the region granularity to be $0.005$ for certification in the general shifting scenario. We use $90\%$ confidence interval when considering finite sampling error. The codes we used follow the MIT license. All experiments are conducted on a 1080 Ti GPU with 11,178 MB memory.

\subsection{Implementation Details of WRM}
\label{sec:detail_wrm}
The optimization problem of tackling distributional robustness is formulated as:
\begin{equation}
    \max_{\gQ} \, \E_{(X,Y)\sim\gQ} [\ell(h_\theta(X), Y)]
    \quad \mathrm{s.t.} \quad \dist(\gP,\gQ) \le \rho
\end{equation}
where $\dist(\cdot,\cdot)$ is a predetermined distribution distance metric. Note that the optimization is the same as our certified fairness optimization in \Cref{prob:certified-fairness-shifting-two} except for the fairness constraint.

WRM~\cite{sinha2017certifying} proposes to use the dual reformulation of the Wasserstein worst-case risk to provide the distributional robustness certificate, which is formulated in the following proposition.
\begin{proposition}[\cite{sinha2017certifying}, Proposition 1]
Let $\ell: \Theta \times \mathcal{Z} \rightarrow \mathbb{R}$ and $c: \Theta \times \mathcal{Z} \rightarrow \mathbb{R}_+$ be continuous and let $\phi_\gamma(\theta;z_0):=\sup_{z\in \mathcal{Z}}\{\ell(z;\theta) - \gamma c(z;\theta) \} $. Then, for any distribution $\gP$ and any $\rho>0$,
\begin{equation}
    \sup_{\gQ:W_c(\gP,\gQ)\le \rho} \E_\gQ [\ell(\theta;Z)] = \inf_{\gamma \ge 0}\{ \gamma \rho + \E_\gP[\phi_\gamma(\theta;Z)] \}
\end{equation}
where $W_c(\gP,\gQ):=\inf_{\pi\in\Pi(\gP,\gQ)}{\int_{\mathcal{Z}}c(z,z^\prime)d\pi(z,z^\prime)}$ is the 1-Wasserstein distance between $\gP$ and $\gQ$.
\end{proposition}
One requirement for the certificate to be tractable is that the surrogate function $\phi_\gamma$ is concave with respect to $Z$, which holds when $\gamma$ is larger than the Lipschitz constant $L$ of the gradient of $\ell$ with respect to $Z$. Since we use the ELU network with JSD loss, we can efficiently calculate $\gamma$ iteratively as shown in Appendix D of \cite{weber2022certifying}.

We select Gaussian mixture data for fair comparison.
The Gaussian mixture data can be formulated as $P(x|\theta)=\sum_{k=1}^K \alpha_k \phi(x|\theta_k)$ where $K$ is the number of Gaussian data, $\alpha_k$ is the proportion of the $k$-th Gaussian, and $\theta_k=(\mu_k,\sigma_k^2)$. In our evaluation, we use 2-dimension Gaussian and mixture data composed of 2 Gaussian ($K=2$) labeled $0$ and $1$, respectively. Concretely, we let $\mu_1=(-2,-0.5), \sigma_1=1.0, \alpha_1=0.5$ and $\mu_2=(2,0.5), \sigma_2=1.0, \alpha_2=0.5$. The second dimension of input vector is selected as the sensitive attribute $X_s$, and the base rate constraint becomes: $\Pr(Y=0|X_s<0)=\Pr(Y=1|X_s> 0)$.
Given the perturbation $\delta \in \mathbb{R}^2$ that induces $X \mapsto X + \delta$, the Wasserstein distance and Hellinger distance can be formulated as follows:
\begin{equation}
    W_2(\gP,\gQ)=\|\delta\|_2, \quad H(\gP,\gQ)=\sqrt{1-e^{-\|\delta\|^2_2/8}}.
\end{equation}

\vspace{+2em}
\subsection{More Results of Certified Fairness with Sensitive Shifting and General shifting}
\label{sec:app_results}
\vspace{+3em}
\begin{figure}[htbp]
\centering
\subfigure{
    \rotatebox{90}{\hspace{-4em}Classification Error}
    \begin{minipage}{0.43\linewidth}
    \centerline{\quad CRIME}
 	\vspace{1pt}
 	\centerline{\includegraphics[width=1.0\textwidth]{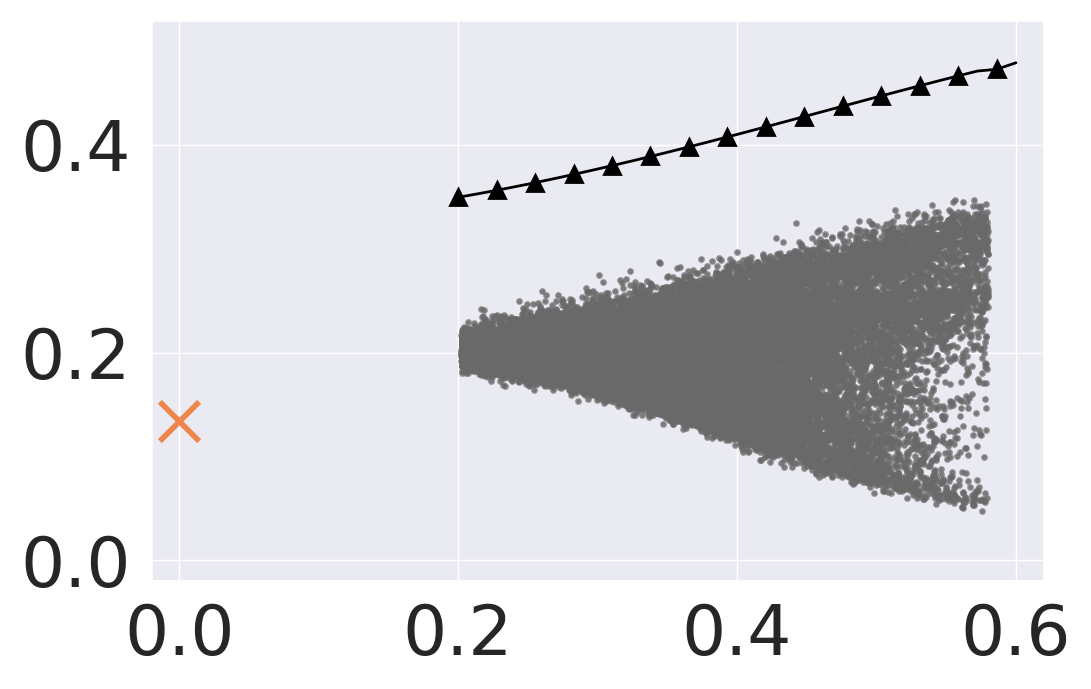}}
 	\vspace{-1.3em}
 \end{minipage}
 \begin{minipage}{0.43\linewidth}
    \centerline{\quad GERMAN}
 	\vspace{1pt}
 	\centerline{\includegraphics[width=1.0\textwidth]{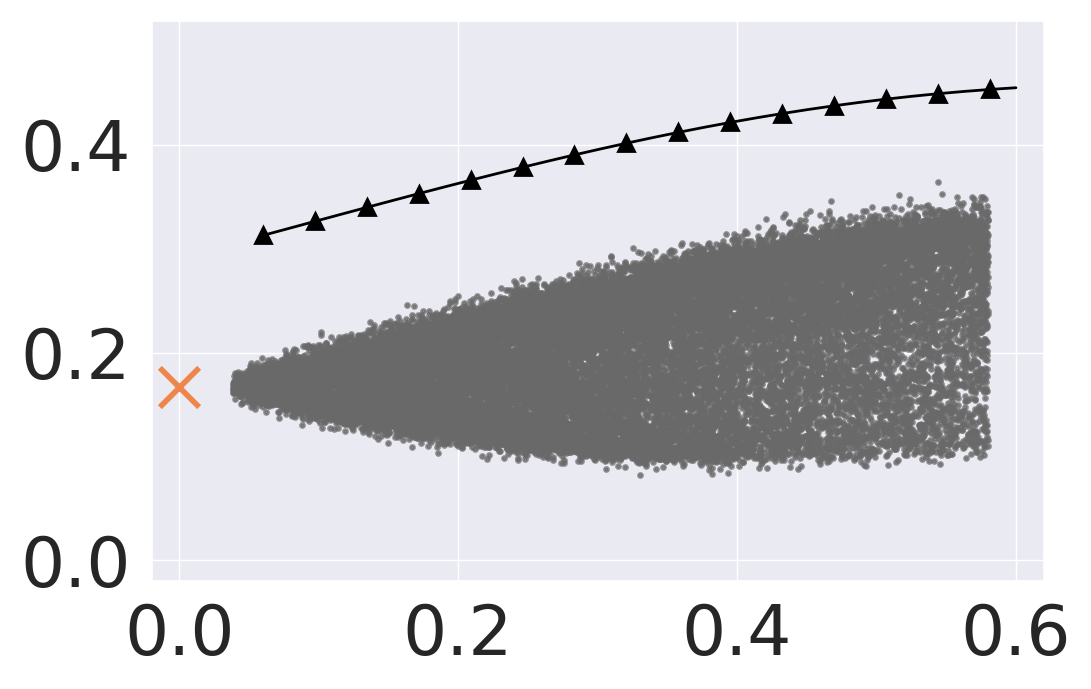}}
 	\vspace{-1.3em}
 \end{minipage}
}
\subfigure{
    \rotatebox{90}{\hspace{-0.9em}BCE Loss}
    \begin{minipage}{0.43\linewidth}
 	\centerline{\includegraphics[width=1.0\textwidth]{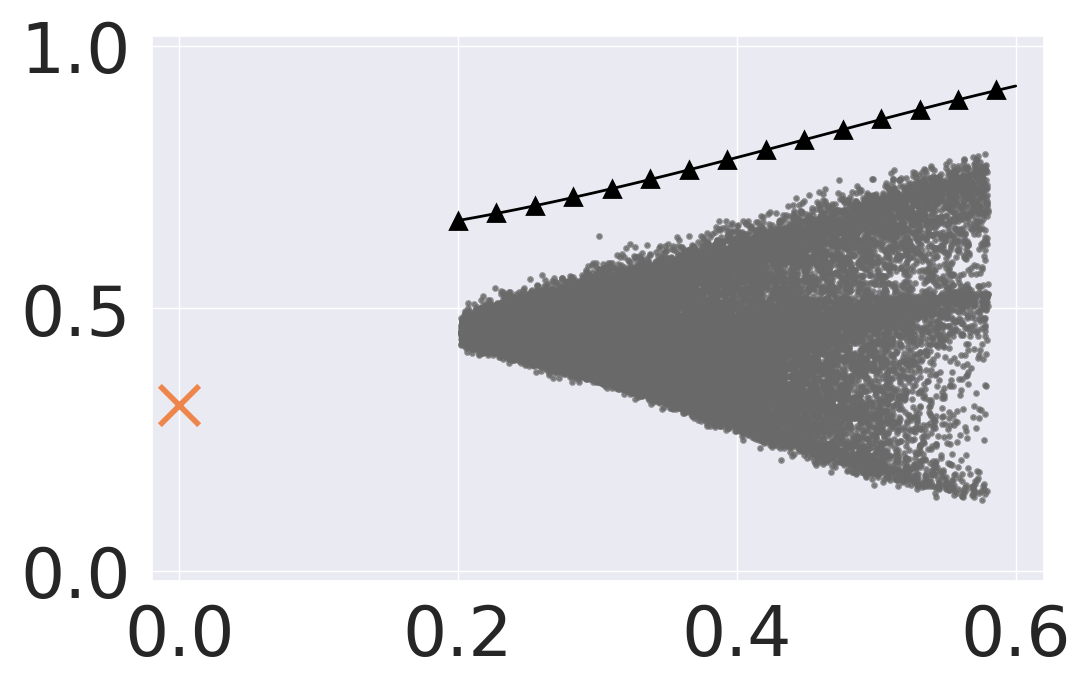}}
 	\centerline{\quad~~ Hellinger Distance}
 \end{minipage}
\begin{minipage}{0.43\linewidth}
 	\centerline{\includegraphics[width=1.0\textwidth]{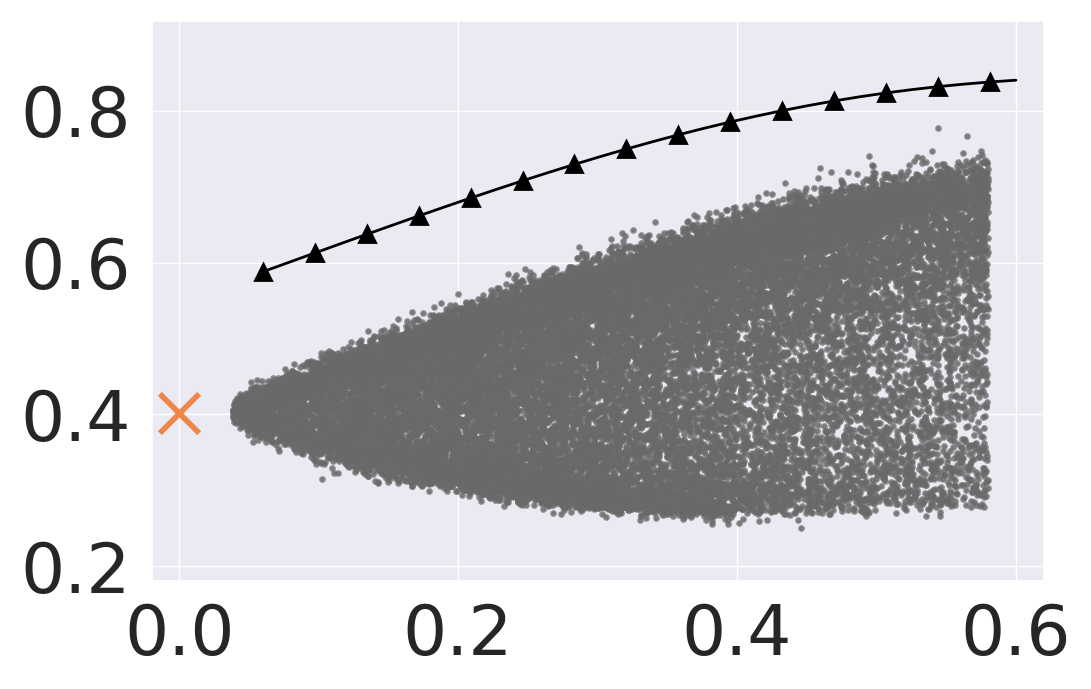}}
 	\centerline{\quad~~ Hellinger Distance}
 \end{minipage}
}
\subfigure{
\centerline{\includegraphics[width=0.7\textwidth]{Figures/legend_1.png}}}
\caption{Certified fairness with \shiftingone on Crime and German.}
\end{figure}

\begin{figure}[htbp]
\centering
\subfigure{
    \rotatebox{90}{\hspace{-4em}Classification Error}
    \begin{minipage}{0.43\linewidth}
    \centerline{\quad CRIME}
 	\vspace{1pt}
 	\centerline{\includegraphics[width=1.0\textwidth]{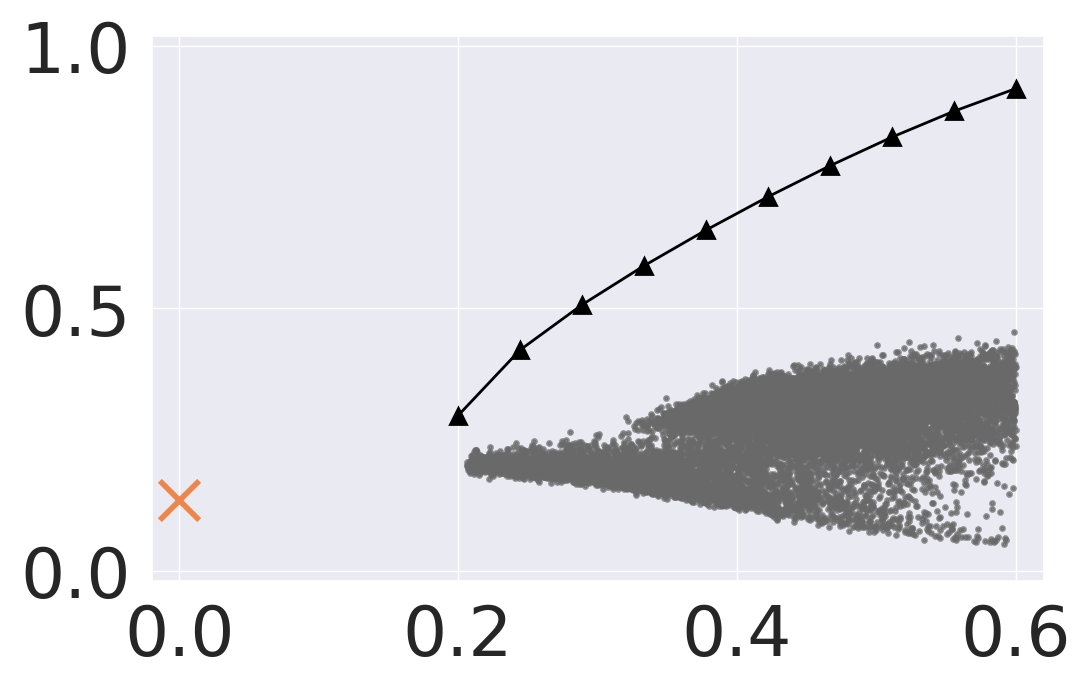}}
 \end{minipage}
 \begin{minipage}{0.43\linewidth}
    \centerline{\quad GERMAN}
 	\vspace{1pt}
 	\centerline{\includegraphics[width=1.0\textwidth]{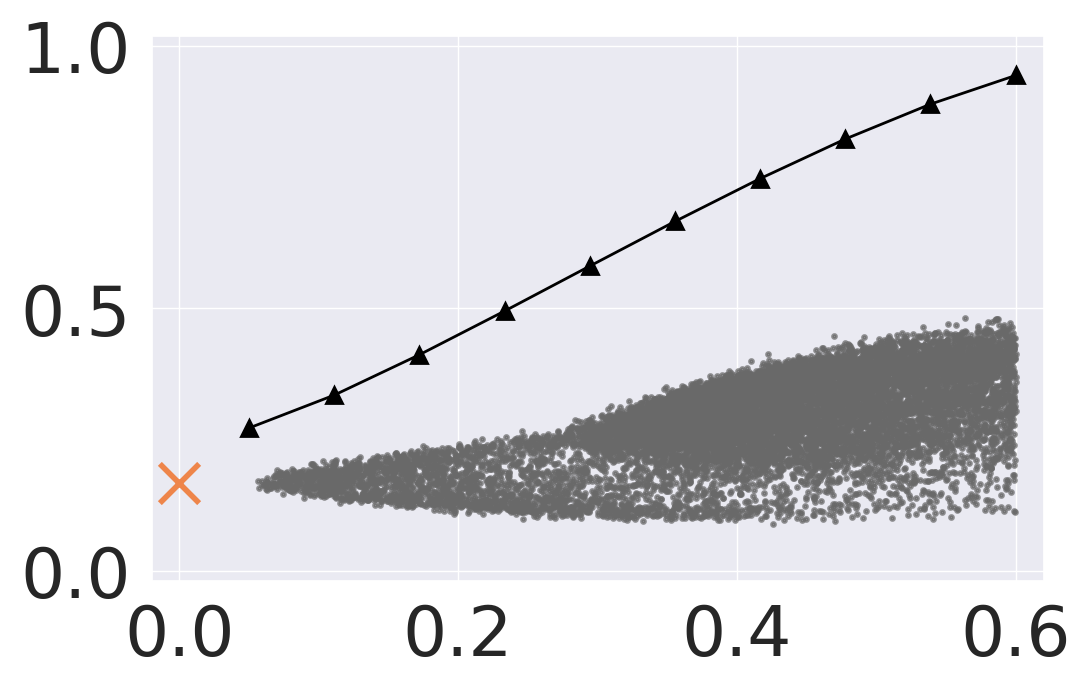}}
 \end{minipage}
}
\subfigure{
    \rotatebox{90}{\hspace{-0.9em}JSD Loss}
    \begin{minipage}{0.43\linewidth}
 	\centerline{\includegraphics[width=1.0\textwidth]{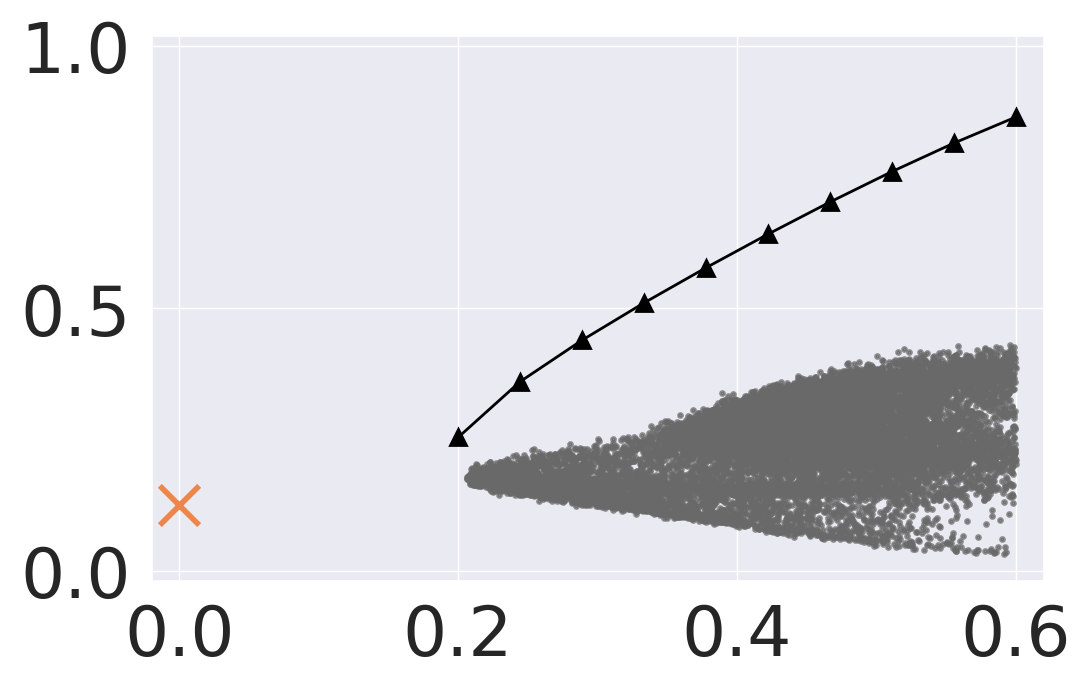}}
 	\centerline{\quad~~~ Hellinger Distance}
 \end{minipage}
\begin{minipage}{0.43\linewidth}
 	\centerline{\includegraphics[width=1.0\textwidth]{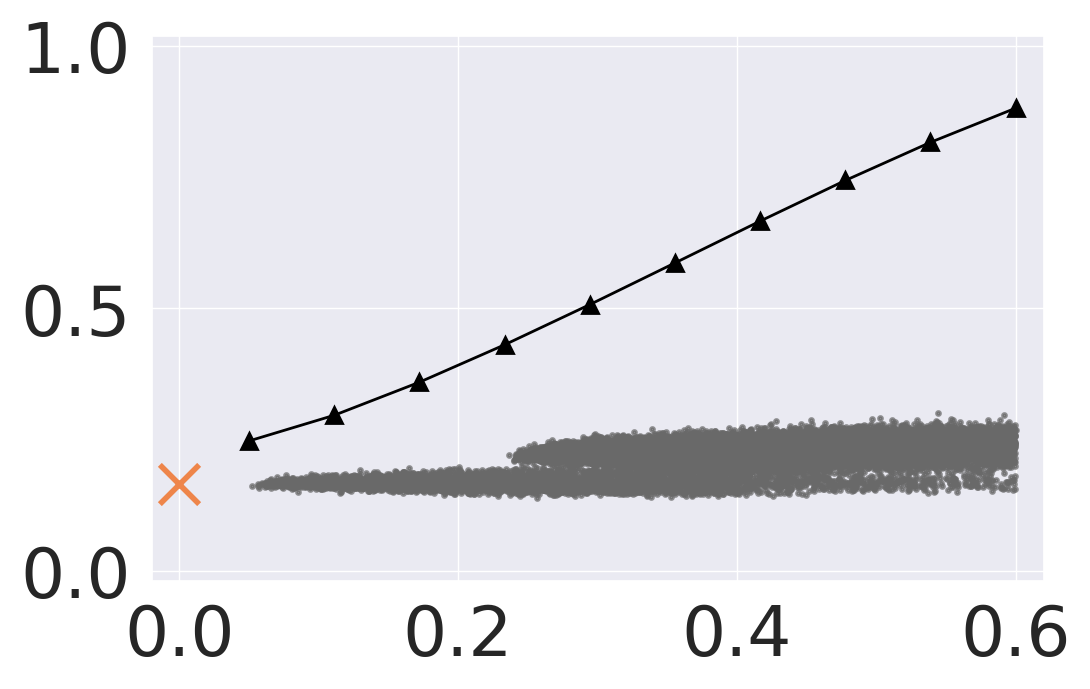}}
 	\centerline{\quad~~~ Hellinger Distance}
 \end{minipage}
}
\subfigure{
\centerline{\includegraphics[width=0.7\textwidth]{Figures/legend_1.png}}}
\caption{Certified fairness with \shiftingtwo on Crime and German.}
\end{figure}

\vspace{-3em}
\begin{figure}[htbp]
\centering
\subfigure{
    \rotatebox{90}{\hspace{-1em}JSD Loss}
    \begin{minipage}{0.35\linewidth}
    \centerline{\quad ADULT}
 	\vspace{1pt}
 	\centerline{\includegraphics[width=1.0\textwidth]{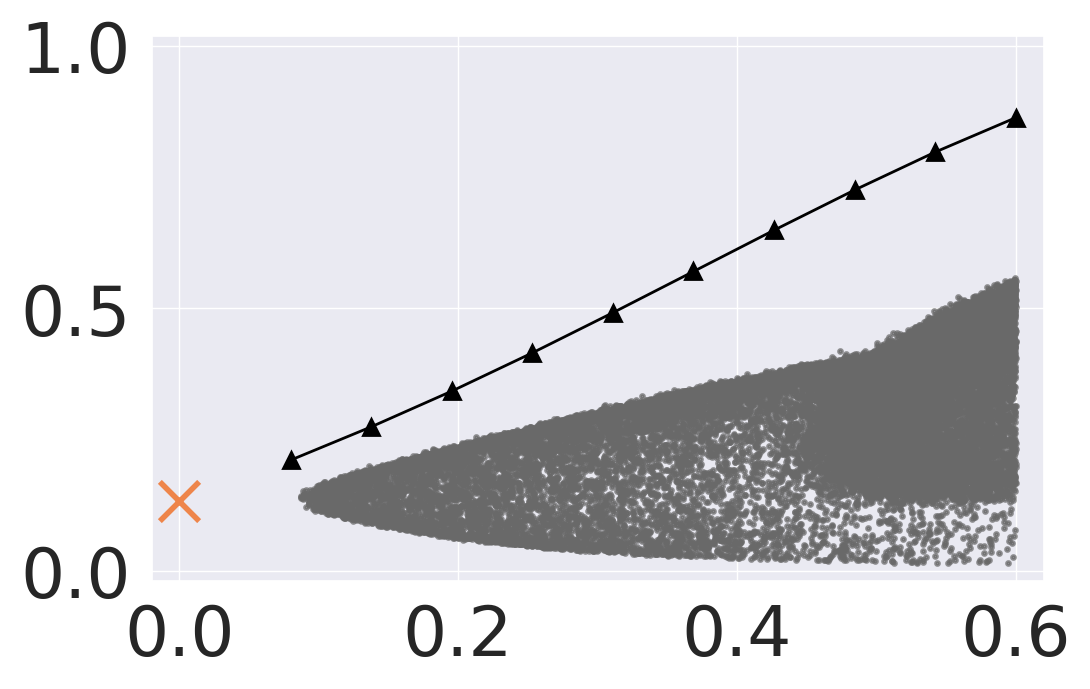}}
 	\vspace{-0.5em}
 	\centerline{\quad~~Hellinger Distance}
 	\vspace{-0.5em}
 \end{minipage}
 \begin{minipage}{0.35\linewidth}
    \centerline{\quad COMPAS}
 	\vspace{1pt}
 	\centerline{\includegraphics[width=1.0\textwidth]{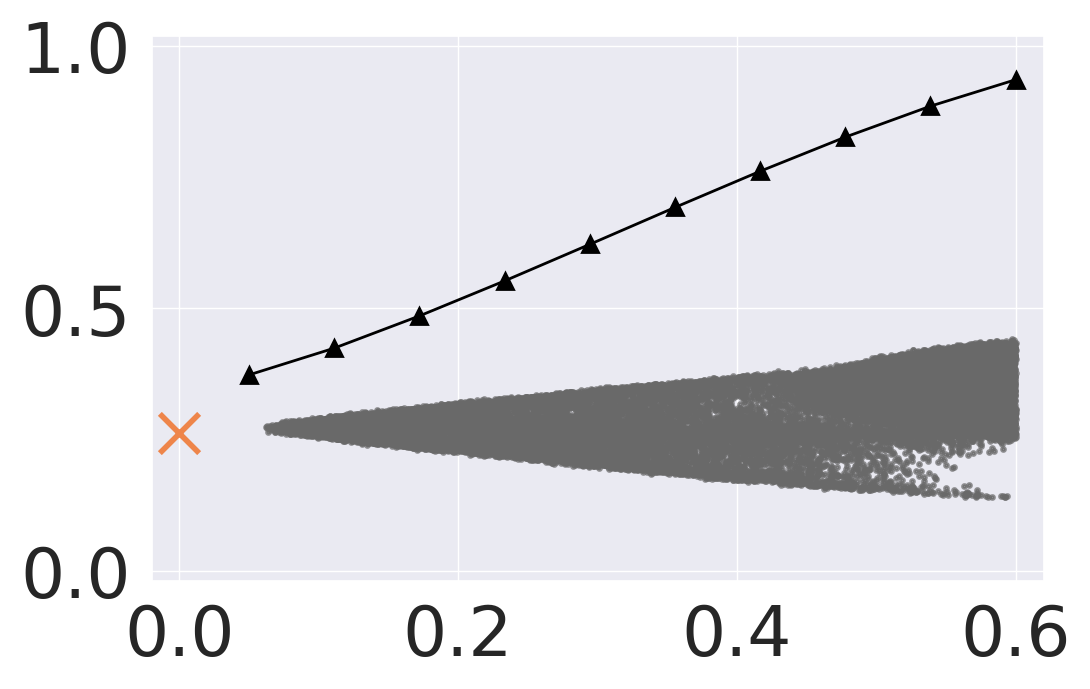}}
 	\vspace{-0.5em}
 	\centerline{\quad~~Hellinger Distance}
 	\vspace{-0.5em}
 \end{minipage}
}
\subfigure{
\rotatebox{90}{\hspace{-1em}JSD Loss}
\begin{minipage}{0.35\linewidth}
    \centerline{\quad HEALTH}
 	\vspace{1pt}
 	\centerline{\includegraphics[width=1.0\textwidth]{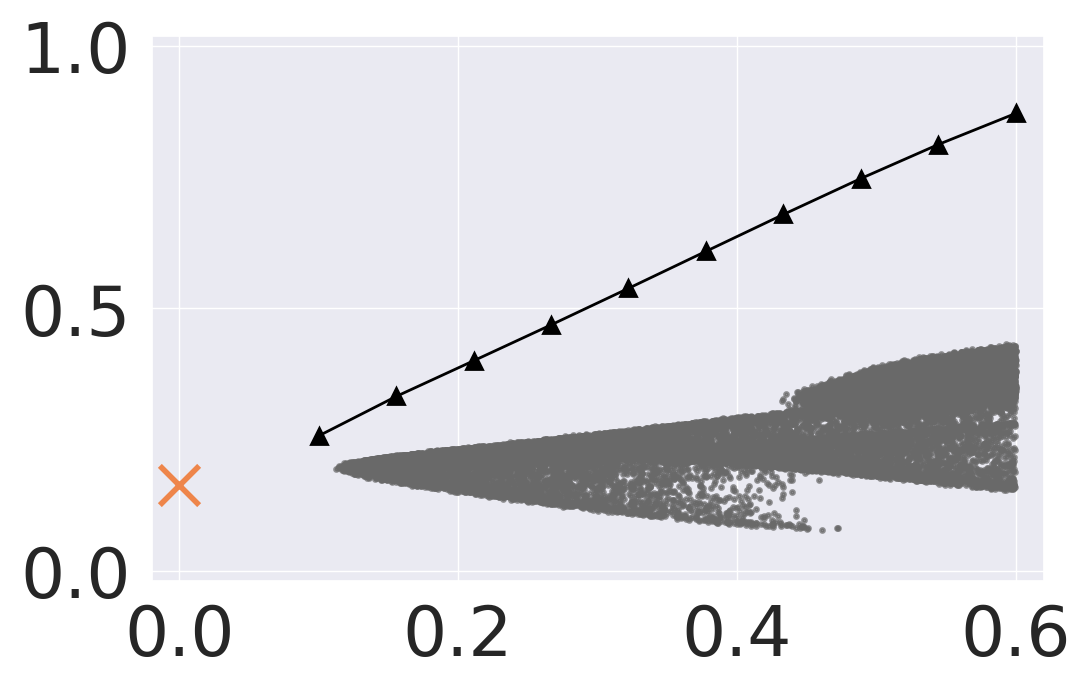}}
 	\vspace{-0.5em}
 	\centerline{\quad~~Hellinger Distance}
 	\vspace{-0.5em}
 \end{minipage}
 \begin{minipage}{0.35\linewidth}
    \centerline{\quad LAW SCHOOL}
 	\vspace{1pt}
 	\centerline{\includegraphics[width=1.0\textwidth]{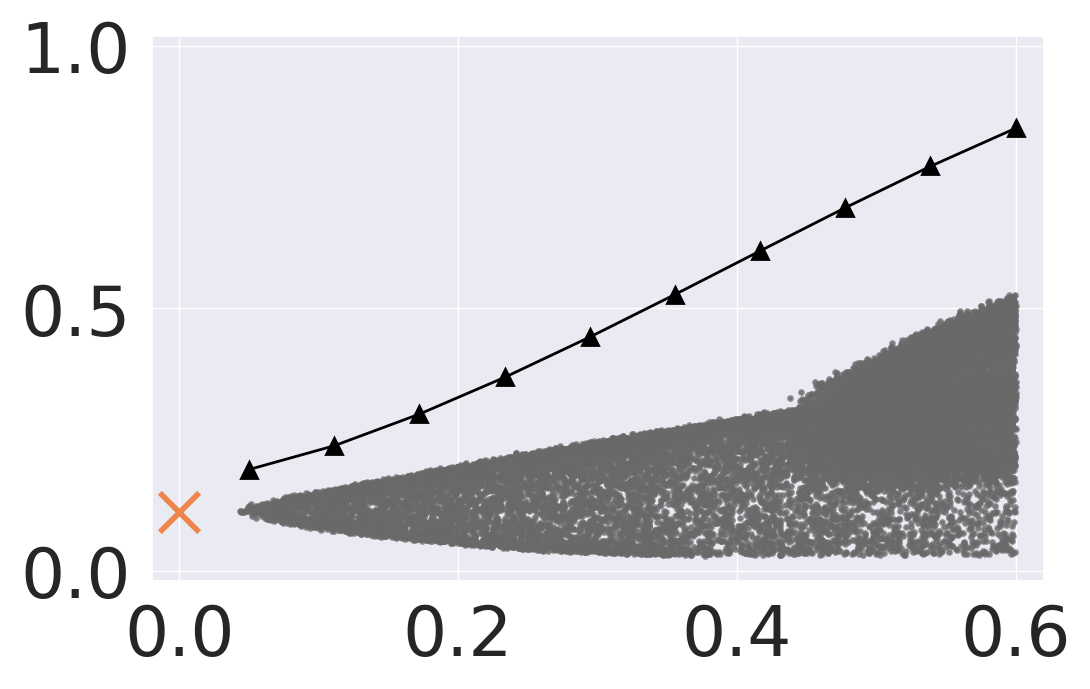}}
 	\vspace{-0.5em}
 	\centerline{\quad~~Hellinger Distance}
 	\vspace{-0.5em}
 \end{minipage}
}
\subfigure{
\centerline{\includegraphics[width=0.7\textwidth]{Figures/legend_1.png}}}
\vspace{-2em}
\caption{Certified fairness with \shiftingtwo using JSD loss on Adult, Compas, Health, and Lawschool.}
\end{figure}

\newpage
\subsection{More Results of Certified Fairness with Additional Non-Skewness Constraints}
\label{sec:app_results_nonskew}
\vspace{+3em}
\begin{figure}[htbp]
\vspace{-1.5em}
 \begin{minipage}{0.5\linewidth}
 	\vspace{1pt}
 	\centerline{\includegraphics[width=1.0\textwidth]{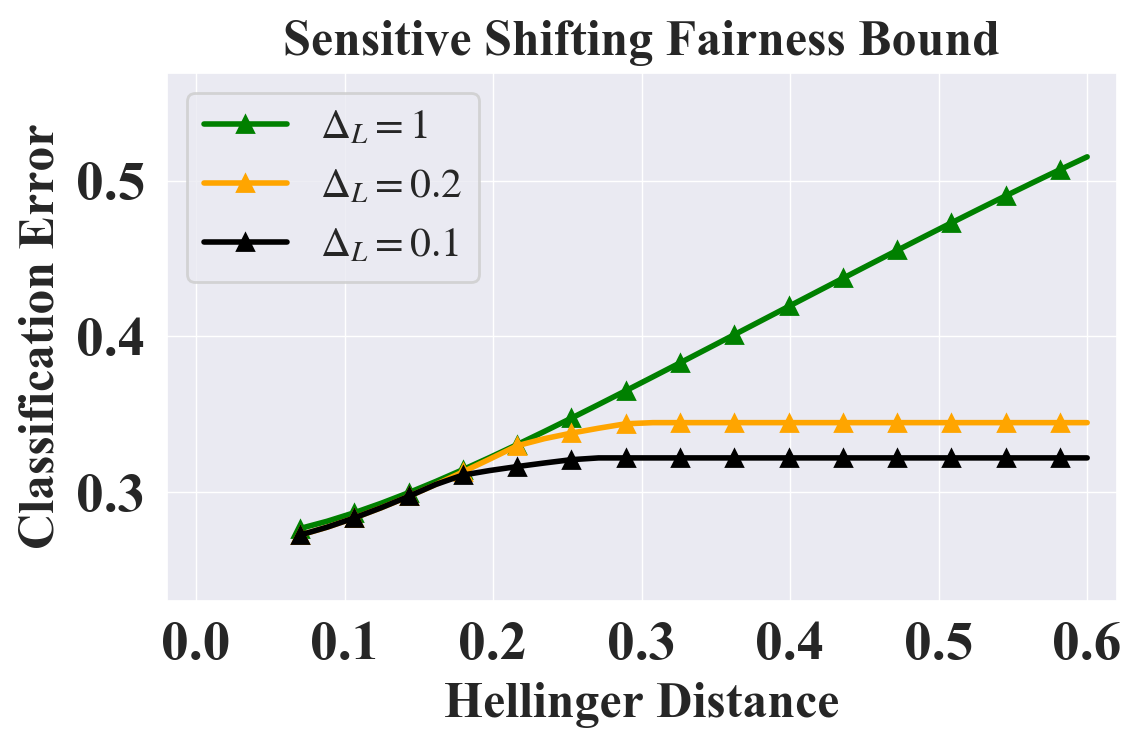}}
 \end{minipage}
 \begin{minipage}{0.5\linewidth}
 	\vspace{1pt}
 	\centerline{\includegraphics[width=1.0\textwidth]{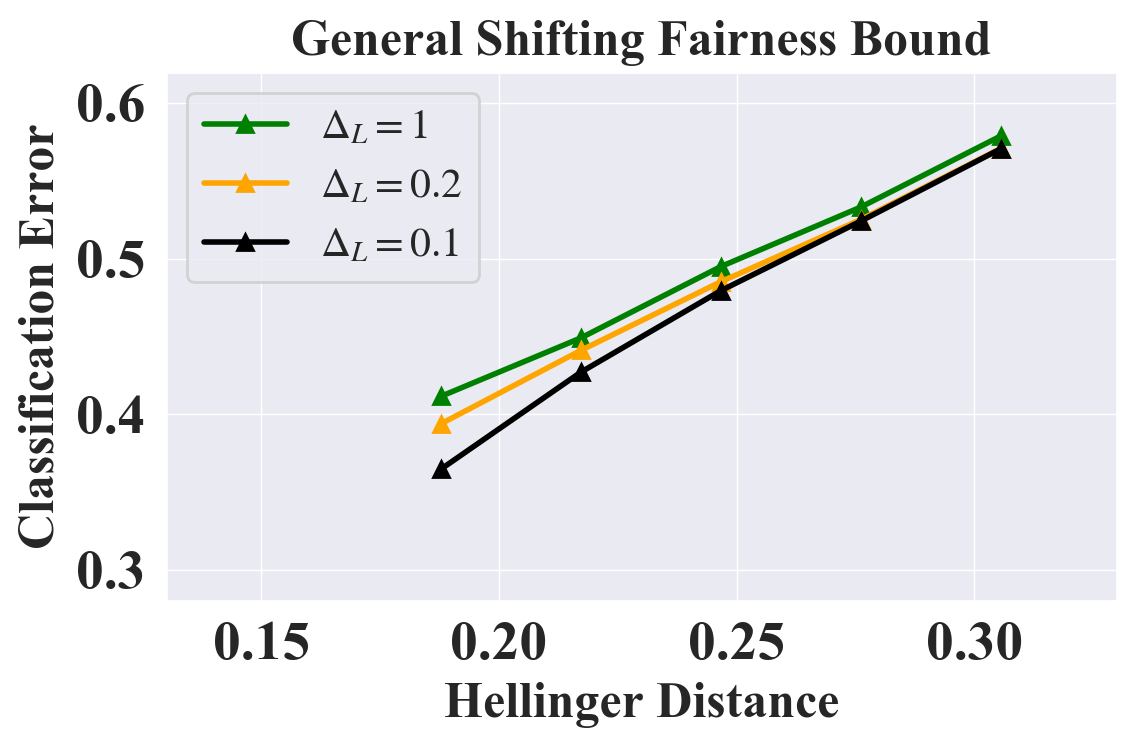}}
 \end{minipage}
    \caption{
    Certified fairness upper bounds with additional non-skewness constraints of labels on Adult. ($| \Pr_{(X,Y)\sim P}[Y=0]-\Pr_{(X,Y)\sim P}[Y=1] | \le \Delta_L$)
    }
\end{figure}

\begin{figure}[th]
 \begin{minipage}{0.5\linewidth}
 	\vspace{1pt}
 	\centerline{\includegraphics[width=1.0\textwidth]{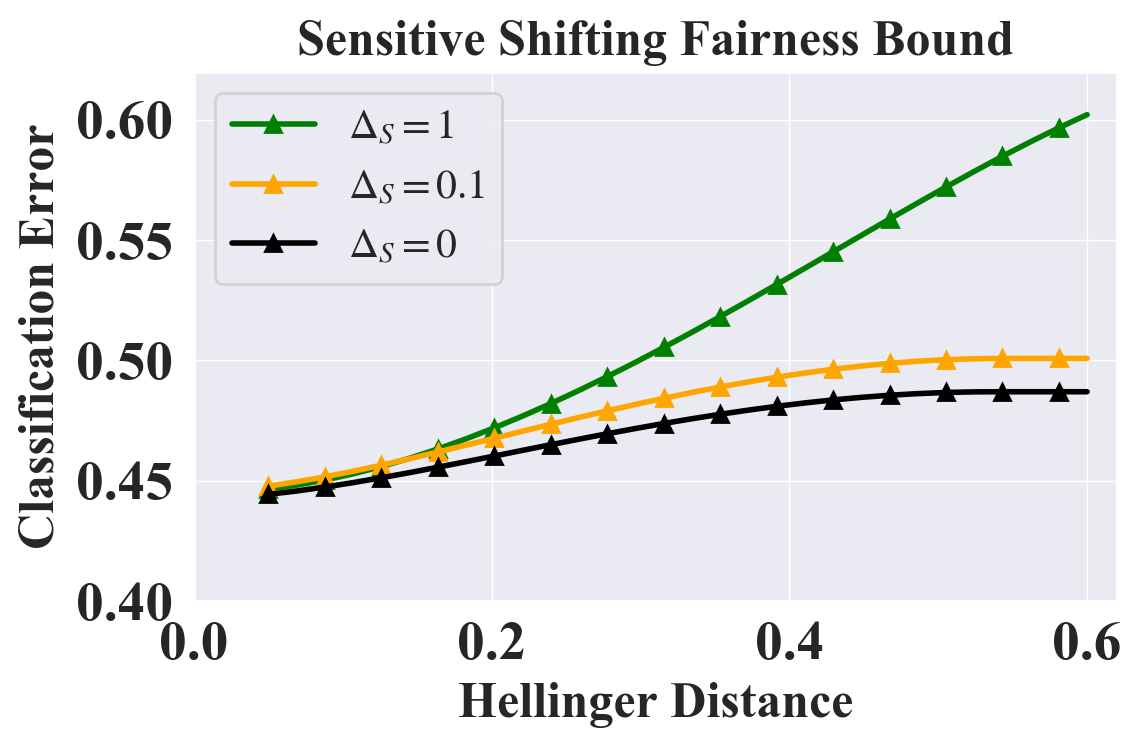}}
 \end{minipage}
 \begin{minipage}{0.5\linewidth}
 	\vspace{1pt}
 	\centerline{\includegraphics[width=1.0\textwidth]{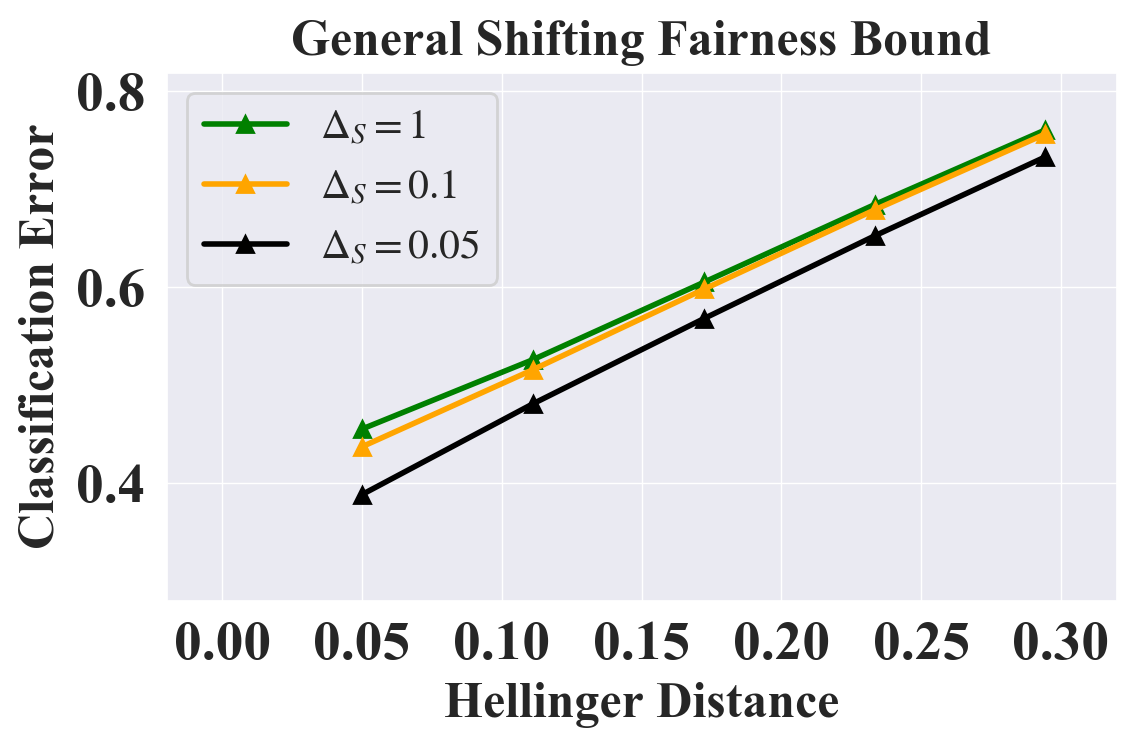}}
 \end{minipage}
    \caption{
    Certified fairness upper bounds with additional non-skewness constraints of sensitive attributes on Compas. ($| \Pr_{(X,Y)\sim P}[X_s=0]-\Pr_{(X,Y)\sim P}[X_s=1] | \le \Delta_s$)
    }
\end{figure}

\begin{figure}[htbp]
 \begin{minipage}{0.5\linewidth}
 	\vspace{1pt}
 	\centerline{\includegraphics[width=1.0\textwidth]{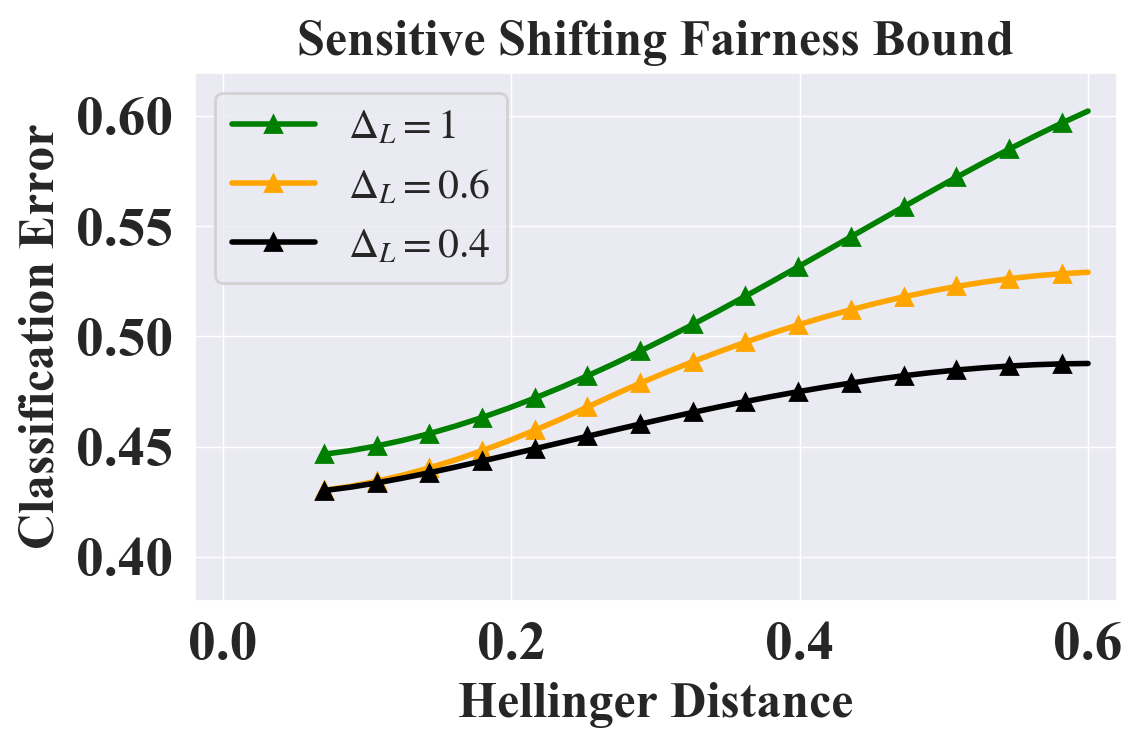}}
 \end{minipage}
 \begin{minipage}{0.5\linewidth}
 	\vspace{1pt}
 	\centerline{\includegraphics[width=1.0\textwidth]{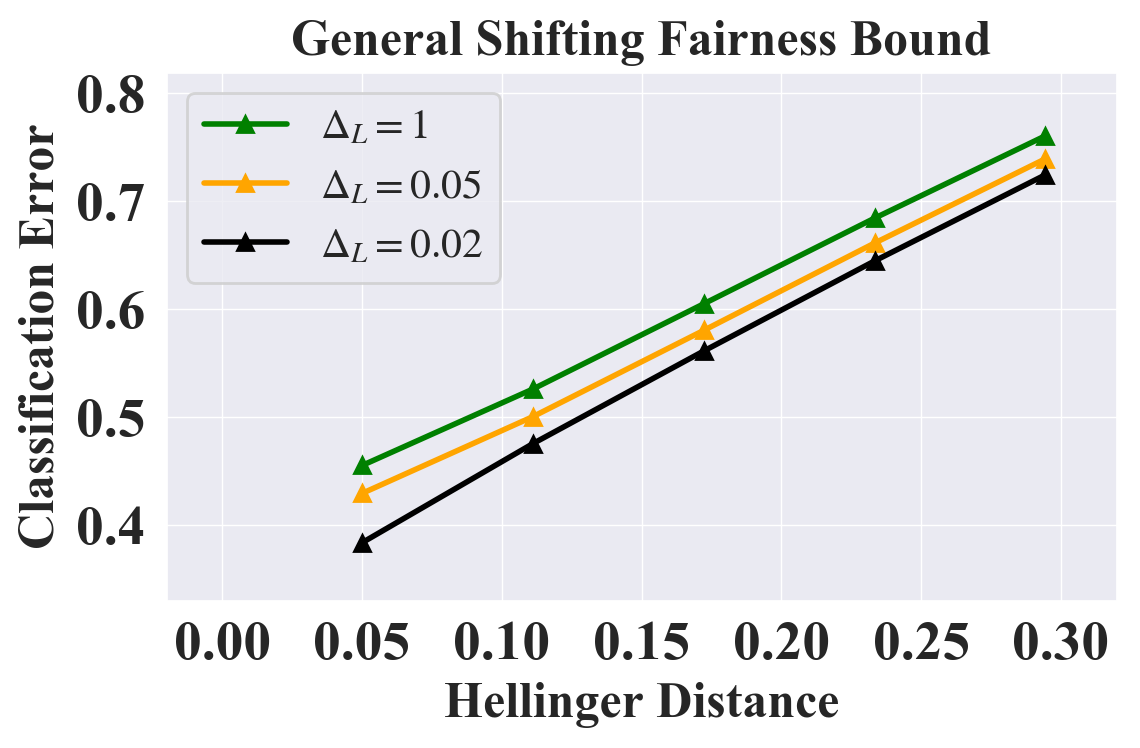}}
 \end{minipage}
    \caption{
    Certified fairness upper bounds with additional non-skewness constraints of labels on Compas. ($| \Pr_{(X,Y)\sim P}[Y=0]-\Pr_{(X,Y)\sim P}[Y=1] | \le \Delta_L$)
    }
\end{figure}

\begin{figure}[th]
 \begin{minipage}{0.5\linewidth}
 	\vspace{1pt}
 	\centerline{\includegraphics[width=1.0\textwidth]{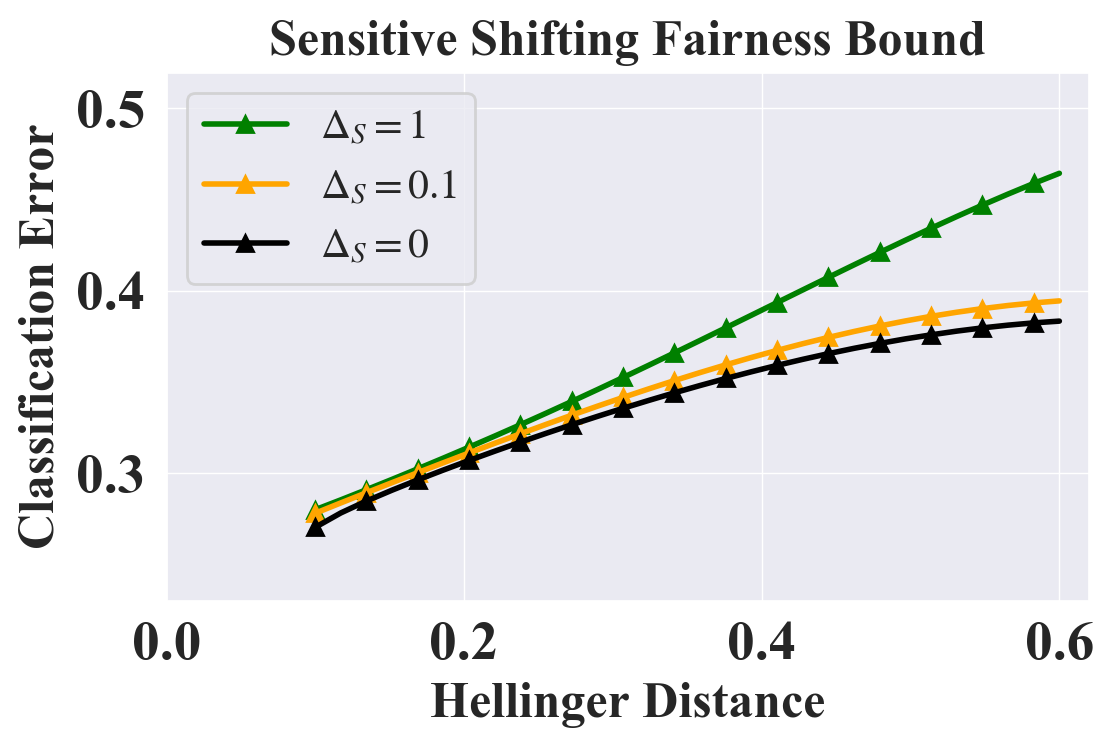}}
 \end{minipage}
 \begin{minipage}{0.5\linewidth}
 	\vspace{1pt}
 	\centerline{\includegraphics[width=1.0\textwidth]{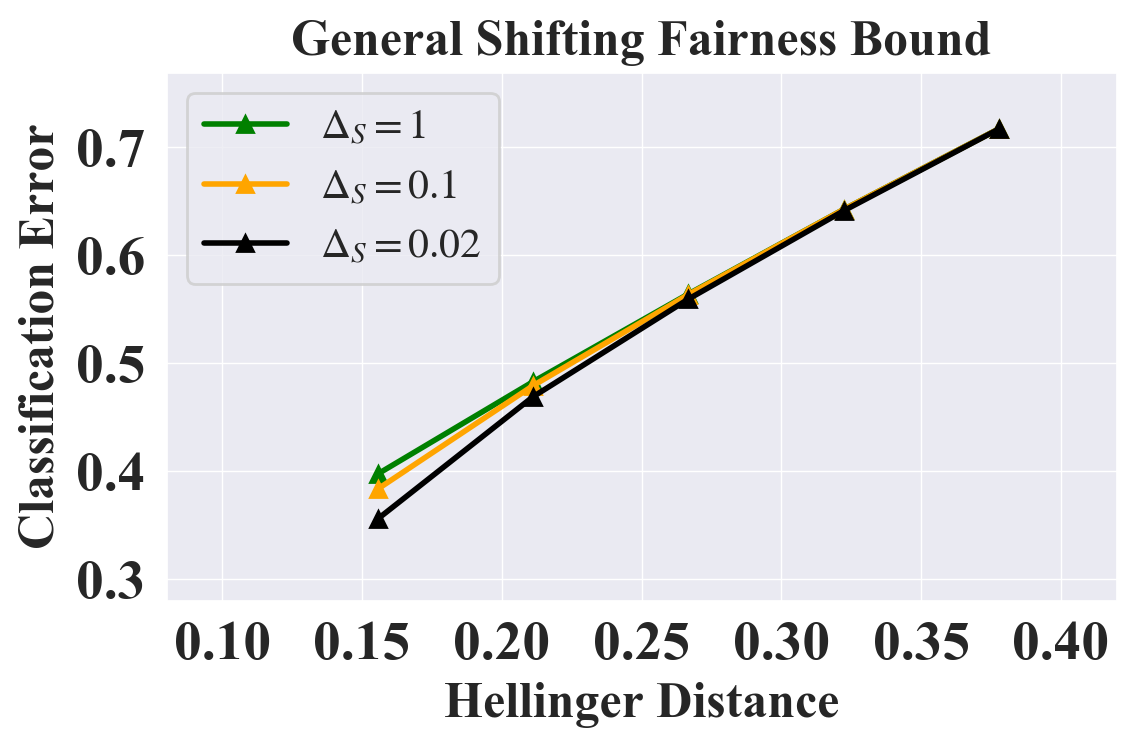}}
 \end{minipage}
    \caption{
    Certified fairness upper bounds with additional non-skewness constraints of sensitive attributes on Health. ($| \Pr_{(X,Y)\sim P}[X_s=0]-\Pr_{(X,Y)\sim P}[X_s=1] | \le \Delta_s$)
    }
\end{figure}

\begin{figure}[htbp]
 \begin{minipage}{0.5\linewidth}
 	\vspace{1pt}
 	\centerline{\includegraphics[width=1.0\textwidth]{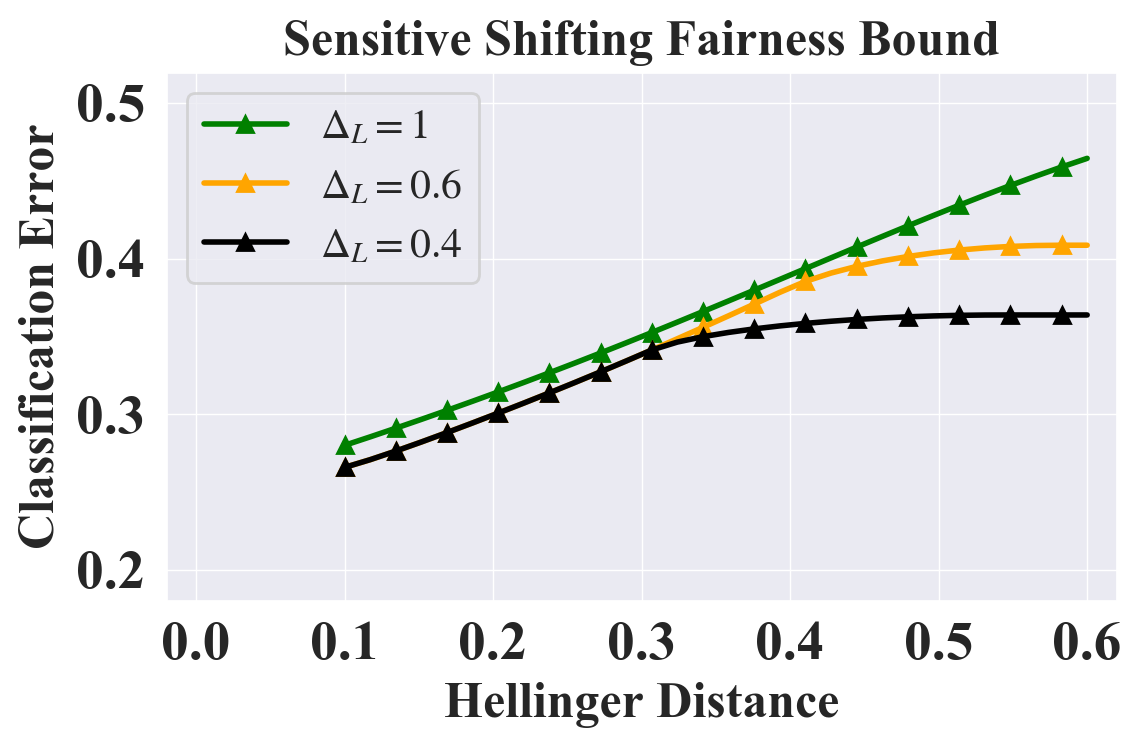}}
 \end{minipage}
 \begin{minipage}{0.5\linewidth}
 	\vspace{1pt}
 	\centerline{\includegraphics[width=1.0\textwidth]{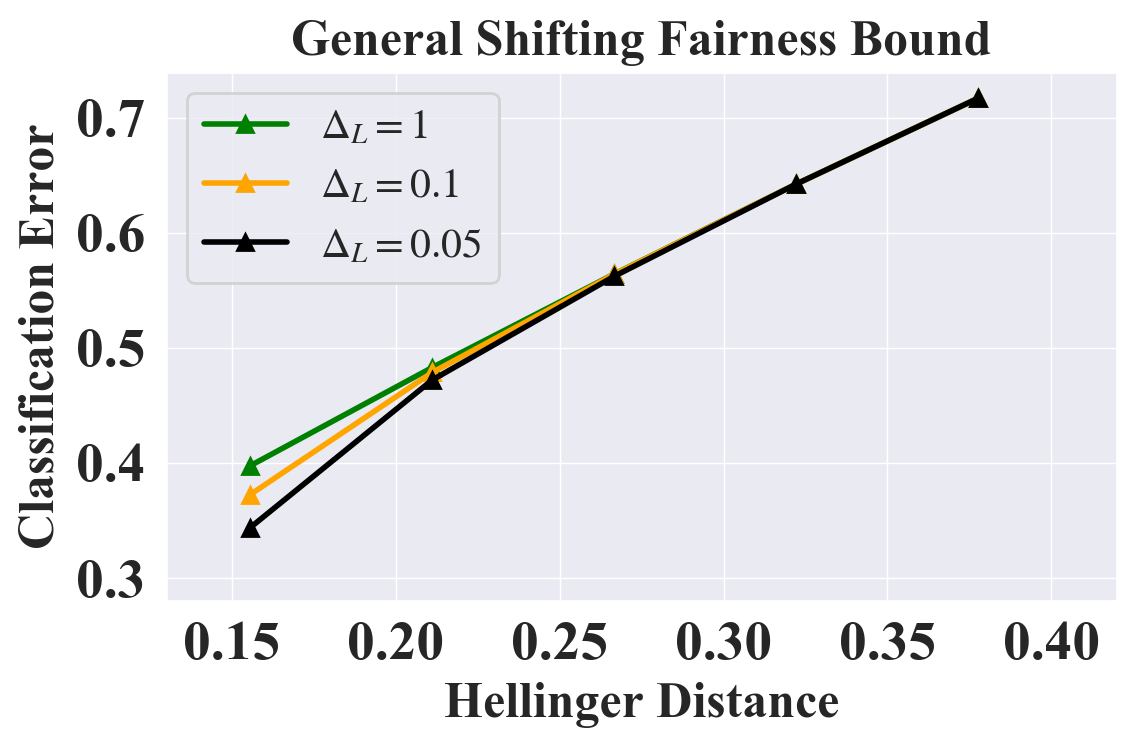}}
 \end{minipage}
    \caption{
    Certified fairness upper bounds with additional non-skewness constraints of labels on Health. ($| \Pr_{(X,Y)\sim P}[Y=0]-\Pr_{(X,Y)\sim P}[Y=1] | \le \Delta_L$)
    }
\end{figure}

\begin{figure}[th]
 \begin{minipage}{0.5\linewidth}
 	\vspace{1pt}
 	\centerline{\includegraphics[width=1.0\textwidth]{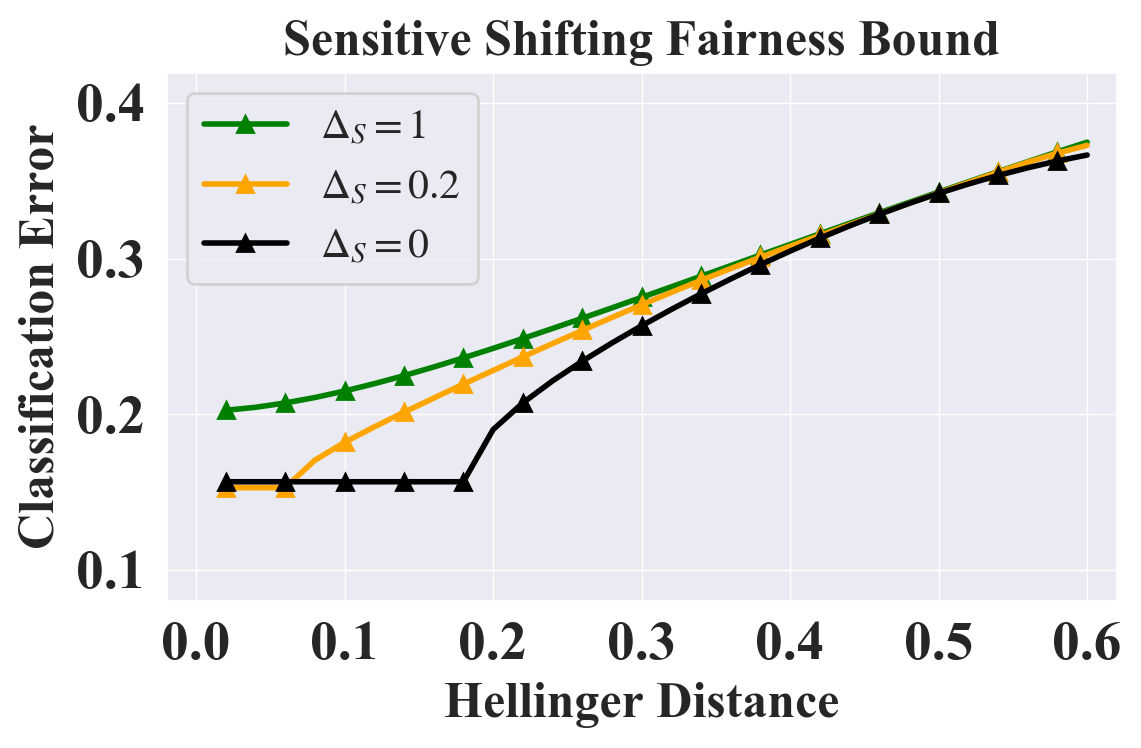}}
 \end{minipage}
 \begin{minipage}{0.5\linewidth}
 	\vspace{1pt}
 	\centerline{\includegraphics[width=1.0\textwidth]{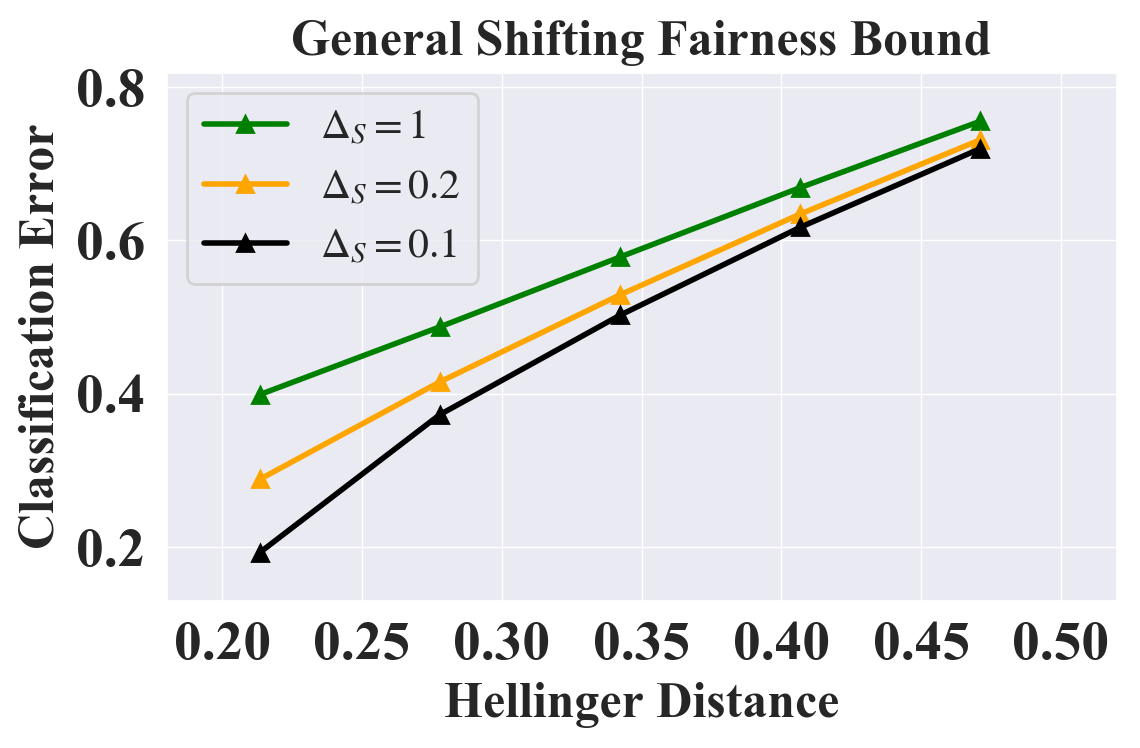}}
 \end{minipage}
    \caption{
    Certified fairness upper bounds with additional non-skewness constraints of sensitive attributes on Lawschool. ($| \Pr_{(X,Y)\sim P}[X_s=0]-\Pr_{(X,Y)\sim P}[X_s=1] | \le \Delta_s$)
    }
\end{figure}

\begin{figure}[tbp]
 \begin{minipage}{0.5\linewidth}
 	\vspace{1pt}
 	\centerline{\includegraphics[width=1.0\textwidth]{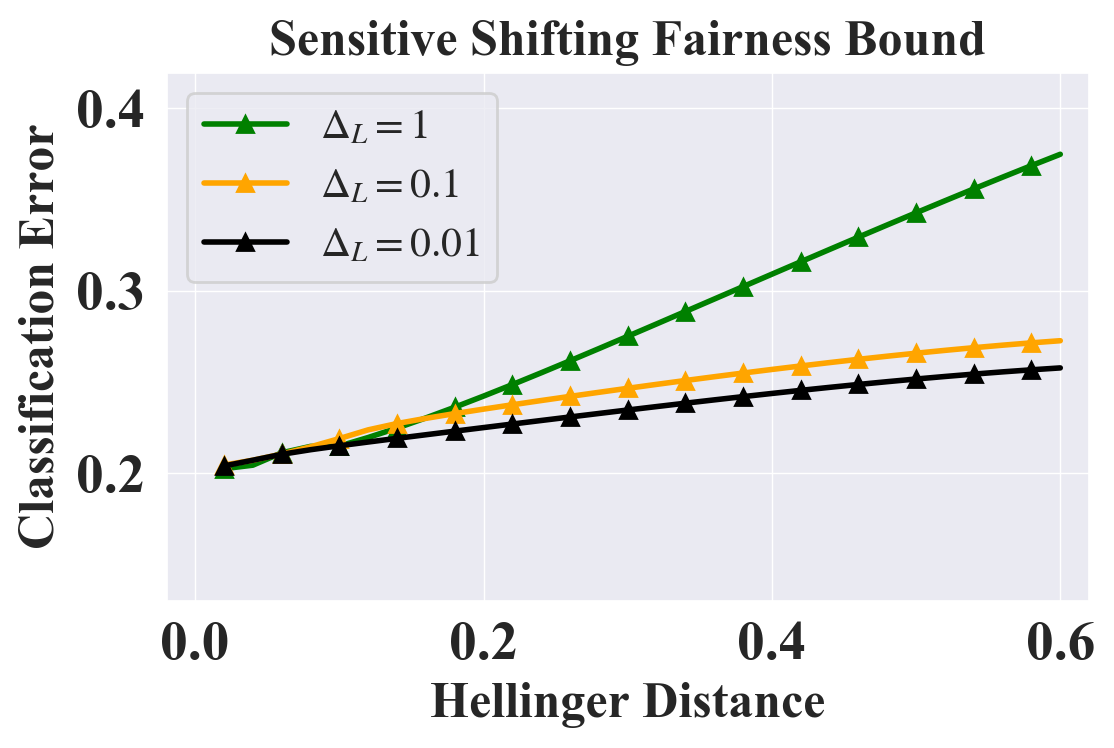}}
 \end{minipage}
 \begin{minipage}{0.5\linewidth}
 	\vspace{1pt}
 	\centerline{\includegraphics[width=1.0\textwidth]{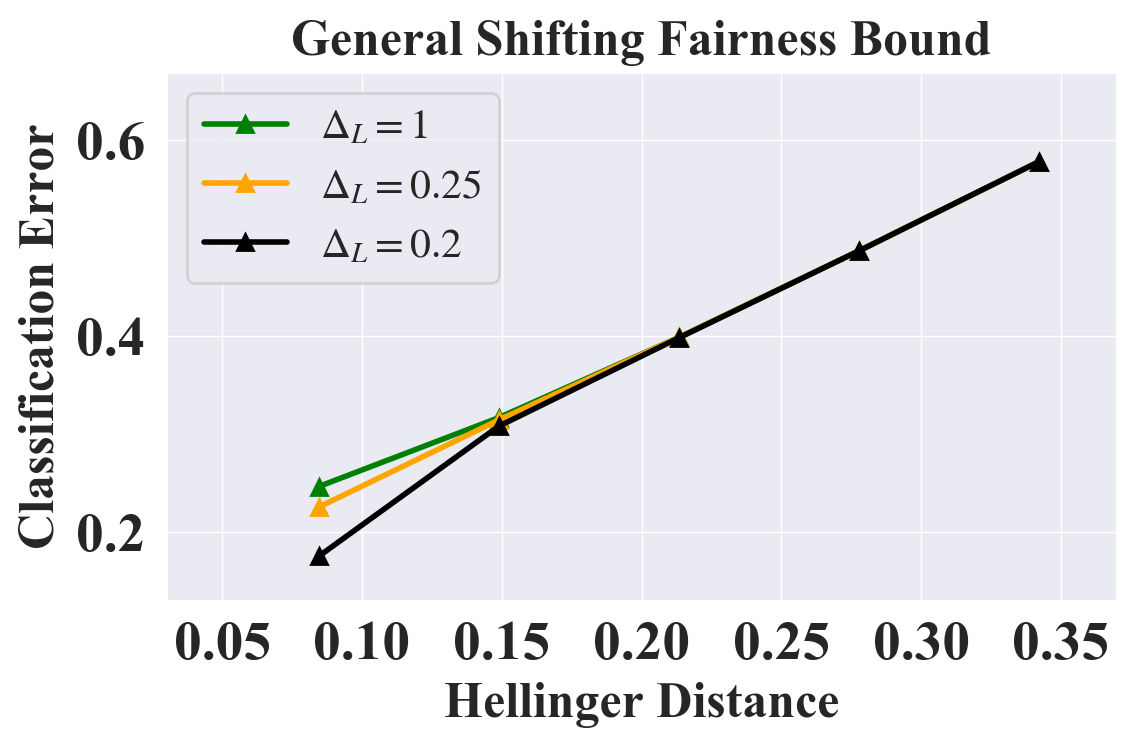}}
 \end{minipage}
    \caption{
    Certified fairness upper bounds with additional non-skewness constraints of labels on Lawschool. ($| \Pr_{(X,Y)\sim P}[Y=0]-\Pr_{(X,Y)\sim P}[Y=1] | \le \Delta_L$)
    }
\end{figure}
\clearpage


\subsection{Fair Classifier Achieves High Certified Fairness}
\label{sec:fair_model_exp}
We compare the fairness certificate of the vanilla model and the perfectly fair model on Adult dataset to demonstrate that our defined certified fairness in \Cref{prob:certified-fairness-shifting-two} and \Cref{prob:certified-fairness-shifting-one} can indicate the fairness in realistic scenarios.
In Adult dataset, we have 14 attributes of a person as input and try to predict whether the income of the person is over $50$k \$/year. The sensitive attribute in Adult is selected as the sex.
We consider four subpopulations in the scenario: 1) male with salary below $50$k, 2) male with salary above $50$k, 3) female with salary below $50$k, and 4) female with salary above $50$k.
We take the overall 0-1 error as the loss.
The vanilla model is real, and trained with standard training loss on the Adault dataset.
The perfectly fair model is hypothetical and simulated by enforcing the loss within each subpopulation to be the same as the vanilla trained classifier’s overall expected loss for fair comparison with the vanilla model.
From the experiment results in \Cref{tab:exp_comparison_sensitive_shifting} and \Cref{tab:exp_comparison_general_shifting}, we observe that our fairness certificates correlate with the actual fairness level of the model and verify that our certificates can be used as model’s fairness indicator: the certified fairness of perfectly fair models are consistently higher than those for the unfair model, for both the general shifting scenario and the sensitive shifting scenario. These findings demonstrate the practicality of our fairness certification.

\begin{table}[h]
    \centering
    \caption{Comparison of the fairness certificate of the vanilla model (an ``unfair'' model) and the perfectly fair model (a ``fair'' model) for sensitive shifting. 0-1 error is selected as the loss in the evaluation.}
    \label{tab:exp_comparison_sensitive_shifting}
    \begin{tabular}{c|ccccc}
    \toprule
    Hellinger Distance $\rho$ &  0.1 & 0.2 & 0.3 & 0.4 & 0.5 \\
    \midrule
    Vanilla Model Fairness Certificate  &  0.182 & 0.243 & 0.297 & 0.349 & 0.397 \\
    Fair Model Fairness Certificate & 0.148 & 0.148 & 0.148 & 0.148 & 0.148\\
    \bottomrule
    \end{tabular}
\end{table}

\begin{table}[h]
    \centering
    \caption{Comparison of the fairness certificate of the vanilla model (an ``unfair'' model) and the perfectly fair model (a ``fair'' model) for general shifting. 0-1 error is selected as the loss in the evaluation.}
    \label{tab:exp_comparison_general_shifting}
    \begin{tabular}{c|ccccc}
    \toprule
    Hellinger Distance $\rho$ &  0.1 & 0.2 & 0.3 & 0.4 & 0.5 \\
    \midrule
    Vanilla Model Fairness Certificate  & 0.274 & 0.414 & 0.559 & 0.701 & 0.828 \\
    Fair Model Fairness Certificate & 0.266 & 0.407 & 0.553 & 0.695 & 0.824 \\
    \bottomrule
    \end{tabular}
\end{table}

\end{document}